%% file: arXiv2.tex
\documentclass[12pt]{article}

\usepackage[margin=1in]{geometry}

\usepackage{setspace}

\usepackage{microtype}
\usepackage{graphicx}
\usepackage{subfigure}
\usepackage{booktabs}
\usepackage{comment}
\usepackage{wrapfig}
\usepackage{arydshln}
\usepackage{enumitem}
\usepackage{multirow}
\usepackage{url}
\usepackage{natbib}
\usepackage{authblk}

\usepackage{hyperref}
\hypersetup{colorlinks,
            linkcolor=blue,
            citecolor=blue,
            urlcolor=magenta,
            linktocpage,
            plainpages=false}

\usepackage{amsmath}
\usepackage{amssymb}
\usepackage{mathtools}
\usepackage{amsfonts}
\usepackage{bm}
\usepackage{bbm}

\input{def.tex}

\usepackage{algorithm}
\usepackage{algorithmic}

\usepackage{microtype}
\usepackage{graphicx}
\usepackage{subfigure}
\usepackage{booktabs} 
\usepackage{url}
\usepackage{amsthm}

\usepackage{hyperref}
\allowdisplaybreaks
\title{Efficient Adaptive Experimental Design\\
for Average Treatment Effect Estimation}

\author[1,2]{Masahiro Kato\thanks{Email: \texttt{mkato-csecon@g.ecc.u-tokyo.ac.jp}}$\,$}
\author[3]{Takuya Ishihara}
\author[4]{Junya Honda}
\author[5]{Yusuke Narita}

\affil[1]{The University of Tokyo}
\affil[2]{Mizuho-DL Financial Technology Co., Ltd.}
\affil[3]{Tohoku University}
\affil[4]{Kyoto University}
\affil[5]{Yale University}

\date{First version:  Feb 2020,  This version is of  \today.
 }

\begin{document}

\maketitle

\begin{abstract}
We study how to efficiently estimate average treatment effects (ATEs) using adaptive experiments. In adaptive experiments, experimenters sequentially assign treatments to experimental units while updating treatment assignment probabilities based on past data. We start by defining the efficient treatment-assignment probability, which minimizes the semiparametric efficiency bound for ATE estimation. 
Our proposed experimental design estimates and uses the efficient treatment-assignment probability 
to assign treatments. At the end of the proposed design, the experimenter estimates the ATE using a newly proposed Adaptive Augmented Inverse Probability Weighting (A2IPW) estimator. We show that the asymptotic variance of the A2IPW estimator using data from the proposed design achieves the minimized semiparametric efficiency bound. 
We also analyze the estimator's finite-sample properties and develop nonparametric and nonasymptotic confidence intervals that are valid at any round of the proposed design. These anytime valid confidence intervals allow us to conduct rate-optimal sequential hypothesis testing, allowing for early stopping and reducing necessary sample size. 
\footnote{This paper was first made public in February 13, 2020 at \url{https://arxiv.org/abs/2002.05308v1}, and was  presented at the NeurIPS 2020 Workshop on Causal Discovery \& Causality-Inspired Machine Learning, as well as workshops at the University of Tokyo and Keio University. We thank the seminar participants for their valuable comments and feedback. We would also like to express our deep gratitude to Professor Hidehiko Ichimura for his insightful guidance and encouragement.}
\end{abstract}

\clearpage

\input{Main.tex}

\end{document}

%% file: def.tex
\usepackage[utf8]{inputenc}
\usepackage[T1]{fontenc}
\usepackage{graphicx,psfrag,epsf}
\usepackage{color,lscape,longtable}
\usepackage{booktabs}
\usepackage{bm}
\usepackage{tcolorbox}
\usepackage{multicol,multirow}
\usepackage{natbib}
\usepackage{comment}
\usepackage{booktabs}

\usepackage{enumerate}   
\usepackage{amsmath}
\usepackage{bbm}
\usepackage{tikz}

\usetikzlibrary{positioning,arrows}
\usepackage{wrapfig}

\DeclareMathOperator*{\argmin}{arg\,min}

\usepackage{amsmath}
\usepackage{appendix}
\usepackage{amssymb}

\usepackage{amsthm}
\usepackage[capitalise]{cleveref}
\Crefname{assumption}{Assumption}{Assumptions}

\theoremstyle{plain}
\newtheorem{theorem}{Theorem}
\newtheorem{lemma}{Lemma}
\newtheorem{corollary}{Corollary}
\newtheorem{assumption}{Assumption}
\newtheorem{proposition}{Proposition}
\newtheorem{definition}{Definition}
\newtheorem{remark}{Remark}

\newcommand{\op}{\mathrm{o}_{p}}

\newcommand{\pa}{\mathrm{\pa}}

\renewcommand{\eqref}[1]{(\ref{#1})}

\newcommand{\RN}[1]{%
  \textup{\uppercase\expandafter{\romannumeral#1}}%
}

\def\boxit#1{\vbox{\hrule\hbox{\vrule\kern6pt\vbox{\kern6pt#1\kern6pt}\kern6pt\vrule}\hrule}}

\usepackage{xcolor}
\newcount\Comments  

%% file: Main.tex
\section{Introduction}
Adaptive experiments are increasingly common in the social sciences, the tech industry, and medicine. In adaptive experiments, experimenters sequentially assign treatments to experimental units while updating treatment assignment probabilities based on past data. 
Compared to the non-adaptive randomized control trial (RCT), adaptive designs often allow experimenters to more efficiently or quickly detect causal effects, thus exposing fewer experimental units to costly or harmful treatments. This merit has led organizations such as the US Food and Drug Administration to recommend adaptive designs \citep{fda}. Adaptive experiments also produce social and economic applications and spark theoretical interest. 

This paper studies how to design an adaptive experiment for efficient estimation of the average effects of treatment (ATE) and hypothesis testing. 
Let $Y(1), Y(0) \in \mathcal{Y}$ be potential outcomes of treatment $1$ and control $0$, respectively, where $\mathcal{Y} \subset \mathbb{R}$ is a bounded outcome space (see Assumption~\ref{asm:boundedness}). Let $X \in \mathcal{X}$ be covariates, where $\mathcal{X}$ represents a space of covariates. The random variables $(X, Y(1), Y(0))$ jointly follow an unknown distribution $P_0 \in \mathcal{P}$, where $\mathcal{P}$ is the set of the distributions over $(X, Y(1), Y(0))$.
We are interested in the estimation of \emph{average treatment effect} (ATE), defined as 
\begin{align*}
    \theta_0 \coloneqq \mathbb{E}[Y(1)] - \mathbb{E}[Y(0)],
\end{align*}
where $\mathbb{E}[Y(a)]$ denotes the mean potential outcome for each treatment $a\in\{1, 0\}$. 
The experiment involves $T \in \mathbb{N}$ experimental units, who are assigned to the treatment ($1$) or the control ($0$). For each $t \in [T]$, let $(X_t, Y_t(1), Y_t(0))$ be an i.i.d. draw of $(X, Y(1), Y(0))$ following the distribution $P_0$.

We propose the following adaptive experiment consisting of (1) a treatment-assignment phase and (2) an ATE-estimation phase using a novel estimator. :
\begin{itemize} 
\item Step 1. Treatment-assignment phase: 
\begin{itemize}
    \item In each round $t \in [T] = {1, 2, \dots, T}$, an experimental unit with covariate $X_t \in \mathcal{X}$ visits the experimenter;
    \item The experimenter assigns treatment $A_t \in \{1, 0\}$ with probability $\pi_t(a \mid X_t, \mathcal{H}_{t-1})$, based on the covariate $X_t$ and past observations \[\mathcal{H}_{t-1} \coloneqq \{X_1, A_1, Y_1, X_2, \dots, Y_{t-2}, X_{t-1}, A_{t-1}, Y_{t-1}\},\] where $Y_t = \mathbbm{1}[A_t = 1]Y_t(1) + \mathbbm{1}[A_t = 0]Y_t(0)$ is the observed outcome; 
    \item After treatment assignment, the experimenter observes the outcome $Y_t \in \mathbb{R}$; 
\end{itemize}
\item Step 2. ATE-estimation phase: 
\begin{itemize}
    \item We estimate ATE $\theta_0$ using observations
\[\mathcal{H}_{T} = \{(X_i, A_i, Y_i)\}^T_{i=1}.\] 
\end{itemize}
\end{itemize}

The treatment-assignment probability can be updated after each round based on the observations collected up to that point. Our method is also applicable to batch settings, where updates occur only in specified rounds. The treatment-assignment probability $\pi_t$ is usually called a propensity score in observational studies. 

Note that from our assumption that $(X_t, Y_t(1), Y_t(0))$ is i.i.d. over $t\in[T]$, the Stable Unit Treatment Value Assumption (SUTVA) holds \citep{imbens_rubin_2015}. Furthermore, unconfoundedness also holds from the construction of the treatment assignment probability $\pi_t(a\mid X_t, \mathcal{H}_{t-1})$; that is, outcomes $(Y_t(1), Y_t(0))$ and treatment $A_t$ are conditionally independent given $X_t$ and $\mathcal{H}_{t-1}$. 

In addition to ATE estimation, we also analyze hypothesis testing about $\theta_0$ with null and alternative hypotheses defined for some $\mu \in \mathbb{R}$ as 
\begin{align}
\label{eq:hypothesis}
H_0: \theta_0 = \mu, \ \ \ H_1: \theta_0 \neq \mu.
\end{align}

We begin by investigating the semiparametric efficiency bound for ATE estimators. Following the approach of \citet{Hahn2011}, we minimize the semiparametric efficiency bound with respect to treatment-assignment probabilities and define the minimizer as the efficient treatment-assignment probability. 
This efficient treatment-assignment probability is expressed as the ratio of the covariate-conditional standard deviations of the potential outcomes. This treatment-assignment probability is a variant of the one proposed in \citet{Neyman1934OnTT}, which recently has been called the Neyman allocation.

Step (1) of our adaptive experiment sequentially estimates these conditional standard deviations, calculates the efficient treatment-assignment probability, and assigns treatment based on this estimate. 
To implement Step (2) of efficient ATE estimation, we introduce and use an ATE estimator, which we call the Adaptive Augmented Inverse Probability Weighting (A2IPW) estimator, which is a variant of the  Augmented Inverse Probability Weighting (AIPW) estimator designed for adaptive experiments \citep{BangRobins2005}.


We analyze both the infinite-sample and finite-sample properties of the A2IPW estimator. In the infinite-sample analysis, we demonstrate its consistency and asymptotic normality, showing that its asymptotic variance reaches the minimized semiparametric efficiency bound. 

We then study hypothesis testing under two frameworks: single-stage testing and sequential testing. In the single-stage approach, we perform standard hypothesis testing by constructing confidence intervals with a fixed sample size to decide whether to reject the null hypothesis. In the sequential testing approach, the sample size is not fixed; instead, we continue collecting data until a decision can be made with a predetermined Type~I error probability. Sequential testing has the potential to reduce the sample size by stopping the adaptive experiment early. 

We propose a sequential testing procedure based on the finite-sample analysis of our estimator. Specifically, we derive a confidence interval that is nonparametric and non-asymptotic; it does not rely on a distributional assumption and an asymptotic approximation. 
We derive our confidence interval based on the Law of the Iterated Logarithm \citep[LIL,][]{Balsubramani2016,Howard2020TimeuniformNN}. In addition, our confidence intervals are Bernstein-type and use information about the variance of potential outcomes. 
As a result, our sequential testing with LIL-type anytime valid confidence intervals is rate-optimal for stopping time and effectively reduces the sample size \citep{Jamieson2014}. 
In particular, our confidence intervals are narrower than other confidence intervals, such as those based on Hoeffding's inequality, which rely solely on the boundedness of outcomes. 

\color{black}

\subsection{Related Work}
This study contributes to the growing work on adaptive experimental design for efficient estimation and inference of treatment effects. Important problems include how to design treatment assignment probabilities \citep{Hahn2011} and how to make statistical decisions \citep{Manski2000}. This paper addresses those problems by designing an adaptive experiment for efficiently estimating the ATE with associated hypothesis testing and decision-making methods. 


Compared to existing studies such as \citet{Hahn2011}, our adaptive experiment offers the following advantages: 
\begin{description}
    \item[Flexible sample size computation:] Our proposed design does not require dividing experimental units into discrete prespecified batches (though our design can also be used in such batch settings). Without prefixing the sample size for batches, our approach allows for the sequential construction of the optimal treatment assignment. 
    \item[Semiparametric inference without the Donsker condition:] Our experiment does not require the Donsker condition for the estimators of the nuisance parameters (i.e., the conditional expected outcome and the efficient treatment-assignment probability). Instead, we impose convergence rate conditions for the estimators, similar to double machine learning in \citet{ChernozhukovVictor2018Dmlf}. This flexibility allows us to use a variety of machine learning estimators for estimating nuisance parameters.
    \item[Weaker assumptions on the covariate distribution:] We do not require specific assumptions (such as discrete support) on the covariate distribution as long as the convergence rate conditions are satisfied.
\end{description}
Furthermore, our study examines the finite sample properties of ATE estimation and the sequential testing method.


\citet{Kato2021adr} complements this work by highlighting that the proposed A2IPW estimator is a variant of double machine learning. They generalize the A2IPW estimator into the Adaptive Doubly Robust (ADR) estimator, which enables the estimation of the treatment-assignment probability. Their findings indicate that empirical performance can be improved by replacing the treatment-assignment probability with its estimator, even when the true value of the treatment-assignment probability is known. For a detailed discussion of double machine learning in adaptive experiments, see their paper. \citet{kato2021adaptivedoublyrobustestimator} further extends the ADR estimator for the case where the average of the treatment-assignment probability converges to a constant, even if the probability itself does not converge.

Our method and the framework for adaptive experimental design for ATE estimation have  been extended in various directions. Some works have relaxed our assumptions \citep{cook2023semiparametric, Waudby-Smith2024}, and others have adapted our proposed estimator for cases with unknown treatment-assignment probabilities \citep{Kato2021adr, li2023double}. \citet{Vikas2024} refines the asymptotic optimality in this problem. \citet{gupta2021efficient} and \citet{chandak2024adaptive} address endogeneity problems with instrumental variables, while \citet{li2024privacy} explore privacy-preserving aspects. \citet{Simchi-Levi2023} investigates the setting under nonstationarity. \citet{zrnic2024active}, \citet{kato2024active} and \citet{ao2024predictionguidedactiveexperiments} introduce the idea of active learning for this problem setting. 

The framework of \citet{Hahn2011} is called a stratified experiment, where experimental units are divided into several strata based on their covariates \citep{Bugni2018,Bugni2019}. 
Concurrently with our work, \citet{Meehan2022} proposes a stratification method based on a tree-based algorithm within a two-stage experimental framework to relax the assumption of discrete support in \citet{Hahn2011}. In contrast, our algorithm does not depend on specific models or algorithms for determining treatment-assignment probabilities or for estimating the ATE. Instead, our method incorporates double machine learning techniques into our experimental design \citep{ChernozhukovVictor2018Dmlf}, allowing for a wide range of traditional and modern machine learning estimators. Furthermore, our method is applicable to various settings of adaptive experimental design, including two-stage, multi-stage, and sequential experiments. 

Furthermore, after the initial public draft of this paper \citep{Kato2020adaptive}, several related studies have emerged. \citet{Kallus2021}, \citet{bai2025efficiencyfinelystratifiedexperiments}, and \citet{rafi2023efficientsemiparametricestimationaverage} discuss efficiency bounds or efficient experiments under the stratification setting. \citet{Armstrong2022} and \citet{hirano2023asymptotic} investigate asymptotically optimal treatment rules in adaptive experiments. Furthermore, \citet{Cai2024} and \citet{zhao2023adaptive} investigate the Neyman allocation from perspectives different from ours.

We derive the asymptotic distribution of our A2IPW estimator using martingale theory. Notably, our asymptotic normality result does not require the Donsker condition for the nuisance parameter estimator. This approach is similar in spirit to sample-splitting methods used in the semiparametric analysis, such as double machine learning \citep{klaassen1987,ZhengWenjing2011CTME,ChernozhukovVictor2018Dmlf}. \citet{hadad2019} also independently proposes a closely related estimator, including ATE estimation, for bandit problems, focusing on cases where the treatment-assignment probability approaches zero at a certain rate with respect to $t$.

Efficient estimation with adaptive experiments is closely related to the Best Arm Identification (BAI) problem in multi-armed bandit (MAB) settings \citep{Bubeck2009, Kasy2021}. Neyman allocation is known to be optimal in BAI problems under certain conditions, such as Gaussian outcomes when variances are known \citep{chen2000, glynn2004large, Kaufman2016complexity}. When variances are unknown, our proposed A2IPW strategy is still optimal in BAI as the ATE approaches zero \citep{adusumilli2022minimax,kato2025generalizedneymanallocationlocally}. \citet{adusumilli2022minimax} proves that the Neyman allocation is minimax optimal for the BAI problem. \citet{Armstrong2022} and \citet{Adusumilli2021risk} study asymptotic treatment rules in adaptive experiments. In the setting of BAI, \citet{kato2025generalizedneymanallocationlocally} generalizes the Neyman allocation for the multi-armed case. In BAI problems with covariates, researchers investigate identifying the best treatment arm based on expected outcomes marginalized over the covariate distribution or the conditional on covariates \citep{Russac2021, kato2021role, simchi2024experimentation, kato2024adaptivepolicylearning}. \citet{SimchiLevi2023} and \citet{Caria2023} integrate the statistical inference problem with the regret minimization problem in MAB.

This study investigates the finite-sample property of the AIPW estimator in adaptive experiments. 
Our non-asymptotic error analysis is based on the law of the iterated logarithm \citep[LIL,][]{Darling1967,Howard2020TimeuniformNN}. The LIL plays an important role in finite-sample analysis and sequential testing since it is known to return tighter confidence intervals. \citet{Balsubramani2016} propose nonparametric sequential testing using the LIL, and we apply their results to adaptive ATE estimation with the A2IPW estimator. \citet{Smith2024} and \citet{Cai2024} also address the finite-sample analysis.

\subsection{Organization}
This study is organized as follows. In Section~\ref{sec:problem}, we introduce the data-generating process and discuss the semiparametric efficiency bound. In Section~\ref{sec:main_exp_design}, we design an adaptive experiment for efficient ATE estimation and powerful hypothesis testing, and we also propose the A2IPW estimator. Next, in Section~\ref{sec:estimator_prop}, we present the theoretical properties of our A2IPW estimator, focusing on its asymptotic normality, efficiency, and non-asymptotic results. Notably, its asymptotic variance aligns with the semiparametric efficiency bound. In Section~\ref{sec:analysis}, we examine hypothesis testing under our adaptive experimental framework. Finally, in Section~\ref{sec:exp}, we assess the empirical performance of the proposed method using both synthetic and semi-synthetic data. All Appendices are provided in the online supplementary materials.

\section{Semiparametric Efficiency Bound and Efficient Assignment Probability}
\label{sec:problem}

\subsection{Semiparametric Efficiency Bound in Adaptive Experimental Design}
This section provides a lower bound for the asymptotic variance of regular estimators of the ATE in adaptive experiments, following the arguments in \citet{Hahn2011}. Specifically, we focus on the semiparametric lower bound, which establishes a theoretical limit for the asymptotic variances of regular ATE estimators under semiparametric models.\footnote{The asymptotic variance can also be interpreted as the asymptotic mean squared error when the ATE estimator is asymptotically normal. Consequently, the semiparametric lower bound serves as a lower bound for the estimation error.}

Consider i.i.d. observations $\{(X_i, A_i, Y_i)\}_{i=1}^n$ generated from a distribution $P_0$ with a treatment-assignment probability $\pi_0(a \mid X_i)$. This treatment-assignment probability can be optimized for ATE estimation; thus, we refer to an algorithm with such a treatment-assignment probability as an \emph{oracle} algorithm. In this case, from Theorem~1 of \citet{Hahn1998}, the semiparametric efficiency bound is given as follows:

\begin{proposition}[Semiparametric efficiency bound of ATE estimators. Based on Theorem~1 of \citet{Hahn1998}.]
Suppose that the same regularity conditions assumed in Theorem~1 of \citet{Hahn1998} hold. Under an oracle algorithm with treatment-assignment probability $\pi_0$, the asymptotic variance of regular ATE estimators is lower bounded by 
\begin{align}
\label{eq:semipara_bound}
V(\pi_0) \coloneqq \mathbb{E}_{P_0}\left[\frac{\sigma^2_0\big(1, X\big)}{\pi_0(1 \mid X)} + \frac{\sigma^2_0\big(0, X\big)}{\pi_0(0 \mid X)} + \Big(\theta_0(X) - \theta_0\Big)^2\right],
\end{align}
where $\sigma^2_0\big(a, X\big)$ is the conditional variance of $Y(a)$ given $X$ for $a\in\{1, 0\}$. 
\end{proposition}

This proposition corresponds to the case where the oracle treatment-assignment probability $\pi_0$ is known in advance, eliminating the need for estimation during the adaptive experiment. In such a scenario, treatments are assigned directly using the oracle treatment-assignment probability. This static oracle algorithm serves as a benchmark in the study of adaptive experimental designs \citep{Hahn2011}.

If we restrict the algorithm to those where $\pi_t \xrightarrow{\mathrm{p}} \pi_0$ as $t \to \infty$, the result can be extended to non-i.i.d. observations using the martingale central limit theorem, as demonstrated in the derivation of asymptotic normality. Various extensions of lower bounds also have been proposed \citep{li2023double,rafi2023efficientsemiparametricestimationaverage}. 


\subsection{Efficient Treatment-assignment Probability}
In the semiparametric efficiency bound \eqref{eq:semipara_bound}, decision-makers can select $\pi_0$ to minimize the asymptotic variance. Denote the efficient treatment-assignment probability by 
\[
\pi^* \coloneqq \argmin_{\pi_0 \in \Pi} V(\pi_0).
\]
The minimization problem has a closed-form solution, as shown below:

\begin{proposition}[Efficient treatment-assignment probability]
\label{optprob}
The efficient treatment-assignment probability $\pi^*$ is:
\begin{align*}
\pi^*(a \mid x) = \frac{\sqrt{\sigma^2_0(a, x)}}{\sqrt{\sigma^2_0(1, x)} + \sqrt{\sigma^2_0(0, x)}}, \quad \forall a \in \{1, 0\},\ \forall x \in \mathcal{X}.
\end{align*}
\end{proposition}

The proof is presented in Appendix~\ref{appdx:prp:opt_prob}. 

Intuitively, conditional on $x$, the asymptotic variance can be minimized by assigning the treatment with a higher variance of the potential outcome. This treatment-assignment probability is recently referred to as the Neyman allocation \citep{Neyman1934OnTT} and has been investigated in various studies on experimental design \citep{chen2000,glynn2004large,Laan2008TheCA,Hahn2011,Kaufman2016complexity,Meehan2022}.

\section{Semiparametric Efficient Adaptive Experiment}
\label{sec:main_exp_design}
In this section, we design an adaptive experiment that minimizes the semiparametric efficiency bound and an ATE estimator whose asymptotic variance hits the minimized semiparametric efficiency bound.
As explained in the Introduction, our experiment consists of two steps: 
\begin{itemize}
    \item \textbf{Step (1). Treatment-assignment phase:} In each round $t\in[T]$, we estimate the efficient treatment-assignment probability $\pi^*(a\mid x)$ and assign a treatment based on the estimated efficient treatment-assignment probability.
    \item \textbf{Step (2). ATE-estimation phase:} At the end of the experiment, we estimate the ATE using our proposed A2IPW estimator.
\end{itemize}

The pseudo-code is provided in Algorithm~\ref{alg}. In the following subsections, we explain the details of our experimental design.

\subsection{Step (1): Treatment-Assignment Phase}
We assign treatments in each round $t \in [T]$ to gather data. Although assigning treatments with probability $\pi^*$ minimizes the semiparametric efficiency bound, it is infeasible since we do not know the conditional variance $\sigma^2_0(a, x)$. To overcome this challenge, in each round $t$, we estimate the conditional variance $\sigma^2_0(a, x)$, estimate the efficient treatment-assignment probability $\pi^*$ using the estimator of $\sigma^2_0(a, x)$, and assign a treatment based on the estimated efficient treatment-assignment probability.

Let $T_0$ ($2 \leq T_0 \leq T$) be the number of initialization rounds, which is a constant independent of $T$. In the initialization rounds $t = 1, 2, \dots, T_0$, we assign treatment $A_t = 1$ if $t$ is odd and $A_t = 0$ if $t$ is even; for example, if $T_0 = 6$, $(A_1, A_2, A_3, A_4, A_5, A_6) = (1, 0, 1, 0, 1, 0)$. We set $\pi_t(1\mid X_t, \mathcal{H}_{t-1}) = 1/2$ for all $t = 1, 2, \dots, T_0$. 

In each round $t \in \{T_0 + 1, T_0 + 2, \dots, T\}$, we construct a consistent estimator $\widehat{\sigma}^2_t(a, x)$ of $\sigma^2_0(a, x)$ such that $\widehat{\sigma}^2_t(a, x) \in (0, \infty)$ for all $a \in \{1, 0\}$ and $x \in \mathcal{X}$, and $\widehat{\sigma}^2_t(a, x)$ is constructed only by using $\mathcal{H}_{t-1}$. 
The reason we use only $\mathcal{H}_{t-1}$ is to construct an ATE estimator whose scores consist of a martingale difference sequence, as shown in the next subsection. Under this property, we can apply the martingale central limit theorem and martingale concentration inequality to analyze the asymptotic and non-asymptotic behaviors of the ATE estimator. 

To estimate $\sigma^2_0(a, X_t)$, we propose estimating $f_0(a, X_t) = \mathbb{E}[Y_t(a) \mid X_t]$ and $e_0(a, X_t) = \mathbb{E}[Y_t^2(a) \mid X_t]$ using nonparametric models based on observations $\mathcal{H}_{t-1}$ up to round $t$. Let $\widehat{f}_t(a, X_t)$ and $\widehat{e}_t(a, X_t)$ denote such estimators. In MAB problems, several nonparametric estimators, such as $K$-nearest neighbor regression and Nadaraya–Watson kernel regression, have been shown to be consistent \citep{yang2002,Qian2016}. For example, given a bandwidth $h_T > 0$ and a kernel function $K:\mathcal{X} \to \mathbb{R}$, a Nadaraya-Watson estimator of $f_0(a, X_t)$ is defined as $\widehat{f}_t(a, X_t) = \frac{1}{\frac{1}{t - 1}\sum^{t-1}_{s=1}\mathbbm{1}[A_s = a]K((X_s - X_t)/h_t)}\frac{1}{t - 1}\sum^{t-1}_{s=1}Y_s\mathbbm{1}[A_s = a]K((X_s - X_t)/h_t)$. We can also estimate $e_0(a, X_t)$. By appropriately obtaining samples, we can also employ random forests \citep{WagerAthey2018} and neural networks as nonparametric estimators \citep{SchmidtHieber2020,Farrell2021}.

We then estimate $\sigma^2_0(a, X_t)$ as follows:
\[
\widehat{\sigma}^2_t = 
\begin{cases}
    \widehat{e}_t(a, X_t) - \widehat{f}^2_t(a, X_t) & \text{if } \widehat{e}_t(a, X_t) - \widehat{f}^2_t(a, X_t) > 0, \\
    \varepsilon & \text{otherwise},
\end{cases}
\]
where $\varepsilon > 0$ is a small positive constant introduced to ensure that $\widehat{\sigma}^2_t$ remains non-negative. Note that when $\sigma^2_0(a, X_t) > 0$, the term $\varepsilon$ becomes unnecessary as $t$ grows large.

We assign treatment $A_t$ with probability $\pi_t(A_t \mid X_t, \mathcal{H}_{t-1})$, defined as
\[
\pi_t(a \mid X_t, \mathcal{H}_{t-1}) = \frac{\sqrt{\widehat{\sigma}^2_t(a, X_t)}}{\sqrt{\widehat{\sigma}^2_t(1, X_t)} + \sqrt{\widehat{\sigma}^2_t(0, X_t)}} \quad \forall a \in \{1, 0\},
\]
Note that our experiment can be used in a batch setting, where we update $\pi_t(a \mid x, \mathcal{H}_{t-1})$ only at certain rounds $T_1, T_2, \dots \in \{1, \dots, T\}$. 
We require that $\pi_t(a \mid x, \mathcal{H}_{t-1}) \to \pi^*(a \mid x)$ for each $x \in \mathcal{X}$ as $t \to \infty$.\footnote{As long as this condition is satisfied, we do not need to sequentially update $\pi_t(a \mid x, \mathcal{H}_{t-1})$. This implies that we can keep $\pi_t(a \mid x, \mathcal{H}_{t-1})$ constant for several rounds and update $\pi_t(a \mid x, \mathcal{H}_{t-1})$ in specific rounds. For example, we can consider a two-stage design similar to \citet{Hahn2011}. In this case, we update $\pi_t(a \mid x, \mathcal{H}_{t-1})$ only at $T_1$. Assume that $T_1 = r T$, where $r \in (0, 1)$ is a constant independent of $T$. In rounds $1, 2, \dots, T_1$, we assign treatment $a \in \{1, 0\}$ with probability $1/2$, where $\pi_t(a \mid x, \mathcal{H}_{t-1}) = 1/2$ for all $x \in \mathcal{X}$. Afterward, we update $\pi_t(a \mid x, \mathcal{H}_{t-1})$ by estimating $\pi^*(a \mid x)$. If $\pi_t(a \mid x, \mathcal{H}_{t-1}) \to \pi^*(a \mid x)$ as $t \to \infty$ holds for all $x \in \mathcal{X}$, we can prove the same asymptotic optimality of our experimental design. To verify that $\pi_t(a \mid x, \mathcal{H}_{t-1}) \to \pi^*(a \mid x)$ as $t \to \infty$, it is sufficient to check $\pi_{T_1}(a \mid x, \mathcal{H}_{t-1}) \to \pi^*(a \mid x)$ as $T_1 \to \infty$ ($T \to \infty$).}

\subsection{Step (2): ATE-Estimation Phase}
At the end of the experiment, we construct an ATE estimator that is asymptotically normal with an asymptotic variance, achieving the semiparametric lower bound \eqref{eq:semipara_bound}. 
In adaptive experiments, due to the changing assignment probabilities, dependencies among samples can complicate the estimation process. To address this dependency problem, we propose the A2IPW estimator:
\[
\widehat{\theta}^{\mathrm{A2IPW}}_T = \frac{1}{T}\sum^T_{t=1} \Psi_t,
\]
\begin{align*}
\text{where}\ \ \Psi_t &= \Bigg(\frac{\mathbbm{1}[A_t=1]\big(Y_t - \widehat{f}_{t-1}(1, X_t)\big)}{\pi_t(1\mid X_t, \mathcal{H}_{t-1})}  - \frac{\mathbbm{1}[A_t=0]\big(Y_t - \widehat{f}_{t-1}(0, X_t)\big)}{\pi_t(0\mid X_t, \mathcal{H}_{t-1})} \\
&\ \ \ \ \ \ \ \ \ \ \ \ \ \ \ \ \ \ \ \ \ \ \ \ \ \ \ \ \ \ \ \ \ \ \ \ \ \ \ \ \ \ \ \ \ \ \ \ \ \ \ \ \ + \widehat{f}_{t-1}(1, X_t) - \widehat{f}_{t-1}(0, X_t)\Bigg).
\end{align*}

\noindent and $\widehat{f}_t(a, x)$ is an estimator of $f_0(a, x)$, constructed from $\mathcal{H}_t$. As stated in Theorem~\ref{thm:asymp_dist_A2IPW}, the asymptotic optimality of our proposed ATE estimator holds with any consistent estimator for $f_0(a, x)$, due to the unbiasedness of $\widehat{\theta}^{\mathrm{A2IPW}}_T$ for $\theta_0$. This point is also discussed in Section~4 of \citet{Kato2021adr}, our follow-up study. Additionally, consistency holds even if the estimator of $f_0(a, x)$ is inconsistent, as stated in Corollary~\ref{cor:consistency}.
Here, $\Psi_t$ is the semiparametric efficient score for ATE estimators. Regular estimators with scores $\Psi_t$ achieve the smallest asymptotic variance within the class of such estimators.

For $z_t = \Psi_t - \theta_0$, the sequence $\{z_t\}_{t=1}^T$ forms a martingale difference sequence, which means that $\mathbb{E}[z_t \mid \mathcal{H}_{t-1}] = 0$. Using this property, we will derive the theoretical results for $\widehat{\theta}^{\mathrm{A2IPW}}_T$. This construction shares a similar motivation to sample-splitting techniques in semiparametric inference \citep{klaassen1987}, including double machine learning \citep{ChernozhukovVictor2018Dmlf}.



\subsection{Stabilizations and Extensions}
While not required to obtain the asymptotic properties, here we introduce stabilization techniques that contribute to the finite-sample stabilization of the designed experiment. The above sections show that our designed experiment is asymptotically efficient in the sense that the asymptotic variance of the ATE estimator aligns with the semiparametric efficiency bound. However, such asymptotic optimality does not necessarily guarantee accurate ATE estimation in finite samples.

\textbf{The ADR estimator.}
\citet{Kato2021adr} reports that replacing the true $\pi_t$ with its estimate can paradoxically improve performance. This is because the original $\pi_t$ may take values close to zero, causing the inverse of $\pi_t$ to become large and making the A2IPW estimator unstable. By replacing $\pi_t$ with its estimate, even when the true value of $\pi_t$ is known, the A2IPW estimator can be stabilized.\footnote{Note that, unlike the classical problem regarding the use of an estimated propensity score in the IPW estimators, the asymptotic properties remain unchanged between the cases where the true $\pi_t$ is used and where $\pi_t$ is estimated when we use the AIPW estimator \citep{hirano03,Henmi2004paradox}. 
}

The ADR estimator is defined as follows:
\begin{align*}
\widehat{\theta}^{\mathrm{ADR}}_T &= \frac{1}{T}\sum_{t=1}^T \Bigg(\frac{\mathbbm{1}[A_t=1]\big(Y_t - \widehat{f}_{t-1}(1, X_t)\big)}{\widehat{g}_t(1\mid X_t)}  
- \frac{\mathbbm{1}[A_t=0]\big(Y_t - \widehat{f}_{t-1}(0, X_t)\big)}{\widehat{g}_t(0\mid X_t)} \\
&\qquad\qquad\qquad\quad \qquad\qquad\qquad\quad \qquad\qquad\qquad\quad + \widehat{f}_{t-1}(1, X_t) - \widehat{f}_{t-1}(0, X_t)\Bigg),
\end{align*}
where $\widehat{g}_t(a \mid X_t)$ is an estimator of $\pi_t(a \mid X_t, \mathcal{H}_{t-1})$ constructed from the past observations $\{(X_s, A_s, Y_s)\}_{s=1}^{t-1}$. Although this estimator is no longer unbiased, asymptotic normality holds under convergence rate conditions for $\widehat{f}_t$ and $\widehat{\pi}_t$, as well as double machine learning techniques. The theorem regarding its asymptotic normality is introduced in Proposition~\ref{prp:adr}.

\citet{Kato2021adr} refer to the sample splitting used in both the A2IPW and ADR estimators as adaptive fitting, where only past observations up to time $t$ are used to obtain the plug-in estimators for each $t$. Figure~\ref{fig:concept} illustrates the difference between cross-fitting as described in \citet{ChernozhukovVictor2018Dmlf} and our adaptive fitting approach.

\begin{figure}[t]
\begin{center}
 \includegraphics[width=100mm]{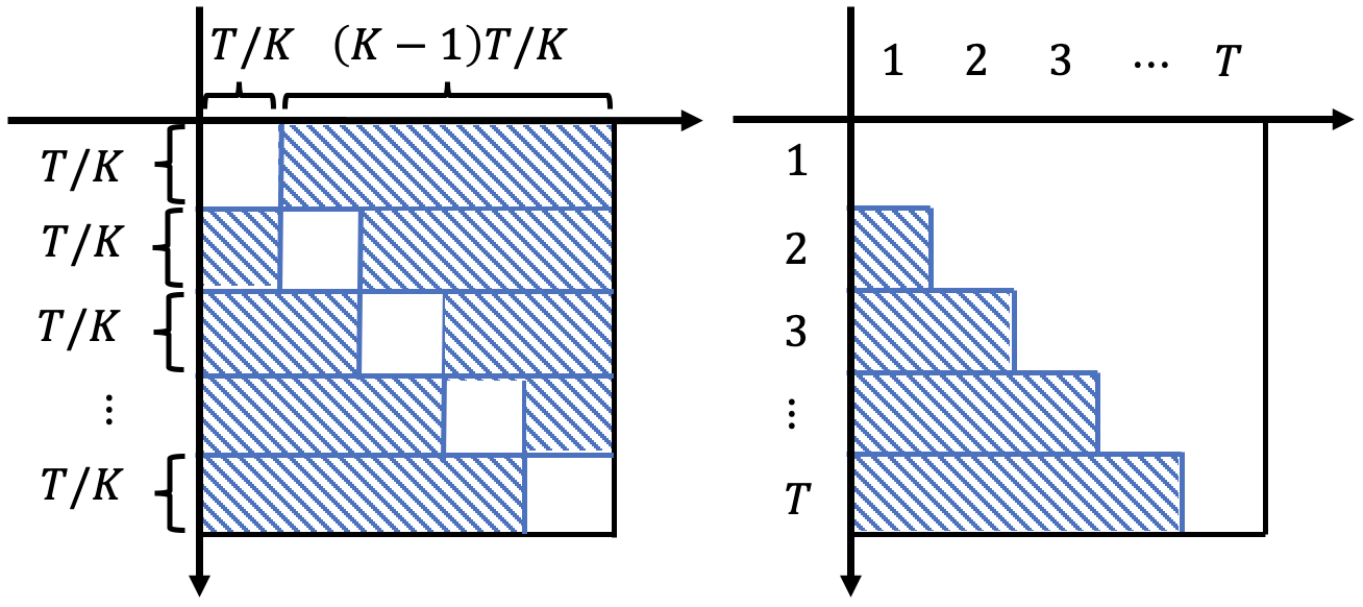}
\end{center}
\caption{The difference between $K$-fold cross-fitting (left) and adaptive-fitting (right) from Figure~1 in \citet{Kato2021adr}. The shaded block indicates the subset of observations used for estimating the nuisance parameters.}
\label{fig:concept}
\end{figure} 

\textbf{Stabilization techniques.}
To stabilize the finite-sample behavior, we can further introduce certain elements into our experiment. These elements are designed not to affect the asymptotic behavior, meaning their influence vanishes as $t \to \infty$.
\begin{description}

\item[(a)] We define the treatment-assignment probability as
\begin{align*}
    \pi_t(1 \mid x, \mathcal{H}_{t-1}) &= \gamma_t\frac{1}{2} + (1-\gamma_t)\frac{\sqrt{\widehat{\sigma}^2_{t-1}(1, x)}}{\sqrt{\widehat{\sigma}^2_{t-1}(1, x)} + \sqrt{\widehat{\sigma}^2_{t-1}(0, x)}},\\
    \pi_t(0 \mid x, \mathcal{H}_{t-1}) &= 1 - \pi_t(1 \mid x, \mathcal{H}_{t-1}),
\end{align*}
where $\gamma_t = O(1/\sqrt{t})$;

\item[(b)] As an alternative estimator, we propose the mixed A2IPW (MA2IPW) estimator, defined as $\widehat{\theta}^{\mathrm{MA2IPW}}_T = \zeta_T \widehat{\theta}^{\mathrm{IPW}}_T + (1-\zeta_T)\widehat{\theta}^{\mathrm{A2IPW}}_T$, where $\widehat{\theta}^{\mathrm{IPW}}_T$ is the IPW estimator defined as $
\widehat{\theta}^{\mathrm{IPW}}_T = \frac{1}{T}\sum^T_{t=1} \left(\frac{\mathbbm{1}[A_t=1]Y_t}{\pi_t(1\mid X_t, \mathcal{H}_{t-1})}  - \frac{\mathbbm{1}[A_t=0]Y_t}{\pi_t(0\mid X_t, \mathcal{H}_{t-1})}\right)$ and 
$\zeta_T = o(1/\sqrt{T})$. 
\end{description}
Note that the IPW estimator is a special case of the A2IPW estimator with $\widehat{f}_{t- 1}(x) = 0$. 

Technique (a) aims to stabilize the treatment assignment probability. If $\pi_t$ fluctuates significantly, it may unstabilize of the A2IPW estimator for the ATE. Furthermore, when the value of $\pi_t$ is too close to zero, its inverse is included in the elements averaged by A2IPW, potentially causing those elements to become extremely large. Technique (a) prevents such cases.\footnote{Performance can be further improved by replacing $\pi_t$ with its estimator constructed from past observations $\{(X_s, A_s, Y_s)\}_{s=1}^{t-1}$ and $X_t$, as noted by \citet{Kato2021adr}, a subsequent study to our study. \citet{Kato2021adr} observes that the A2IPW estimator with the true $\pi_t$ incurs a larger mean squared error than when using an estimated $\pi_t$. This is because when the true $\pi_t$ fluctuates significantly, the A2IPW estimator also becomes unstable. However, \citet{Kato2021adr} finds that replacing the volatile $\pi_t$ with a more stable estimator helps stabilize the behavior of the A2IPW estimator. The A2IPW estimator with an estimated $\pi_t$ is referred to as the Adaptive Doubly Robust (ADR) estimator. 
Although we do not focus on this type of stabilization in this study, it is a promising approach. We compare our estimator with the ADR estimator in our simulation studies. \citet{cook2023semiparametric} also develops stabilization techniques based on \citet{Waudby-Smith2024}.}

Technique (b) controls the estimator's behavior by avoiding situations where $\widehat{f}_{t-1}$ takes unpredictable values in the early stages. Since the nonparametric convergence rate is generally slower than $1/\sqrt{t}$, the convergence rate of $\pi_t$ to $\pi^*$ does not exceed $O(1/\sqrt{t})$. Therefore, $\gamma_t = O(1/\sqrt{t})$ does not asymptotically affect the convergence rate of the treatment-assignment probability. Similarly, the asymptotic distribution of $\widehat{\theta}^{\mathrm{MA2IPW}}_T$ is asymptotically equivalent to $\widehat{\theta}^{\mathrm{A2IPW}}_T$ because it holds that
$
\sqrt{T} \widehat{\theta}^{\mathrm{MA2IPW}}_T = \sqrt{T}\big(\zeta_T \widehat{\theta}^{\mathrm{IPW}}_T + (1-\zeta_T)\widehat{\theta}^{\mathrm{A2IPW}}_T\big)=\sqrt{T} \widehat{\theta}^{\mathrm{A2IPW}}_T + o(1)
$
as $T \to \infty$. 

There are additional stabilization techniques. For example, 
\citet{cook2023semiparametric} also develops stabilization techniques based on \citet{Waudby-Smith2024}.

\begin{algorithm}[tb]
   \caption{Adaptive experiment for efficient ATE estimation.}
   \label{alg}
\begin{algorithmic}
   \STATE {\bfseries Parameter:} The number of initialization rounds, $T_0$. The lower bound of the variance $\nu$, $\underline{\nu}> 0$. The stabilization parameter $\gamma_t, \zeta_T \in (0, 1)$, such that $\gamma_t = O(1/\sqrt{t})$ and $\zeta_T = o(1/\sqrt{T})$.
   \STATE {\bfseries Initialization:} 
   \STATE At $t=1, 2, \dots, T_0$, assign treatment $A_t= 1$ if $t$ is odd and assign treatment $A_t = 2$ if $t$ is even. Set $\pi_t(a \mid X_t, \mathcal{H}_{t-1})=1/2$ for all $a\in\{1, 0\}$.
   \FOR{$t=T_0 + 1$ to $T$}
   \IF{$t < \rho$}
   \STATE Set $\pi_t(1 \mid X_t, \Omega_{t-1}) = 0.5$.
   \ELSE
   \STATE Construct estimators $\widehat{f}_{t-1}$ and $\widehat{e}_{t-1}$ using a nonparametric method.
   \STATE Construct $\widehat{\nu}_{t-1}$ from $\widehat{f}_{t-1}$ and $\widehat{e}_{t-1}$.
   \STATE Using $\widehat{\nu}_{t-1}$, construct an estimator of $\pi^*(k \mid X_t)$ and set it as $\pi_t(k \mid X_t, \Omega_{t-1})$.
   \ENDIF
   \STATE Draw $\xi_t$ from the uniform distribution on $[0,1]$. 
   \STATE $A_t = \mathbbm{1}[\xi_t \leq \pi_t(1 \mid X_t, \Omega_{t-1})]$. 
   \ENDFOR
   \STATE Estimate the ATE by using the A2IPW estimator $\widehat{\theta}^{\mathrm{A2IPW}}_T$.
\end{algorithmic}
\end{algorithm}

\section{Theoretical Results about Treatment Effect Estimation}
\label{sec:estimator_prop}
This section provides theoretical results on the A2IPW estimator. We present its asymptotic distribution, a regret bound, and a non-asymptotic confidence bound for the A2IPW estimator.

\subsection{Consistency and Asymptotic Normality of the A2IPW Estimator}
We first show the asymptotic normality of the A2IPW estimator $\widehat{\theta}^{\mathrm{A2IPW}}_T$. Before showing the asymptotic normality, we make the following assumption. 

\begin{assumption}[Boundedness]
\label{asm:boundedness}
There exists an absolute constant $C$ such that $|Y_t(a)| \leq C$ holds for $a\in\{1, 0\}$. 
\end{assumption}

\noindent The following theorem states the asymptotic normality.
\begin{theorem}[Asymptotic distribution of the A2IPW estimator]
\label{thm:asymp_dist_A2IPW}
Suppose that Assumption~\ref{asm:boundedness} holds, and 
\begin{description}
\item[(i)] point convergence in probability of $\widehat{f}_{t-1}$ and $\pi_t$, i.e., for all $x\in\mathcal{X}$ and $a\in\{0,1\}$, 
\begin{align*}
&\widehat{f}_{t-1}(a, x)-f_0(a, x)\xrightarrow{\mathrm{p}}0\quad \mathrm{and}\quad \pi_t(a\mid  x, \mathcal{H}_{t-1})-\widetilde{\pi}(a\mid  x)\xrightarrow{\mathrm{p}}0,
\end{align*}
where $\widetilde{\pi} \in \Pi$;
\item[(ii)] there exists a constant $C_f$ such that $|\widehat{f}_{t-1}| \leq C_f$.
\end{description}
Then, the A2IPW estimator is asymptotically normal:  $$\sqrt{T}\left(\widehat{\theta}^{\mathrm{A2IPW}}_T-\theta_0\right)\xrightarrow{d}\mathcal{N}\left(0, V\right),$$
where 
$$V \coloneqq \mathbb{E}\left[\frac{\widetilde{\sigma}^2\big(1, X_t\big)}{\widetilde{\pi}(1\mid  X_t)} + \frac{\widetilde{\sigma}^2\big(0, X_t\big)}{\widetilde{\pi}(0\mid  X_t)} + \Big(f_0(1, X_t) -  f_0(0, X_t) - \theta_0\Big)^2\right].$$
\end{theorem}

The asymptotic variance aligns with the semiparametric efficiency bound derived under the treatment-assignment probability $\widetilde{\pi}$. 
Note that we do not have to impose the Donsker condition, similar to cross-fitting \citep{klaassen1987,ZhengWenjing2011CTME,ChernozhukovVictor2018Dmlf}. 
Here, we do not impose the convergence rate of $\widehat{f}_{t-1}$ owing to the unbiasedness of the A2IPW estimator $\widehat{\theta}^{\mathrm{A2IPW}}_T$ for the ATE $\theta_0$. 


Consistency holds under a weaker assumption, i.e., even if the treatment-assignment probability $\pi_t$ does not converge. 
We omit the proof because it follows from the boundedness of $z_t$ and the weak law of large numbers for a martingale difference sequence (Proposition~\ref{prp:mrtgl_WLLN} in Appendix~\ref{appdx:prelim}).

\begin{corollary}
    [Consistency of the A2IPW estimator]
\label{cor:consistency}
Suppose that there exists a constant $C_f$ such that $|\widehat{f}_{t-1}| \leq C_f$. Then, under Assumption~\ref{asm:boundedness}, $\widehat{\theta}^{\mathrm{A2IPW}}_T\xrightarrow{\mathrm{p}} \theta_0$ holds as $T\to\infty$.
\end{corollary}

Note that Corollary~\ref{cor:consistency} holds even if $\widehat{f}_t$ is inconsistent. Therefore, compared to Theorem~\ref{thm:asymp_dist_A2IPW}, Corollary~\ref{cor:consistency} holds with a weaker assumption.

We also present the theorem about the asymptotic normality of the ADR estimator from \citet{Kato2021adr}, which is a follow-up study that investigates the A2IPW estimator and generalizes it as the ADR estimator. 

\begin{proposition}[Asymptotic distribution of the ADR estimator. From Theorem~1 in \citet{Kato2021adr}.]
\label{prp:adr}
Suppose that Assumption~\ref{asm:boundedness} holds, and 
\begin{description}
\item[(i)] For all $x\in\mathcal{X}$ and $a\in\{0,1\}$, there exist $p, q > 0$ such that $p+q = 1/2$, it holds that $|\widehat{g}_{t-1}(a\mid x) -\widetilde{\pi}(a\mid  x)|=\op(t^{-p})$, and $|\widehat{f}_{t-1}(a, x)-f_0(a, x)|=\op(t^{-q})$
where $\widetilde{\pi} \in \Pi$;
\item[(ii)] There exists a constant $C_f$ such that $|\widehat{f}_{t-1}| \leq C_f$.
\end{description}
Then, the ADR estimator is consitent and asymptotically normal:  $$\sqrt{T}\left(\widehat{\theta}^{\mathrm{ADR}}_T-\theta_0\right)\xrightarrow{d}\mathcal{N}\left(0, V\right).$$
\end{proposition}

\subsection{Regret Bound of the A2IPW Estimator} 
In addition to the above asymptotic analysis, we introduce the finite-sample regret framework often used in the literature on the MAB problem. We define regret based on the MSE. We define the optimal experiment $\Pi^{\mathrm{OPT}}$ as an experiment that chooses a treatment with the probability $\pi^*$ defined in Proposition \ref{optprob}, and an estimator $\widehat{\theta}^\mathrm{OPT}_T$ with oracle $f_0$ as 
\begin{multline*}
\widehat{\theta}^\mathrm{OPT}_T=\frac{1}{T}\sum^T_{t=1}\Bigg(\frac{\mathbbm{1}[A_t=1]\big(Y_t - f_0(1, X_t)\big)}{\pi^*(1\mid  X_t)} - \frac{\mathbbm{1}[A_t=0]\big(Y_t - f_0(0, X_t)\big)}{1-\pi^*(1\mid  X_t)}\\
+  f_0(1, X_t) - f_0(0, X_t)\Bigg).
\end{multline*}
For any experiment $\Pi$ adapted by the experimenter, we define the regret of $\Pi$ as 
\[{\tt regret} = \mathbb{E}_{\Pi}\left[\left(\theta_0 - \widehat{\theta}^{\mathrm{A2IPW}}_T \right)^2\right] -\mathbb{E}_{\Pi^{\mathrm{OPT}}}\left[\left(\theta_0-\widehat{\theta}^{\mathrm{OPT}}_T\right)^2\right],\]
where the expectations are taken over each experiment. The following theorem provides an upper bound on the regret. 
\begin{theorem}[Regret Bound of A2IPW]
\label{thm:regret}
Suppose that there exists a constant $C_f$ such that $|\widehat{f}_{t-1}| \leq C_f$. Then, under Assumption~\ref{asm:boundedness}, there exist constants $C > 0$ and $T_0$ such that for all $T > T_0$, it holds that
\begin{align*}
{\tt regret} &\leq \frac{C}{T^2}\sum_{a\in\{1,0\}}\sum^T_{t=1}\Bigg(\mathbb{E}\left[\Big| \sqrt{\pi^*(a\mid  X_t)} - \sqrt{\pi_t(a\mid  X_t, \mathcal{H}_{t-1})} \Big|\right]\\
&\ \ \ \ \ \ \ \ \ \ \ \ \ \ \ \ \ \ \ \ \ \ \ \ \ \ \ \ \ \ \ \ \ \ \ \ \ \ \ \ \ \ \ \ \ \ \ \ + \mathbb{E}\left[\Big| f_0(a, X_t) - \widehat{f}_{t-1}(a, X_t) \Big|\right]\Bigg),
\end{align*}
where the expectation is taken over the random variables including $\mathcal{H}_{t-1}$. 
\end{theorem}
The proof is shown in Appendix~\ref{appdx:thm:regret}. This result tells us that regret is bounded by $o(1/T)$ under the consistencies of $\pi_t$ and $\hat f_t$, since under the consistencies and uniform integrability, as $T \to \infty$, it holds that
\begin{align*}
&\sum_{a\in\{1,0\}}\sum^T_{t=1}\Bigg(\mathbb{E}\left[\Big| \sqrt{\pi^*(a\mid  X_t)} - \sqrt{\pi_t(a\mid  X_t, \mathcal{H}_{t-1})} \Big|\right]\\
&\ \ \ \ \ \ \ \ \ \ \ \ \ \ \ \ \ \ \ \ \ \ \ \ \ \ \ \ \ \ \ \ \ \ \ \ \ \ \ \ \ \ \ \ \ \ \ \ + \mathbb{E}\left[\Big| f_0(a, X_t) - \widehat{f}_{t-1}(a, X_t) \Big|\right]\Bigg) = o(T).
\end{align*}
By contrast, if we use a constant value for $\pi_t$, regret is $O(1/T)$.
The regret bound for finite samples can also be obtained by substituting the finite sample bounds of 
\[\mathbb{E}\left[\Big| \sqrt{\pi^*(a\mid  X_t)} - \sqrt{\pi_t(a\mid  X_t, \mathcal{H}_{t-1})} \Big|\right]\]
and $\mathbb{E}\left[\Big| f_0(a, X_t) - \widehat{f}_{t-1}(a, X_t) \Big|\right]$. We can bound $\widehat{f}_{t-1}(a, X_t)$ and $\sqrt{\pi_t(a\mid  X_t, \mathcal{H}_{t-1})}$ by the same argument as existing work on the MAB problem such as \citet{yang2002}.

\subsection{Any Time Confidence Interval}
\label{sec:confidence}

The asymptotic normality shown in the previous section holds for large fixed $T$. In this section, we consider a confidence interval that is valid for any $t \in [T]$. This type of anytime confidence interval guarantees a finite sample estimation error and plays an important role in sequential hypothesis testing. 

Among the various candidates for constructing confidence intervals, we employ concentration inequalities based on the LIL. The LIL is originally derived as an asymptotic property of independent random variables by \citet{Khintchine1924} and \citet{Kolmogoroff1929}. Following their methods, several works have derived an asymptotic LIL for a martingale difference sequence under some regularity conditions \citep{Stout1970,Fisher1992}. \citet{Balsubramani2016} derived a non-asymptotic LIL-based concentration inequality for sequential testing. The reason for using the LIL-based concentration inequality is that sequential testing with the LIL-based confidence sequence requires a smaller sample size needed to identify the parameter of interest since the confidence intervals depend on the distributional information and are said to be tight \citep{Jamieson2014}, as explained later. Due to the tightness of the inequality, LIL-based concentration inequalities have been widely accepted in sequential testing \citep{Balsubramani2016} and in the best arm identification in the Multi-Armed Bandit (MAB) problem \citep{Jamieson2014,NIPS2018_7624}.

We construct the confidence sequence $\big\{q_t\big\}_{t\in\mathbb{N}}$ based on the LIL-based concentration inequality for the A2IPW estimator as follows.

\begin{theorem}[Concentration Inequality of the A2IPW Estimator]
\label{thm:conc_A2IPW} Suppose that the null hypothesis is correct; that is, $\mu = \theta_0$ and $z_t = \Psi_t - \theta_0$. Let $C > 0$ and $C_z > 0$ be constants independent of $t$ and $T$ such that $|z_t| \leq C$ and $|(z_t - z_{t-1})^2 - \mathbb{E}[(z_t - z_{t-1})^2|  \mathcal{H}_{t-1}]| \leq C_z$ hold. For any $\delta$, with probability $\geq 1-\delta$, for all $t\geq T_0$ simultaneously, 
\begin{align*}
&\left|\sum^t_{i=1}z_i\right| = t\left|\widehat{\theta}^{\mathrm{A2IPW}}_t-\theta_0\right| \leq \frac{2C}{e^2}\left(C_0(\delta) + \sqrt{2C_1\widehat{V}^*_t\left(\log\log\widehat{V}^*_t + \log\left(\frac{4}{\delta}\right)\right)}\right),
\end{align*}
where $\widehat{V}^*_t = C_f\left(\frac{e^4}{4C^2}\sum^t_{i=1}z^2_i + \frac{2C_0(\delta)C_z}{e^2}\right)$, $C_0(\delta)=3(e-2) + 2\sqrt{\frac{173}{2(e-2)}} \log\left(\frac{4}{\delta}\right)$, $C_1 = 6(e-2)$ and $C_f$ is an absolute constant.
\end{theorem}
The proof is provided in Appendix~\ref{appdx:thm:conc_A2IPW}. This result, derived by applying the findings of \citet{Balsubramani2014SharpFI}, not only shows an anytime confidence interval but also establishes a finite-sample estimation error bound in estimating $\theta_0$. 

We obtain confidence sequences, $\big\{q_t\big\}^T_{t=1}$, with the Type I error at $\alpha$ from the results of Theorem~\ref{thm:conc_A2IPW} and \citet{Balsubramani2016} as
$$
q_t \propto \log \left(\frac{1}{\alpha}\right) + \sqrt{2\sum^t_{i=1}z^2_i\left(\log\frac{\log \sum^t_{i=1}z^2_i}{\alpha}\right)}.
$$
\citet{Balsubramani2016} proposes using constant $1.1$ to specify $q_t$, namely,
$$
q_t = 1.1\left(\log \left(\frac{1}{\alpha}\right) + \sqrt{2\sum^t_{i=1}z^2_i\left(\log\frac{\log \sum^t_{i=1}z^2_i}{\alpha}\right)}\right).
$$
This choice is motivated by the asymptotic property of the LIL such that
$$
\limsup_{t\to\infty}\frac{\left|t\widehat{\theta}^{\mathrm{A2IPW}}_t - t\theta_0\right|}{\sqrt{2\widetilde{V}^*_t\left(\log\log\widetilde{V}^*_t\right)}} = 1
$$
with probability $1$ for sufficiently large samples \citep{Stout1970,Balsubramani2016}, where $\tilde{V}^2_t = \sum^t_{i=1}\mathbb{E}[z^2_i\mid \mathcal{H}_{i-1}]$, as well as the empirical results of \citet{Balsubramani2016}. 

The confidence interval tightly depends on the underlying distribution through the variances, in contrast to confidence intervals that rely on less information, such as those based on Hoeffding's inequality, which only uses the boundedness of the outcomes. Additionally, the tightness of the confidence intervals is also guaranteed by the lower bounds for sequential testing, as discussed in \citet{Jamieson2014}. It is known that the $O(\sqrt{t^{-1}\log\log t})$ asymptotic rate of the confidence intervals aligns with the lower bound implied by the LIL \citep{Farrell1964}. Non-asymptotic bounds of this form are referred to as finite LIL bounds \citep{Howard2020TimeuniformNN}.

\section{Hypothesis Testing}
\label{sec:analysis}
This section studies hypothesis testing about the ATE. 
We begin by formulating the hypothesis testing framework, introduce the testing procedures, and conclude by presenting the theoretical properties. We demonstrate how to compute the required sample size for hypothesis testing. The pseudo-code for our experimental design incorporating hypothesis testing is provided in Algorithm~\ref{alg2}, which encompasses Algorithm~\ref{alg}.

\begin{algorithm}[tb]
   \caption{Adaptive experiment for efficient ATE estimation with hypothesis testing.}
   \label{alg2}
\begin{algorithmic}
   \STATE {\bfseries Parameter:} The number of initialization rounds, $T_0$. The lower bound of the variance $\nu$, $\underline{\nu}> 0$. The stabilization parameter $\gamma_t, \zeta_t \in (0, 1)$, such that $\gamma_t = O(1/\sqrt{t})$ and $\zeta_t = o(1/\sqrt{t})$. Type~I error $\alpha$. Set $\rho\geq 0$, which is the number of samples that we assign treatments with equal probability.
   \STATE {\bfseries Initialization:} 
   \STATE At $t=1, 2$, select $A_t=t-1$. Set $\pi_t(1 \mid X_t, \Omega_{t-1})=1/2$.
   \FOR{$t=3$ to $T$}
   \IF{$t < \rho$}
   \STATE Set $\pi_t(1 \mid X_t, \Omega_{t-1}) = 0.5$.
   \ELSE
   \STATE Construct estimators $\widehat{f}_{t-1}$ and $\widehat{e}_{t-1}$ using a nonparametric method.
   \STATE Construct $\widehat{\nu}_{t-1}$ from $\widehat{f}_{t-1}$ and $\widehat{e}_{t-1}$.
   \STATE Using $\widehat{\nu}_{t-1}$, construct an estimator of $\pi^*(k \mid X_t)$ and set it as $\pi_t(k \mid X_t, \Omega_{t-1})$.
   \ENDIF
   \STATE Draw $\xi_t$ from the uniform distribution on $[0,1]$. 
   \STATE $A_t = \mathbbm{1}[\xi_t \leq \pi_t(1 \mid X_t, \Omega_{t-1})]$. 
   \IF{Sequential testing based on LIL}
   \STATE Construct $\widehat{\theta}^{\mathrm{A2IPW}}_t$.
   \STATE Construct $q_t$ based on \eqref{sec:confidence} with $\alpha$.
   \IF{$t\widehat{\theta}^{\mathrm{A2IPW}}_t > q_t$}
   \STATE Reject the null hypothesis. 
   \ENDIF
   \ENDIF
   \IF{Sequential testing based on BF correction}
   \STATE Construct $\widehat{\theta}^{\mathrm{A2IPW}}_t$.
   \STATE Construct p-value from $\widehat{\theta}^{\mathrm{A2IPW}}_t$ under BF correction.
   \IF{If the p-value is less than $\alpha$}
   \STATE Reject the null hypothesis.
   \ENDIF
   \ENDIF
   \ENDFOR
   \IF{Standard hypothesis testing}
   \STATE Construct $\widehat{\theta}^{\mathrm{A2IPW}}_T$.
   \STATE Construct p-value from $\widehat{\theta}^{\mathrm{A2IPW}}_T$.
   \IF{If the p-value is less than $\alpha$}
   \STATE Reject the null hypothesis.
   \ENDIF
   \ENDIF
\end{algorithmic}
\end{algorithm} 

\subsection{Hypothesis Testing in Adaptive Experiments}
\label{appdx:standard}
The experimenter aims to decide whether to reject the null hypothesis $H_0$ in \eqref{eq:hypothesis} while maximizing the power and controlling the Type~I error. In adaptive experiments, hypothesis testing can be framed in two ways: single-stage testing and sequential hypothesis testing. In single-stage testing, the test is performed only at the end of the experiment ($t = T$). In sequential hypothesis testing, the test is conducted sequentially at each stage of the experiment, where the sample size is treated as a stopping time (a random variable). Sequential testing is expected to reduce the required sample size by allowing the experiment to stop earlier. 

For each $t \in [T]$, let $\widehat{\theta}_T$ be an ATE estimator constructed using $\mathcal{H}_T$. In single-stage testing, we fix a threshold $p_T \in \mathbb{R}^+$ before gathering data via the experiment. At the end of the experiment, we reject the null hypothesis if:
\[
T\left|\widehat{\theta}_T - \mu\right| > p_T.
\]
\noindent We can conduct the most powerful test in single-stage testing by using the $t$-test with our efficient ATE estimator, which is asymptotically normal. 

In sequential testing, we define thresholds $q_t \in \mathbb{R}^+$ for each $t \in [T]$. At each time $t \in [T]$, we reject the null hypothesis if:
\[
t\left|\widehat{\theta}_t - \mu\right| \geq q_t.
\]

\noindent The difference between single-stage and sequential testing is illustrated in Figure~\ref{fig:gentests}.
\begin{figure*}[ht!]
    \begin{minipage}[t]{.5\linewidth}
    \vspace{0pt}
\begin{algorithmic}[1]
   \STATE Fix $T$ and compute $p_T$.
   \IF{ $T\left|\widehat{\theta}^{\mathrm{A2IPW}}_T-\mu\right|  > p_T$}
   \STATE Reject $H_0$.
   \ELSE
   \STATE Fail to reject $H_0$.
   \ENDIF
\end{algorithmic}
    \end{minipage}
    \hfill
    \begin{minipage}[t]{.5\linewidth}
\begin{algorithmic}[1]
   \STATE Fix $T$.
   \FOR{$t = 1$ {\bfseries to} $T$}
   \STATE Compute $q_t$.
   \IF{$t\left|\widehat{\theta}_t - \mu\right| > q_t$}
   \STATE Reject $H_0$.
   \ENDIF
   \ENDFOR
   \STATE Fail to reject $H_0$.
\end{algorithmic}
    \end{minipage}
\caption{Single-stage (left) and sequential (right) hypothesis testing (from Figure~1 in \citet{Balsubramani2016}).}
\label{fig:gentests}
\end{figure*}

In single-stage testing, it is natural to analyze the power of the test. In contrast, in sequential testing, we focus on the expected sample size (stopping time). Both are related to the minimum required sample size under Type I error control.

\textbf{Controlling the Type~I and Type~II errors.}  
Recall that our null and alternative hypotheses are $H_0: \theta_0 = \mu$ and $H_1: \theta_0 \neq \mu$, respectively. 
Let $\mathbb{P}_{H_0}$ and $\mathbb{P}_{H_1}$ represent the probabilities when the null and alternative hypotheses are correct, respectively. When $\mathbb{P}_{H_0}\big(\text{reject}\ H_0\big)\leq \alpha$, we say that we control the Type I error at level $\alpha$. Similarly, when $\mathbb{P}_{H_1}\big(\text{reject}\ H_0\big)\geq 1-\beta$, we say that we control the Type II error, where $\beta$ is also referred to as the power of the test.

We set $p_T$ and $q_t$ to control both the Type~I and Type~II errors. We use asymptotic normality to construct $p_T$ and the LIL-based concentration inequality to construct $\{q_t\}_{t=1}^T$. 
Controlling errors in sequential testing is more complex than in single-stage testing. If we naively apply standard single-stage testing at each $t$ sequentially, the probability of a Type I error increases due to the multiple testing problem \citep{Balsubramani2016}.\footnote{By contrast, the probability of a Type II error does not increase in sequential testing \citep{Balsubramani2016}, though there are methods to control the Type II error more precisely \citep{NIPS2018_7624}.} A common approach to this problem is to apply multiple testing corrections, such as the Bonferroni (BF) or Benjamini–Hochberg procedures. However, these methods tend to be overly conservative, resulting in suboptimal outcomes when conducting many tests. 
To avoid this issue in sequential hypothesis testing, we employ the non-asymptotic anytime confidence interval derived in Section~\ref{sec:confidence}, which holds for any time $t$ \citep{1512.04922,Howard2020TimeuniformNN}.

\textbf{Sample size and stopping time.}
We are interested in determining the sample size required to reject the null hypothesis while controlling the Type II error at level $\beta$, assuming the alternative hypothesis $H_1$ is true.

To control the Type II error, we introduce a parameter $\Delta > 0$, commonly referred to as the \emph{effect size} in hypothesis testing literature. We redefine the alternative hypothesis as $H_1(\Delta): |\theta_0 - \mu| > \Delta$, where $\mathbb{P}_{H_1(\Delta)}$ represents the probability when the alternative hypothesis $H_1(\Delta)$ is correct. Let $R_n$ denote the rejection region for controlling the Type II error at level $\beta$ given $n$ observations. In other words, when $\widehat{\theta}^{\mathrm{A2IPW}}_n \in R_n$ and the alternative hypothesis $H_1$ is true, the null hypothesis is rejected with a probability of at least $1 - \beta$. For $\Delta$ and $\beta$, the minimum sample size required to control the Type II error at $\beta$ is defined as:
\[
n^*_{\beta}(\Delta) = \min\left\{n: \mathbb{P}_{H_1(\Delta)}\left(\widehat{\theta}^{\mathrm{A2IPW}}_n \in R_n\right) \geq 1 - \beta\right\}.
\]

In single-stage testing, we can compute the sample size $n^*_{\beta}$ by using the asymptotic distribution of $\widehat{\theta}^{\mathrm{A2IPW}}_T$. See Section~\ref{sec:minimum_sample}. Note that to compute $n^*_{\beta}(\Delta)$, we need to know the conditional variance $\sigma^2_0(a, x)$ to calculate $V$ in Theorem~\ref{thm:asymp_dist_A2IPW}. In practice, conjectured values or upper bounds of the conditional variance or $V$ can be used. It is important to note that as the conditional variance or $V$ increases, the required sample size also increases.

In sequential testing, the sample size corresponds to the stopping time when the algorithm stops after rejecting the null hypothesis. Letting $\tau$ denote the stopping time, we evaluate the expected value of $\tau$, which is also referred to as the \emph{sample complexity}. In Theorem~\ref{thm:exp_sample}, we show that the sequential test is essentially as powerful as a batch test with a sample size of $T$.

\subsection{Implementation}
Let $\alpha \in (0, 1)$ be the target Type-I error, and the experimenter aims to perform hypothesis testing without a Type-I error exceeding $\alpha$.

\textbf{Single-stage testing.}  
When our interest lies in single-stage testing (at the end of the experiment), we utilize the asymptotic normality of $\widehat{\theta}^{\mathrm{A2IPW}}_T$ in Theorem~\ref{thm:asymp_dist_A2IPW}:
\[
\sqrt{T}\left(\widehat{\theta}^{\mathrm{A2IPW}}_T - \theta_0\right) \xrightarrow{\mathrm{d}} \mathcal{N}\big(0, V\big).
\]
In this case, we apply the (asymptotic) \emph{Student's $t$-test} using the $t$-statistic $\frac{\widehat{\theta}^{\mathrm{A2IPW}}_T - \mu}{\sqrt{\widehat{V}/T}}$, where $\widehat{V}$ is a consistent estimator of $V$. 

If the null hypothesis (i.e., $\theta_0=0$) is true, the $t$-statistic asymptotically follows the standard normal distribution. Based on these results, the test rejects the null hypothesis when
\[
T\left|\widehat{\theta}^{\mathrm{A2IPW}}_T - \mu\right| > \sqrt{T\widehat{V}}z_{1-\alpha/2} := p_T,
\]
where $z_\alpha$ is the $\alpha$ quantile of the standard normal distribution. When the sample size $T$ is large, the Type I error is controlled as 
\[
\mathbb{P}_{H_0}\left(T\left|\widehat{\theta}^{\mathrm{A2IPW}}_T - \mu \right| > p_T\right) \leq \alpha.
\]

\textbf{Sequential testing.}  
In sequential testing, we construct a confidence interval using a LIL-based concentration inequality, as shown in Theorem~\ref{thm:conc_A2IPW}. Based on this result, we define the confidence sequences $\{q_t\}_{t \in [T]}$ as
\[
q_t \coloneqq 1.1\left(\log \left(\frac{1}{\alpha}\right) + \sqrt{2\sum^t_{i=1}z^2_i\left(\log\frac{\log \sum^t_{i=1}z^2_i}{\alpha}\right)}\right),
\]
where $z_t \coloneqq z_t(\mu) \coloneqq \Psi_t - \mu$.

\subsection{Sample Size Computation}
In this subsection, we compute the sample size needed to control the Type~I error at $\alpha$ while achieving power $\beta$. In single-stage testing, we calculate the required minimum sample size $T$. In sequential testing, we compute the expected stopping time $\mathbb{E}[\tau]$, where $\tau$ is the stopping time when the null hypothesis is rejected.

\subsubsection{Minimum Sample Size under the Optimal Experiment}
\label{sec:minimum_sample}
First, we derive the required minimum sample size for single-stage testing. Theorem~\ref{thm:asymp_dist_A2IPW} shows that
$
\sqrt{T}\left(\widehat{\theta}_T - \theta_0\right) \xrightarrow{\mathrm{d}} \mathcal{N}(0, V)$. 
When the null hypothesis is true ($\theta_0 = \mu$),
\[
\frac{\sqrt{T}\left(\widehat{\theta}^{\mathrm{A2IPW}}_T - \mu\right)}{\sqrt{V}} \xrightarrow{\mathrm{d}}_{H_0} \mathcal{N}(0, 1).
\]
Based on these results, with sufficient samples and knowledge of $\sigma^2_0$, we reject the null hypothesis when
\[
\left| \sqrt{T}\left(\widehat{\theta}^{\mathrm{A2IPW}}_T - \mu\right) \right| > \sqrt{V} z_{1-\alpha/2},
\]
where $z_{1-\alpha/2}$ is the $1 - \alpha/2$ quantile of the standard normal distribution. As explained in Section~\ref{appdx:standard}, the Type~I error is controlled at $\alpha$. 

We now compute the smallest sample size $n^{\mathrm{OPT}*}_{\beta}(\Delta)$ required to achieve power $\beta$. The asymptotic power is given as
\begin{align*}
    &\mathbb{P}_{H_1}\left(\left|\sqrt{T}\left(\widehat{\theta}^{\mathrm{A2IPW}}_T - \mu\right)\right| > \sqrt{V} z_{1-\alpha/2}\right)\\
    &= \mathbb{P}_{H_1}\left(\sqrt{T}\left(\widehat{\theta}^{\mathrm{A2IPW}}_T - \mu\right) > \sqrt{V} z_{1-\alpha/2}\right) + \mathbb{P}_{H_1}\left(\sqrt{T}\left(\widehat{\theta}^{\mathrm{A2IPW}}_T - \mu\right) < - \sqrt{V} z_{1-\alpha/2}\right)\\
    &\approx 1 - \Phi\left(z_{1-\alpha/2} - \frac{\sqrt{T} \Delta}{\sqrt{V}}\right) + \Phi\left(-\frac{\sqrt{T} \Delta}{\sqrt{V}} - z_{1-\alpha/2}\right) \geq 1 - \Phi\left(z_{1-\alpha/2} - \frac{\sqrt{T} \Delta}{\sqrt{V}}\right).
\end{align*}
For $T \geq \frac{\sigma^2}{\Delta^2}\big(z_{1-\alpha/2} - z_{\beta}\big)^2$, the asymptotic power becomes at least $\beta$. Therefore, to achieve power $\beta$, the required sample size is:
\[
n^{\mathrm{OPT}*}_{\beta}(\Delta) = \frac{\mathbb{E}\left[\frac{\sigma^2_0(1, X_t)}{\pi^*(1 \mid X_t)} + \frac{\sigma^2_0(0, X_t)}{\pi^*(0 \mid X_t)} + \left(f_0(1, X_t) - f_0(0, X_t) - \theta_0\right)^2\right]}{\Delta^2}\big(z_{1-\alpha/2} - z_{\beta}\big)^2.
\]

\subsubsection{Expected Sample Size in Sequential Testing}
\label{sec:expect}
In this section, we calculate the upper bound of the expected stopping time $\tau$. In sequential testing using an LIL-based concentration inequality, we propose an algorithm that rejects the null hypothesis when
\[
\left|t \widehat{\theta}^{\mathrm{A2IPW}}_t - t \mu\right| > 1.1\left(\log \left(\frac{1}{\alpha}\right) + \sqrt{2\sum^t_{i=1}z^2_i\left(\log\frac{\log \sum^t_{i=1}z^2_i}{\alpha}\right)}\right) = q_t.
\]
Let $\tau$ be the stopping time of the sequential test, i.e., $\tau = \min \left\{t: \left|t \widehat{\theta}^{\mathrm{A2IPW}}_t - t \mu\right| > q_t\right\}$. When $t = \tau$, the null hypothesis is rejected.

We show that as time progresses, the probability that the sequential test does not reject the hypothesis becomes small. We bound $\mathbb{P}_{H_1}(\tau > \widetilde{t})$ for sufficiently large $\widetilde{t}$ such that $\widetilde{t} \Delta \gg 1.1\left(\log \left(\frac{1}{\alpha}\right) + \sqrt{2C^2 \widetilde{t}\left(\log\frac{\log C^2 \widetilde{t}}{\alpha}\right)}\right)$. First, we consider the probability of $\tau \geq \widetilde{t}$ for a stopping time $\tau$. The proof is shown in Appendix~\ref{appdx:proof_prob}.
\begin{lemma}
\label{lem:proof}
When the alternative hypothesis is true, $\tau > \widetilde{t}$ occurs with probability  
\begin{align*}
&\mathbb{P}_{H_1}(\tau > \widetilde{t})=O\left( \exp\left(-\frac{\widetilde{t}\Delta^2}{8C^2}\right)\right).
\end{align*}
\end{lemma}

\noindent With this lemma, we prove the following theorem. The proof is in Appendix~\ref{appdx:exp_sample}.
\begin{theorem}[Expected sample size in sequential testing]
\label{thm:exp_sample}
When the alternative hypothesis is true, if 
\[n^{\mathrm{OPT}*}_{\beta}(\Delta)\Delta \gg 1.1\left(\log \left(\frac{1}{\alpha}\right) + \sqrt{2C^2 n^{\mathrm{OPT}*}_{\beta}(\Delta)\left(\log\frac{\log C^2 n^{\mathrm{OPT}*}_{\beta}(\Delta)}{\alpha}\right)}\right),\]
then
the expected sample size in sequential testing satisfies
\begin{align*}
\mathbb{E}_{H_1}[\tau]=\left(1+\frac{8C^2}{V\big(z_{1-\alpha/2}-z_{\beta}\big)^2}\mathbb{P}_{H_1}(\tau > n^{\mathrm{OPT}*}_{\beta}(\Delta))\right)n^{\mathrm{OPT}*}_{\beta}(\Delta).
\end{align*}
\end{theorem} 
This result implies that given $\alpha$ and $\beta$, $\mathbb{E}_{H_1}[\tau]$ is approximately equal to $n^{\mathrm{OPT}*}_{\beta}(\Delta)$, multiplied by a constant term independent of $\Delta$. As $\Delta$ approaches zero, both $\mathbb{E}_{H_1}[\tau]$ and $n^{\mathrm{OPT}*}_{\beta}(\Delta)$ approach infinity. From this result, we find that the expected stopping time $\mathbb{E}_{H_1}[\tau]$ in sequential testing grows proportionally to the sample size $n^{\mathrm{OPT}*}_{\beta}(\Delta)$ in single-stage testing (\(\mathbb{E}_{H_1}[\tau] = (1 + O(1))n^{\mathrm{OPT}*}_{\beta}(\Delta)\) as \(\Delta \to 0\)). Hence, for sufficiently small $\Delta$, we can consider that $\mathbb{E}_{H_1}[\tau]$ becomes close to $n^{\mathrm{OPT}*}_{\beta}(\Delta)$. 

This result suggests that sequential testing has the potential to stop an experiment earlier than single-stage testing since the expected sample size is nearly identical to the (non-random) oracle sample size of single-stage testing, even though we do not know $n^{\mathrm{OPT}*}_{\beta}(\Delta)$ in advance of the experiment. That is, our sequential testing only uses the (unknown) minimum sample size in expectation.

Here, we emphasize that the oracle sample size $n^{\mathrm{OPT}*}_{\beta}(\Delta)$ is unknown because computing it requires the efficiency bound, which depends on the true expected conditional outcomes and the conditional variances. In single-stage testing, we cannot change the sample size during an experiment, as doing so is considered a violation of standard experimental design principles. Sequential testing, on the other hand, allows us to conduct a nearly optimal adaptive experiment without knowing $n^{\mathrm{OPT}*}_{\beta}(\Delta)$. Thus, sequential testing effectively reduces the sample size.


Note that the condition 
\[n^{\mathrm{OPT}*}_{\beta}(\Delta)\Delta \gg 1.1\left(\log \left(\frac{1}{\alpha}\right) + \sqrt{2C^2 n^{\mathrm{OPT}*}_{\beta}(\Delta)\left(\log\frac{\log C^2 n^{\mathrm{OPT}*}_{\beta}(\Delta)}{\alpha}\right)}\right)\] 
holds when $\beta$ is sufficiently close to 0. This theorem leads to the following corollary.
\begin{corollary}
Suppose that 
\[n^{\mathrm{OPT}*}_{\beta}(\Delta)\Delta \gg 1.1\left(\log \left(\frac{1}{\alpha}\right) + \sqrt{2C^2 n^{\mathrm{OPT}*}_{\beta}(\Delta)\left(\log\frac{\log C^2 n^{\mathrm{OPT}*}_{\beta}(\Delta)}{\alpha}\right)}\right)\] 
and $\pi_t = \pi^*$. Under $H_1$, for a sufficiently large sample size, the expected stopping time for the sequential test using $q_t$ is proportional to $n^{\mathrm{OPT}*}_{\beta}(\Delta)$.
\end{corollary}

\subsubsection{Minimum Sample Size and Early Stopping}
For a user-defined treatment assignment probability $\pi_t$, if $\pi_t(a\mid x)\xrightarrow{\mathrm{p}}\pi^*(a\mid x)$ holds for all $a, x$, the asymptotic variance is the same as $\widetilde{\sigma}^2$ from Theorem~\ref{thm:asymp_dist_A2IPW}. Therefore, when $\pi_t(a\mid x)\xrightarrow{\mathrm{p}}\pi^*(a\mid x)$ holds for all $a, x$, the minimum sample size required for hypothesis testing is also $n^{\mathrm{OPT}*}_{\beta}(\Delta)$. By applying the same method as in the previous section, we can verify that the expected stopping time for sequential testing under a user-defined treatment assignment probability $\pi_t$ using $q_t$ is proportional to $n^{\mathrm{OPT}*}_{\beta}(\Delta)$.

\subsection{Summary}
We introduced two approaches for hypothesis testing: single-stage testing and sequential testing. Single-stage testing employs a fixed, non-random sample size determined prior to the experiment, while sequential testing continues the experiment until a predefined stopping criterion is satisfied. Single-stage testing is justified by the asymptotic normality of the A2IPW estimator, which allows us to compute both the statistical power and the required sample size. Sequential testing, in contrast, is designed for finite-sample analysis and allows early stopping.

In contrast to single-stage testing with a fixed sample size, sequential testing has the potential to reduce sample size by terminating the experiment early. For example, if the null hypothesis assumes zero ATE but the actual ATE is significantly large, sequential testing may be able to reject the null and finish the experiment in an early round. On the other hand, in cases where the null hypothesis is not easily rejected, sequential testing may require larger sample sizes. Even in such cases, if the true ATE is sufficiently small and the sample size required for single-stage testing is large, the expected sample size for sequential testing is approximately equal to the fixed sample size used in single-stage testing.
\color{black}

\begin{table}[t]
\begin{center}
\caption{Experimental results using Dataset~1. } 
\medskip
\label{tbl:synthetic_exp_table}
\scalebox{0.9}{
\begin{tabular}{l|rrr|rrr|rr}
\toprule
 \multirow{2}{*}{} & \multicolumn{3}{c}{$T=150$} & \multicolumn{3}{c}{$T=300$} & \multicolumn{2}{c}{ST} \\
{} &     MSE &      STD &      Testing &      MSE &      STD &    Testing   &      LIL &      BF \\
\midrule
RCT &  0.145 &  0.178 &   25.0\% &  0.073 &  0.100 &   46.0\% &  455.4 &  370.4 \\
A2IPW (K-nn)&   0.085 &  0.116 &   38.4\% &  0.038 &  0.054 &   67.9\% &  389.5 &  302.8 \\
A2IPW (NW) &   0.064 &  0.092 &   51.4\% &  0.025 &  0.035 &   88.1\% &  303.8 &  239.8 \\
MA2IPW (K-nn)&  0.092 &  0.126 &   38.5\% &  0.044 &  0.058 &   66.2\% &  387.5 &  303.4 \\
MA2IPW (NW)&   { 0.062} &  0.085 &   52.7\% &  { 0.023} &  0.033 &   90.2\% &  303.3 &  236.6 \\
IPW (K-nn)&   0.151 &  0.208 &   26.1\% &  0.075 &  0.103 &   43.6\% &  446.3 &  367.0 \\
IPW (NW)&   0.161 &  0.232 &   23.4\% &  0.081 &  0.115 &   41.1\% &  446.6 &  375.0 \\
DM (K-nn)&   0.175 &  0.252 &   88.7\% &  0.086 &  0.126 &   96.1\% & 59.9 &  164.6 \\
DM (NW)&   0.111 &  0.167 &   82.1\% &  0.045 &  0.066 &   95.6\% &  119.6 &  176.2 \\
\midrule
OPT &  0.008 &  0.011 &  100.0\% &  0.004 &  0.005 &  100.0\% &   63.9 &  150.0 \\
\bottomrule
\end{tabular}
} 
\vspace{5mm}
\end{center}
\begin{center}
\caption{Experimental results using Dataset~2. } 
\medskip
\label{tbl:synthetic_exp_table_2}
\scalebox{0.9}{
\begin{tabular}{l|rrr|rrr|rr}
\toprule
 \multirow{2}{*}{} & \multicolumn{3}{c}{$T=150$} & \multicolumn{3}{c}{$T=300$} & \multicolumn{2}{c}{ST} \\
{} &     MSE &      STD &      Testing &      MSE &      STD &    Testing   &     LIL &      BF \\
\midrule
RCT &  0.084 &  0.129 &   4.7\% &  0.044 &  0.062 &   4.9\% &  497.2 &  481.8 \\
A2IPW (K-nn)&   0.050 &  0.071 &   5.6\% &  0.026 &  0.037 &   5.6\% &  497.2 &  477.3 \\
A2IPW (NW) &   0.029 &  0.045 &   4.4\% &  { 0.012} &  0.018 &   4.7\% &  496.2 &  480.6 \\
MA2IPW (K-nn)&  0.052 &  0.073 &   5.4\% &  0.025 &  0.034 &   4.7\% &  497.9 &  477.0 \\
MA2IPW (NW)&   0.032 &  0.047 &   6.3\% &  { 0.012} &  0.018 &   4.4\% &  496.6 &  475.3 \\
IPW (K-nn)&   0.088 &  0.126 &   5.6\% &  0.043 &  0.062 &   5.2\% &  495.8 &  478.1 \\
IPW (NW)&   0.094 &  0.140 &   5.8\% &  0.045 &  0.064 &   5.3\% &  495.6 &  471.6 \\
DM (K-nn)&   0.096 &  0.129 &  85.3\% &  0.046 &  0.063 &  89.5\% &   97.3 &  188.3 \\
DM (NW)&   0.054 &  0.075 &  53.7\% &  0.023 &  0.032 &  55.4\% &  312.8 &  305.3 \\
\midrule
OPT &   0.005 &  0.007 &   4.4\% &  0.002 &  0.003 &   4.4\% &  498.4 &  483.0 \\
\bottomrule
\end{tabular}
} 
\end{center}
\end{table}

\begin{table}
\begin{center}
\caption{Experimental results using Datasets~3. } 
\medskip
\label{tbl:synthetic3}
\scalebox{0.9}{
\begin{tabular}{lrrrrrrrr}
\toprule
 & \multicolumn{3}{c}{$T=150$} & \multicolumn{3}{c}{$T=300$} & \multicolumn{2}{c}{ST} \\
{} &     MSE &      STD &      Testing &      MSE &      STD &    Testing   &      LIL &      BF \\
\midrule
RCT  &  0.139 &  0.191 &   24.2\% &  0.069 &  0.102 &   44.8\% &  450.1 &  371.7 \\
A2IPW (K-nn)  &  0.089 &  0.127 &   39.0\% &  0.042 &  0.064 &   69.8\% &  385.8 &  296.6 \\
A2IPW (NW)  &  0.061 &  0.089 &   53.8\% &  { 0.024} &  0.033 &   90.3\% &  290.5 &  230.4 \\
MA2IPW (K-nn)  &  0.087 &  0.121 &   42.6\% &  0.040 &  0.054 &   70.2\% &  378.1 &  291.4 \\
MA2IPW (NW)  &  { 0.060} &  0.083 &   53.1\% &  0.025 &  0.035 &   90.8\% &  292.6 &  233.6 \\
IPW (K-nn)  &  0.158 &  0.214 &   26.3\% &  0.076 &  0.110 &   46.0\% &  443.2 &  365.6 \\
IPW (NW)  &  0.147 &  0.202 &   25.1\% &  0.080 &  0.112 &   46.1\% &  440.0 &  367.6 \\
DM (K-nn)  &  0.167 &  0.237 &   { 90.3\%} &  0.084 &  0.120 &   96.0\% &   { 57.3} &  { 162.6} \\
DM (NW)  &  0.109 &  0.156 &   83.2\% &  0.044 &  0.065 &   { 96.8\%} &  116.8 &  173.0 \\
\hline
OPT &  0.007 &  0.010 &  100.0\% &  0.003 &  0.005 &  100.0\% &   55.8 &  150.0 \\
\bottomrule
\end{tabular}
}
\end{center}
\vspace{5mm}
\begin{center}
\caption{Experimental results using Datasets~4. } 
\medskip
\label{tbl:synthetic4}
\scalebox{0.9}{
\begin{tabular}{lrrrrrrrr}
\toprule
 & \multicolumn{3}{c}{$T=150$} & \multicolumn{3}{c}{$T=300$} & \multicolumn{2}{c}{ST} \\
{} &     MSE &      STD &      Testing &      MSE &      STD &    Testing   &      LIL &      BF \\
\midrule
RCT  &  0.081 &  0.117 &   4.5\% &  0.041 &  0.056 &   { 3.5\%} &  496.3 &  484.0 \\
A2IPW (K-nn)  &  0.053 &  0.073 &   6.2\% &  0.024 &  0.035 &   5.1\% &  496.8 &  474.1 \\
A2IPW (NW) &  { 0.031} &  0.044 &   5.2\% &  0.012 &  0.017 &   6.1\% &  495.6 &  477.0 \\
MA2IPW (K-nn)  &  0.048 &  0.065 &   5.1\% &  0.024 &  0.035 &   4.9\% &  495.8 &  477.5 \\
MA2IPW (NW)  &  0.029 &  0.042 &   4.3\% &  { 0.011} &  0.015 &   4.4\% &  { 498.1} &  477.6 \\
IPW (K-nn) &  0.091 &  0.120 &   4.7\% &  0.048 &  0.067 &   6.1\% &  496.0 &  475.2 \\
IPW (NW) &  0.098 &  0.132 &   5.1\% &  0.049 &  0.066 &   5.9\% &  497.2 &  474.6 \\
DM (K-nn)  &  0.101 &  0.155 &  84.1\% &  0.049 &  0.075 &  87.2\% &  102.9 &  190.4 \\
DM (NW)  &  0.057 &  0.086 &  53.6\% &  0.023 &  0.034 &  57.6\% &  299.9 &  306.1 \\
\hline
OPT &  0.004 &  0.005 &   4.5\% &  0.002 &  0.003 &   4.5\% &  497.4 &  482.3 \\
\bottomrule
\end{tabular}
}
\end{center}
\end{table}

\begin{table*}[t]
\begin{center}
\caption{Experimental results using IHDP dataset with surface A. } 
\medskip
\label{tbl:A_exp_table}
\scalebox{0.9}{
\begin{tabular}{lrrrrrrrr}
\toprule
\multirow{3}{*}{} & \multicolumn{8}{c}{IHDP dataset with surface A, $\theta_0=4\neq 0$} \\
 & \multicolumn{3}{c}{$T=150$} & \multicolumn{3}{c}{$T=300$} & \multicolumn{2}{c}{ST} \\
{} &     MSE &      STD &      Testing &      MSE &      STD &    Testing   &      LIL &      BF \\
\midrule
RCT  &  0.674 &  1.066 &   60.4\% &  0.333 &  0.562 &   93.4\% &  355.4 &  228.0 \\
A2IPW (K-nn)  &  0.606 &  0.891 &   99.6\% &  0.310 &  0.500 &  { 100.0\%} &   86.3 &  150.5 \\
A2IPW (NW)  &  0.485 &  0.740 &   99.8\% &  { 0.202} &  0.311 &  { 100.0\%} &   76.2 &  150.2 \\
DM (K-nn)  &  1.138 &  1.745 &   99.9\% &  0.578 &  0.892 &  { 100.0\%} &   15.1 &  { 150.1} \\
DM (NW) &  0.999 &  1.427 &  { 100.0\%} &  0.454 &  0.623 &  { 100.0\%} &   26.4 &  150.0 \\
\bottomrule
\end{tabular}
}
\end{center}
\vspace{5mm}
\begin{center}
\caption{Experimental results using IHDP dataset with surface B. } 
\medskip
\label{tbl:B_exp_table}
\scalebox{0.9}{
\begin{tabular}{lrrrrrrrr}
\toprule
\multirow{3}{*}{} & \multicolumn{8}{c}{IHDP dataset with surface B, $\theta_0\neq 0$} \\
 & \multicolumn{3}{c}{$T=150$} & \multicolumn{3}{c}{$T=300$} & \multicolumn{2}{c}{ST} \\
{} &     MSE &      STD &      Testing &      MSE &      STD &    Testing   &      LIL &      BF \\
\midrule
RCT  &  4.522 &  19.635 &  53.9\% &  2.492 &   9.903 &  72.7\% &  355.3 &  274.4 \\
A2IPW (K-nn)  &  5.153 &  33.698 &  84.5\% &  2.683 &  13.545 &  90.6\% &  147.7 &  186.2 \\
A2IPW (NW)  &  4.379 &  23.713 &  84.3\% &  { 2.198} &  11.874 &  91.0\% &  142.9 &  185.0 \\
DM (K-nn)  &  7.065 &  23.954 &  { 98.1\%} &  3.892 &  14.737 &  { 98.8\%} &   { 18.7} &  { 152.1} \\
DM (NW)  &  7.410 &  30.313 &  94.1\% &  3.821 &  16.227 &  96.5\% &   53.0 &  162.6 \\
\bottomrule
\end{tabular}
}
\end{center}
\end{table*}

\section{Simulation Studies}
\label{sec:exp}
In this section, we evaluate the effectiveness of the proposed algorithm through experimental comparisons. The proposed method using the A2IPW estimator is compared against several alternative approaches, including the MA2IPW estimator, the IPW estimator, a randomized controlled trial (RCT) with a fixed treatment assignment probability of \( \pi_t(1\mid X_t, \mathcal{H}_{t-1}) = \pi_t(0\mid X_t, \mathcal{H}_{t-1}) = 0.5 \) for all \( t \), an oracle estimator \( \hat{\theta}^\mathrm{OPT}_T \) that operates under the optimal treatment-assignment probability, and a direct method (DM) estimator defined as $\frac{1}{T} \sum^T_{t=1} \left( \widehat{f}_t(1, X_t) - \widehat{f}_t(0, X_t) \right)$. 

To estimate the treatment-assignment probability and expected outcomes, we consider two cases using different nonparametric estimators: the Nadaraya–Watson (NW) estimator and the \( K \)-nearest neighbor (K-nn) estimator. For the MA2IPW estimator, we set the parameter as \( \zeta = t^{-1/1.5} \).

In Appendix~\ref{appdx:additional_exp_result}, we show simulation studies in which we compare our method using the A2IPW and the ADR estimator with the stratification tree method proposed in \citet{Meehan2022}. 

\subsection{Setting}
We conduct simulation studies using synthetic and semi-synthetic datasets. In each dataset, we perform the following three types of hypothesis testing:
\begin{itemize}
    \item Single-stage testing using a $T$-test.
    \item Sequential testing with Bonferroni (BF) correction.
    \item Sequential testing based on an adaptive confidence sequence derived from the LIL-based concentration inequality.
\end{itemize}

For all settings, the null and alternative hypotheses are given by
\[
\mathcal{H}_0: \theta_0 = 0, \quad \mathcal{H}_1: \theta_0 \neq 0.
\]
For the standard hypothesis testing, we construct confidence intervals using $T$-statistics derived from the asymptotic distribution in Theorem~\ref{thm:asymp_dist_A2IPW}. The sequential testing with BF correction is conducted at $t=150, 250, 350, 450$. For the LIL-based sequential testing, confidence intervals are constructed using $q_t$ as shown in Theorem~\ref{thm:conc_A2IPW}.

\subsection{Simulation Studies with Synthetic Dataset}
We first conduct experiments using synthetic datasets to evaluate the proposed method. At each round $t$, a covariate vector $X_t \in \mathbb{R}^5$ is generated as
\[
X_t = (X_{t1}, X_{t2}, X_{t3}, X_{t4}, X_{t5})^\top, \quad X_{tk} \sim \mathcal{N}(0, 1) \text{ for } k=1,2,3,4,5.
\]
The potential outcome model is given by
\[
Y_t(d) = \mu_d + \sum^5_{k=1} X_{tk} + e_{td},
\]
where $\mu_d$ is a constant and $e_{td}$ follows a normal distribution with standard deviation $\sigma_d$. The expectation of the potential outcome is $\mathbb{E}[Y_t(d)] = \mu_d$.

We generate four datasets, each containing $500$ units, under different settings for $\mu_d$ and $\sigma_d$:
\begin{itemize}
    \item Dataset~1: $\mu_1=0.8$, $\mu_0=0.3$, $\sigma_1=0.8$, $\sigma_0=0.3$.
    \item Dataset~2: $\mu_1=0.5$, $\mu_0=0.5$, $\sigma_1=0.8$, $\sigma_0=0.3$.
    \item Dataset~3: $\mu_1=0.8$, $\mu_0=0.3$, $\sigma_1=0.6$, $\sigma_0=0.4$.
    \item Dataset~4: $\mu_1=0.5$, $\mu_0=0.5$, $\sigma_1=0.6$, $\sigma_0=0.4$.
\end{itemize}

For each setting, we conduct $1000$ independent trials. The results are summarized in Tables~\ref{tbl:synthetic_exp_table}, \ref{tbl:synthetic3}, and \ref{tbl:synthetic4}. We report the mean squared error (MSE) between $\theta$ and $\hat{\theta}$, the standard deviation of the MSE (STD), and the rejection rates of hypothesis testing based on $T$-statistics at the $150$th (mid) and $300$th (final) rounds. Additionally, we present the stopping times for the LIL-based algorithm and the multiple testing with BF correction. If the null hypothesis is not rejected in sequential testing, the stopping time is set to $500$.

Across various datasets, the proposed algorithm achieved lower MSE compared to other methods. The DM estimator tends to reject the null hypothesis with small samples in Dataset~1 but exhibited high Type~II error in Dataset~2.

\subsection{Simulation Studies with Semi-Synthetic Data}
We also evaluated the proposed algorithm using semi-synthetic datasets constructed from the Infant Health and Development Program (IHDP). The IHDP dataset consists of simulated outcomes and covariates based on a real study, following the simulation setting proposed by \citet{doi:10.1198/jcgs.2010.08162}. The dataset contains $747$ units with $6$ continuous and $19$ binary covariates, with outcomes generated artificially.

\citet{doi:10.1198/jcgs.2010.08162} considers two response surfaces:
\begin{itemize}
    \item Response Surface~A:
    \begin{align*}
    &Y_t(0) \sim \mathcal{N}(X^\top_{t}\bm{\beta}_A, 1),\\
    &Y_t(1) \sim \mathcal{N}(X^\top_{t}\bm{\beta}_A+4, 1),
    \end{align*}
    where elements of $\bm{\beta}_A \in \mathbb{R}^{25}$ are randomly sampled from $\{0, 1, 2, 3, 4\}$ with probabilities $(0.5, 0.2, 0.15, 0.1, 0.05)$.
    \item Response Surface B:
    \begin{align*}
    &Y_t(0) \sim \mathcal{N}(\exp((X_{t} + W)^\top\bm{\beta}_B), 1),\\
    &Y_t(1) \sim \mathcal{N}(X^\top_{t}\bm{\beta}_B - q, 1),
    \end{align*}
    where $\bm{W}$ is an offset matrix with all elements equal to $0.5$, $q$ is a constant ensuring an average treatment effect of $4$, and elements of $\bm{\beta}_B$ are randomly sampled from $\{0, 0.1, 0.2, 0.3, 0.4\}$ with probabilities $(0.6, 0.1, 0.1, 0.1, 0.1)$.
\end{itemize}

For experiments, we randomly select $500$ units from the dataset. The results, summarized in Tables~\ref{tbl:A_exp_table} and \ref{tbl:B_exp_table}, include MSE, the standard deviation of MSE (STD), rejection rates at the $150$th and $300$th periods, and stopping times for the LIL-based and BF correction-based sequential testing. If the hypothesis is not rejected in sequential testing, the stopping time is set to $500$.

\subsection{Results}
The experimental results in Tables~\ref{tbl:synthetic_exp_table}--\ref{tbl:synthetic4} (synthetic data) and Tables~\ref{tbl:A_exp_table}--\ref{tbl:B_exp_table} (semi-synthetic IHDP data) reveal several notable patterns. In all datasets, the proposed adaptive algorithm using the A2IPW estimator, particularly with the Nadaraya--Watson kernel, consistently yields lower mean squared error than the baseline RCT and DM estimators. This performance gap becomes more pronounced at larger sample sizes, indicating that adaptively refining treatment-assignment probabilities based on accumulated data improves estimation accuracy.

Another observation concerns the DM estimator, which sometimes rejects the null hypothesis more readily when the true effect is clearly different from zero (as in Dataset~1). However, in scenarios where the true effect is close to zero (Dataset~2), it can fail to reject the null and thus exhibit higher Type II error. This pattern underscores that DM methods are sensitive to both sample size and the true effect magnitude and may be less robust when the treatment effect is marginal.

Standard RCT designs with constant assignment probabilities maintain unbiasedness but often show higher mean squared error relative to A2IPW. The adaptive nature of A2IPW allows it to focus allocations more efficiently, leading to more precise estimates of the treatment effect. The oracle (optimal) estimator, which knows the true assignment probabilities in advance, outperforms all other methods in terms of mean squared error and is included only to demonstrate the theoretical upper bound of performance.

The sequential testing procedures exhibit distinct behaviors. The LIL-based approach is typically conservative in practice and requires larger sample sizes before rejecting the null hypothesis, while the Bonferroni-based correction often stops earlier but can inflate Type I error. For instance, Tables~\ref{tbl:synthetic_exp_table} and~\ref{tbl:synthetic4} show cases where the Bonferroni-based method rejects the null more frequently, even when the underlying effect is subtle. Standard hypothesis testing (using a fixed sample size and T-statistics) avoids the complexity of sequential testing but does not allow the possibility of early termination.

Overall, the results suggest that the A2IPW-based adaptive design achieves lower estimation error and maintains favorable operating characteristics in both standard and sequential testing frameworks. Whether to use a sequential testing procedure depends on factors such as how rapidly decisions must be reached, the acceptable risk of false positives, and whether the total sample size can be determined in advance.

\section{Conclusion}
In this study, we designed an adaptive experimental framework to efficiently estimate the ATE and conduct hypothesis testing. We began by reviewing the semiparametric efficiency bound, which characterizes the fundamental limits of estimation efficiency as a function of the treatment-assignment probability. We then defined the efficient treatment-assignment probability as the minimizer of the semiparametric efficiency bound and leveraged this result to develop an optimal adaptive experimental design.

Our proposed method consists of two key phases: the treatment-assignment phase and the ATE-estimation phase. In the treatment-assignment phase, treatments are adaptively assigned based on an estimate of the efficient treatment-assignment probability. In the ATE estimation phase, we estimate the ATE using the proposed A2IPW estimator, which is constructed from the data collected in the treatment-assignment phase. We demonstrated that this estimator achieves asymptotic optimality by proving that its asymptotic variance matches the semiparametric efficiency bound. This optimality also ensures smaller sample sizes in hypothesis testing, improving the efficiency of experimental design.

In addition to establishing asymptotic optimality, we derived both asymptotic and non-asymptotic confidence intervals for the A2IPW estimator. The non-asymptotic bounds provide finite-sample guarantees, which are particularly useful in practical applications where sample sizes are limited. These confidence intervals enable rigorous inference while maintaining a tight dependence on the underlying data distribution.

Furthermore, we developed a hypothesis-testing framework tailored to our adaptive experimental design. We introduced two approaches: single-stage testing, which relies on a fixed sample size and asymptotic normality, and sequential testing, which dynamically determines sample size based on intermediate test results. We analyzed the theoretical properties of both approaches and highlighted scenarios where sequential testing can substantially reduce the required sample size while maintaining statistical rigor.

Our study contributes to the broader literature on adaptive experimental design by providing a theoretically grounded and practically implementable methodology for efficient ATE estimation and inference. Future research directions include extending our framework to accommodate more complex settings, such as network interference \citep{viviano2022experimentaldesignnetworkinterference}, clustered experimental designs \citep{viviano2025causalclusteringdesigncluster}, and heterogeneous treatment effects \citep{kato2024adaptivepolicylearning}. Further exploration of optimality guarantees in finite-sample regimes and their connections to best-arm identification remains an important avenue for research \citep{Kasy2021,KOCK2023624}. Additionally, incorporating reinforcement learning techniques into the treatment-assignment phase may enhance adaptability and extend the applicability of our approach to more dynamic experimental settings \citep{NathanUehara2019,adusumilli2024dynamicallyoptimaltreatmentallocation,sakaguchi2024policylearningoptimaldynamic}.

In summary, this study provides a comprehensive methodological framework for designing and analyzing adaptive experiments, ensuring both statistical efficiency and practical applicability in treatment-effect estimation and hypothesis testing.

\bibliographystyle{asa}
\bibliography{arXiv2.bbl}

\clearpage
\onecolumn

\appendix

\begin{center}
    {\bf Online Appendix}
\end{center}

\section{Estimation of \texorpdfstring{$\mathbb{E}\big[Y_t(a)\mid  x\big]$}{f} and \texorpdfstring{$\mathbb{E}\big[Y^2_t(a)\mid  x\big]$}{f}}
First, we consider how to estimate $f_0(a, x) = \mathbb{E}\big[Y_t(a)\mid  x\big]$ and $e_0(a, x) = \mathbb{E}\big[Y^2_t(a)\mid  x\big]$. When estimating $f_0(a, x) $ and $e_0(a, x)$, we need to construct consistent estimators from dependent samples obtained from an adaptive experiment. In a MAB problem, several non-parametric estimators are proved to be consistent, such as the $K$-nearest neighbor regression estimator and the Nadaraya-Watson kernel regression estimator \citep{yang2002,Qian2016}. As an example, we show the theoretical properties of the $K$-nearest neighbor regression estimator when using samples with bandit feedback in the following part.

\textbf{$K$-nearest neighbor regression:} We introduce nonparametric estimation of $f_0$ based on $K$-nearest neighbor regression using samples with bandit feedback \citep{yang2002}. For simplicity, we restrict $\mathcal{X}$ as $\mathcal{X}=[0, 1]^d$, which can be relaxed for each application.

First, we fix $x^*\in \mathcal{X}$. Let $k_n > 0$ be a value depending on the sample size $n$. Let $N_{t,k}$ be $\sum^t_{s=1}\mathbbm{1}[A_s=k]$. At $t$-th round, we gather $N_{t,k}$ samples from the case of $A_{t'}=k$ and reindex the samples as $\{(X_{t'}, Y_{t'})\}^{N_{t,k}}_{t'=1}$. We construct estimators using the $k_{N_{t,k}}$-NN regression and $\{(X_{t'}, Y_{t'})\}^{N_{t,k}}_{t'=1}$ as 
\begin{align*}
\widehat{f}_{t}(a, x^*) = \frac{1}{k_{N_{t,k}}}\sum^{k_{N_{t,k}}}_{i=1}Y_{\pi(x^*, i)},\quad\mathrm{and}\quad \widehat{e}_{t}(a, x^*) = \frac{1}{k_{N_{t,k}}}\sum^{k_{N_{t,k}}}_{i=1}Y^2_{\pi(x^*, i)},
\end{align*}
where $\pi$ is the permutation of $\{1,2,\dots,N_{t,k}\}$ such that 
\begin{align*}
\|X_{\pi(x^*, 1)}-x^*\| \leq \|X_{\pi(x^*, 2)}-x^*\|\leq \dots \leq \|X_{\pi(x^*, N_{t,k})}-x^*\|.
\end{align*}

For $\widehat{f}_{t-1}(a, x)$, \citet{yang2002} showed the following theoretical results. For simplicity, assume $\mathcal{X}=[0, 1]^d$ for an integer $d>0$.  First, they make the following assumption.
\begin{assumption}[\citet{yang2002}, Eq.~(5)]
The function $f_0(a, x)$ be continuous in $x\in\mathcal{X}$ for all $k\in\{1, 0\}$. 
\end{assumption}
Let $\psi(z; f_0(a, \cdot))$ be a modulus of continuity defined by
\begin{align*}
\psi(z; f_0(a, \cdot)) = \sup\left\{\left|f_0(a, x') - f_0(a, x'')\right|:|x' - x''|_\infty \leq z\right\}.
\end{align*}
The term $\psi$ represents the smoothness of the function $\nu_d$.
\begin{assumption}[\citet{yang2002}, Assumption~2]
The probability $p(x)$ is uniformly bounded above and away from $0$ on $\mathcal{X}=[0, 1]^d$, i.e., $\underline{c}\leq p(x)\leq \overline{c}$.
\end{assumption}
Assume $Y_t(a) = f_0(a, X_t) + \epsilon_{t, k}$, where $\epsilon_{t, k}$ is a random variable with mean $0$ and finite variance. 
\begin{assumption}[\citet{yang2002}, Assumption~3]
The error term $\epsilon_{t, k}$ also satisfies the moment condition such that there exist positive constants $v$ and $w$ satisfying, for all $m\geq 2$,
\begin{align*}
\mathbb{E}[|\epsilon_{t, k}|^m]\leq \frac{m!}{2}v^2w^{m-2}.
\end{align*}
\end{assumption}
Under these assumptions, we can show the following lemma from the result of \citet{yang2002}.
\begin{lemma}[\citet{yang2002}, Eq.~(4)]
\label{lmm:ci_condvar}
For $\kappa>0$, let $\eta_\kappa=\sup\{z: \psi(z; f_0(a, \cdot))\leq \kappa\}$. There exists a constant $M>0$ such that, for $\kappa>0$, $h<\eta_{\kappa/4}$, and $k_{N_{t,k}} \leq \underbar{c}th^k/2$, 
\begin{align*}
&\mathbb{P}\left(\left|\widehat{f}_{t}(a, x^*)-f_0(a,x^*)\right|\geq \kappa \right)\nonumber\\
&\leq M\exp\left(-\frac{3k_{N_{t,k}}}{14}\right)+\left(t^{d+2}+1\right)\left(\exp\left(-\frac{3k_{N_{t,k}}\varepsilon}{28}\right)+\exp\left(-\frac{k_{N_{t,k}}\varepsilon^2\kappa^2}{16(v^2+w\varepsilon\kappa/4)}\right)\right).
\end{align*}
\end{lemma}
According to \citet{yang2002}, for $k_t$ such that $k_t\varepsilon^2/\log t\to \infty$ and $k_{N_{t,k}}=o(t)$, we can choose $h\to 0$ that satisfies $h\geq (2k_{N_{t,k}}/(\underbar{c}t))^{1/d}$. From this discussion and the Borel-Cantelli lemma, we can show the following corollary \citep{yang2002}.
\begin{corollary}[\citet{yang2002}]
\label{cor:asconv_condvar}
For $k_t$ such that $k_t\varepsilon^2/\log t\to \infty$ and $k_{N_{t,k}}=o(t)$,
\begin{align*}
\left|\widehat{f}_{t}(a, x^*))-f_0(a,x^*)\right| \xrightarrow{\mathrm{p}} 0.
\end{align*}
\end{corollary}
Besides, when we use $k_{N_{t, k}}=O(\sqrt{t})$ in our algorithm, which satisfies $k_{N_{t,k}}\varepsilon^2/\log t\to \infty$ and $k_{N_{t,k}}=o(t)$, the following corollary holds.
\begin{corollary}
\label{cor:ci_condvar2}
For $k_t=\sqrt{t}$, there exists a constant $M>0$ such that, for $t>\left(\frac{2}{\underbar{c}\eta^k_{\kappa/4}}\right)^2$, 
\begin{align*}
&\mathbb{P}\left(\left|\widehat{f}_{t}(a, x^*)-f_0(a,x^*)\right|\geq \kappa \right)\\
&\leq M\exp\left(-\frac{3k_{N_{t,k}}}{14}\right)+\left(t^{d+2}+1\right)\left(\exp\left(-\frac{3k_t\varepsilon}{28}\right)+\exp\left(-\frac{k_{N_{t,k}}\varepsilon^2\kappa^2}{16(v^2+w\varepsilon\kappa/4)}\right)\right).
\end{align*}
\end{corollary}
Using these results, we can bound $\mathbb{E}\left[\big|\widehat{f}_{t}(a, x^*)-f_0(a,x^*)\big|\right]$ by the following lemma.
\begin{lemma}
\label{cor:ci_exp_condvar}
For $\kappa>0$, $\eta_\kappa=\sup\{z: \psi(z; v_d)\leq \kappa\}$, $k_t = \sqrt{t}$, and $t>\left(\frac{2}{\underbar{c}\eta^k_{\kappa/4}}\right)^2$, there exists a constant $M>0$ such that
\begin{align*}
&\mathbb{E}\left[\big|\widehat{f}_{t}(a, x^*)-f_0(a,x^*)\big|\right]\\
&\leq \kappa + C_2\Bigg(M\exp\left(-\frac{3k_{N_{t,k}}}{14}\right)\\
&\ \ \ \ \ \ \ \ \ \ \ \ \ \ \ \ \ \ \ \ \ +\left(t^{d+2}+1\right)\left(\exp\left(-\frac{3k_{N_{t,k}}\varepsilon}{28}\right)+\exp\left(-\frac{k_{N_{t,k}}\varepsilon^2\kappa^2}{16(v^2+w\varepsilon\kappa/4)}\right)\right)\Bigg).
\end{align*}
\end{lemma}
\begin{proof}
For $\kappa>0$, $\eta_\kappa=\sup\{z: \psi(z; v_d)\leq \kappa\}$, and $t>\left(\frac{2}{\underbar{c}\eta^m_{\kappa/4}}\right)^2$,
\begin{align*}
&\mathbb{E}\left[\big|\widehat{f}_{t}(a, x^*)-f_0(a,x^*)\big|\right]\\
&\leq \kappa + C_2\mathbb{P}\left(\big|\widehat{f}_{t}(a, x^*)-f_0(a,x^*)\big|\geq \kappa \right)\\
&\leq \kappa + C_2\Bigg(M\exp\left(-\frac{3k_{N_{t,k}}}{14}\right)\\
&\ \ \ \ \ \ \ \ \ \ \ \ \ \ \ \ \ \ \ \ \ +\left(T^{d+2}+1\right)\left(\exp\left(-\frac{3k_{N_{t,k}}\varepsilon}{28}\right)+\exp\left(-\frac{k_{N_{t,k}}\varepsilon^2\kappa^2}{16(v^2+w\varepsilon\kappa/4)}\right)\right)\Bigg).
\end{align*}
\end{proof}
The theoretical results of \citet{yang2002} is based on the assumption that the flexibility of the function is restricted and assignment probabilities are $>0$ for all treatments. Therefore, we can easily check that their results can apply to our case.

\section{Preliminaries}
\label{appdx:prelim}

\subsection{Mathematical Tools}

\begin{definition}[Uniformly Integrable, \citet{Hamilton1994}, p.~191]  
\label{dfn:uniint}
A sequence $\{A_t\}$ is said to be uniformly integrable if for every $\epsilon > 0$ there exists a number $c>0$ such that 
\begin{align*}
\mathbb{E}[|A_t|\cdot I[|A_t \geq c|]] < \epsilon
\end{align*}
for all $t$.
\end{definition}

\begin{proposition}[Sufficient Conditions for Uniformly Integrable, \citet{Hamilton1994}, Proposition~7.7, p.~191] 
\label{prp:suff_uniint}
(a) Suppose there exist $r>1$ and $M<\infty$ such that $\mathbb{E}[|A_t|^r]<M$ for all $t$. Then $\{A_t\}$ is uniformly integrable. (b) Suppose there exist $r>1$ and $M < \infty$ such that $\mathbb{E}[|b_t|^r]<M$ for all $t$. If $A_t = \sum^\infty_{j=-\infty}h_jb_{t-j}$ with $\sum^\infty_{j=-\infty}|h_j|<\infty$, then $\{A_t\}$ is uniformly integrable.
\end{proposition}

\begin{proposition}[$L^r$ Convergence Theorem, \citet{loeve1977probability}]
\label{prp:lr_conv_theorem}
Let $0<r<\infty$, suppose that $\mathbb{E}\big[|a_n|^r\big] < \infty$ for all $n$ and that $a_n \xrightarrow{\mathrm{p}}a$ as $n\to \infty$. The following are equivalent: 
\begin{description}
\item{(i)} $a_n\to a$ in $L^r$ as $n\to\infty$;
\item{(ii)} $\mathbb{E}\big[|a_n|^r\big]\to \mathbb{E}\big[|a|^r\big] < \infty$ as $n\to\infty$; 
\item{(iii)} $\big\{|a_n|^r, n\geq 1\big\}$ is uniformly integrable.
\end{description}
\end{proposition}

\subsection{Martingale Limit Theorems}
\begin{proposition}
[Weak Law of Large Numbers for Martingale, \citet{hall2014martingale}]
\label{prp:mrtgl_WLLN}
Let $\{S_n = \sum^{n}_{i=1} X_i, \mathcal{H}_{t}, t\geq 1\}$ be a martingale and $\{b_n\}$ a sequence of positive constants with $b_n\to\infty$ as $n\to\infty$. Then, writing $X_{ni} = X_i\mathbbm{1}[|X_i|\leq b_n]$, $1\leq i \leq n$, we have that $b^{-1}_n S_n \xrightarrow{\mathrm{p}} 0$ as $n\to \infty$ if 
\begin{description}
\item[(i)] $\sum^n_{i=1}P(|X_i| > b_n)\to 0$;
\item[(ii)] $b^{-1}_n\sum^n_{i=1}\mathbb{E}[X_{ni}|  \mathcal{H}_{t-1}] \xrightarrow{\mathrm{p}} 0$, and;
\item[(iii)] $b^{-2}_n \sum^n_{i=1}\big\{\mathbb{E}[X^2_{ni}] - \mathbb{E}\big[\mathbb{E}\big[X_{ni}|  \mathcal{H}_{t-1}\big]\big]^2\big\}\to 0$.
\end{description}
\end{proposition}
\noindent The weak law of large numbers for martingale holds when the random variable is bounded by a constant.
\begin{proposition}
[Central Limit Theorem for a Martingale Difference Sequence, \citet{Hamilton1994}, Proposition~7.9, p.~194] 
\label{prp:marclt}
Let $\{X_t\}^\infty_{t=1}$ be an $n$-dimensional vector martingale difference sequence with $\overline{X}_T=\frac{1}{T}\sum^T_{t=1}X_t$. Suppose that 
\begin{description}
\item[(a)] $\mathbb{E}[X^2_t] = \sigma^2_t$, a positive value with $(1/T)\sum^T_{t=1}\sigma^2_t\to\sigma^2_0$, a positive value; 
\item[(b)] $\mathbb{E}[|X_t|^r] < \infty$ for some $r>2$;
\item[(c)] $(1/T)\sum^{T}_{t=1}X^2_t\xrightarrow{p}\sigma^2_0$. 
\end{description}
Then $\sqrt{T}\overline{X}_T\xrightarrow{d}\mathcal{N}(\bm{0}, \sigma^2)$.
\end{proposition}

On the convergence rate of the central limit theorem  for a martingale difference sequence, see \citet{hall1980martingale}.

\section{Proof of Proposition~\ref{optprob}}
\label{appdx:prp:opt_prob}
\begin{proof}
Let $\mathcal{P}$ be a function class of $p:\mathcal{X}\to(0, 1)$, and let us define the following function $b:\mathcal{P}\to\mathbb{R}$:
\begin{align*}
b(p) = \mathbb{E}\left[\frac{e(1, X_t)}{b(X_t)}\right] + \mathbb{E}\left[\frac{e(0, X_t)}{1-b(X_t)}\right].
\end{align*}
Here, we rewrite $b(p)$ as follows:
\begin{align*}
b(p) = \mathbb{E}\left[\mathbb{E}\left[\frac{e(1, X_t)}{p(X_t)} + \frac{e(0, X_t)}{1-p(X_t)}\bigg|X_t\right]\right].
\end{align*}
We consider minimizing $b(p)$ by minimizing $\widetilde{b}(q) = \mathbb{E}\left[\frac{e(1, X_t)}{q} + \frac{e(0, X_t)}{1-q}\bigg|X_t\right]$ for $q\in[\varepsilon, 1-\varepsilon]$. The first derivative of $\widetilde{b}(q)$ with respect to $q$ is given as follows:
\begin{align*}
\widetilde{b}'(q) = -\frac{e(1, X_t)}{q^2} + \frac{e(0, X_t)}{(1-q)^2}.
\end{align*}
The second derivative of $f$ is given as follows:
\begin{align*}
\widetilde{b}''(q) =  2 \frac{e(1, X_t)}{q^3} + 2\frac{e(0, X_t)}{(1-q)^3}.
\end{align*}
For $\varepsilon < q < 1-\varepsilon$, because $\widetilde{b}''(q) > 0$, the minimizer $q^*$ of $\widetilde{b}$ satisfies the following equation:
\begin{align*}
-\frac{e(1, X_t)}{(q^{*})^2} + \frac{e(0, X_t)}{(1-q^*)^2}=0.
\end{align*}
This equation is equivalent to 
\begin{align*}
&-(q^*)^2e(0, X_t) + (1-q^*)^2e(1, X_t) =0\\
&\Leftrightarrow\ q^*\sqrt{e(0, X_t)} = (1-q^*)\sqrt{e(1, X_t)}\\
&\Leftrightarrow\  q^* = \frac{\sqrt{e(1, X_t)}}{\sqrt{e(1, X_t)}+\sqrt{e(0, X_t)}}.
\end{align*}
Therefore,
\begin{align*}
b^{\mathrm{OPT}}(D=1|X_t) = \frac{\sqrt{e(1, X_t)}}{\sqrt{e(1, X_t)}+\sqrt{e(0, X_t)}}.
\end{align*}
\end{proof}

\section{Proof of Theorem~\ref{thm:asymp_dist_A2IPW}}
\label{appdx:thm:asymp_dist_A2IPW}
\begin{proof}
Note that the estimator is given as follows:
\begin{align*}
&\widehat{\theta}^{\mathrm{A2IPW}}_T = \frac{1}{T}\sum^T_{t=1}\Bigg(\frac{\mathbbm{1}[A_t=1]\big(Y_t - \widehat{f}_{t-1}(1, X_t)\big)}{\pi_t(1\mid  X_t, \mathcal{H}_{t-1})} - \frac{\mathbbm{1}[A_t=0]\big(Y_t - \widehat{f}_{t-1}(0, X_t)\big)}{\pi_t(0\mid  x, \mathcal{H}_{t-1})}\\
&\ \ \ \ \ \ \ \ \ \ \ \ \ \ \ \ \ \ \ \ \ \ \ \ \ \ \ \ \ \ \ \ \ \ \ \ \ \ \ \ \ \ \ \ \ \ \ \ \ \ \ \ \ \ \ \ \ \ \ \ \ \ \ \ \ \ \ \ \ \ \ \ \ \ \ \ \ \ \ \ \ \ \ \ \ \ \ +  \widehat{f}_{t-1}(1, X_t) - \widehat{f}_{t-1}(0, X_t)\Bigg).\nonumber
\end{align*}

Let us note that $z_t$ is defined as  
\begin{align*}
&\frac{\mathbbm{1}[A_t=1]\big(Y_t - \widehat{f}_{t-1}(1, X_t)\big)}{\pi_t(1\mid  x, \mathcal{H}_{t-1})} - \frac{\mathbbm{1}[A_t=0]\big(Y_t - \widehat{f}_{t-1}(0, X_t)\big)}{\pi_t(0\mid  X_t, \mathcal{H}_{t-1})} +  \widehat{f}_{t-1}(1, X_t) - \widehat{f}_{t-1}(0, X_t) - \theta_0.
\end{align*}

The sequence $\{z_t\}^T_{t=1}$ is a martingale difference sequence, i.e.,
\begin{align*}
&\mathbb{E}\big[z_t|  \mathcal{H}_{t-1}\big]\\
&= \mathbb{E}\Bigg[\frac{\mathbbm{1}[A_t=1]\big(Y_t - \widehat{f}_{t-1}(1, X_t)\big)}{\pi_t(1\mid  X_t, \mathcal{H}_{t-1})} - \frac{\mathbbm{1}[A_t=k]\big(Y_t - \widehat{f}_{t-1}(0, X_t)\big)}{\pi_t(0\mid  X_t, \mathcal{H}_{t-1})}\\
&\ \ \ \ \ \ \ \ \ \ \ \ \ \ \ \ \ \ \ \ \ \ \ \ \ \ \ \ \ \ \ \ \ \ \ \ \ \ \ \ \ \ \ \ \ \ \ \ \ \ \ \ \ \ \ \ +  \widehat{f}_{t-1}(0, X_t) - \widehat{f}_{t-1}(0, X_t) - \theta_0|  \mathcal{H}_{t-1}\Bigg]\\
&= \mathbb{E}\Bigg[\widehat{f}_{t-1}(1, X_t) - \widehat{f}_{t-1}(0, X_t) - \theta_0\\
&\ \ \ \ \ \ \ + \mathbb{E}\left[\frac{\mathbbm{1}[A_t=1]\big(Y_t - \widehat{f}_{t-1}(1, X_t)\big)}{\pi_t(1\mid  X_t, \mathcal{H}_{t-1})} - \frac{\mathbbm{1}[A_t=0]\big(Y_t - \widehat{f}_{t-1}(0, X_t)\big)}{\pi_t(0\mid  X_t, \mathcal{H}_{t-1})}\mid  X_t, \mathcal{H}_{t-1}\right]|  \mathcal{H}_{t-1}\Bigg]\\
&= \mathbb{E}\Bigg[\widehat{f}_{t-1}(1, X_t) - \widehat{f}_{t-1}(0, X_t) - \theta_0 + f_0(1, X_t) - f_0(0, X_t)\\
&\ \ \ \ \ \ \ \ \ \ \ \ \ \ \ \ \ \ \ \ \ \ \ \ \ \ \ \ \ \ \ \ \ \ \ \ \ \ \ \ \ \ \ \ \ \ \ \ \ \ \ \ \ \ \ \ \ \ \ \ \ \ - \widehat{f}_{t-1}(1, X_t) + \widehat{f}_{t-1}(0, X_t)|  \mathcal{H}_{t-1}\Bigg]\\
&= 0.
\end{align*}

Therefore, to derive the asymptotic distribution, we consider applying the central limit theorem  for a martingale difference sequence introduced in Proposition~\ref{prp:marclt}. There are the following three conditions in the statement.
\begin{description}
\item[(a)] $\mathbb{E}\big[z^2_t\big] = \nu^2_t > 0$ with $\big(1/T\big) \sum^T_{t=1}\nu^2_t\to \nu^2 > 0$;
\item[(b)] $\mathbb{E}\big[|z_t|^r\big] < \infty$ for some $r>2$;
\item[(c)] $\big(1/T\big)\sum^T_{t=1}z^2_t\xrightarrow{\mathrm{p}} \nu^2$. 
\end{description}
Because we assumed the boundedness of $z_t$ by assuming the boundedness of $Y_t$, $\widehat{f}_{t-1}$, and $1/\pi_{t}$, the condition~(b) holds. Therefore, the remaining task is to show the conditions~(a) and (c) hold.

\subsection*{Step~1: Checking Condition~(a)}
We can rewrite $\mathbb{E}\big[z^2_t\big]$ as
\begin{align*}
&\mathbb{E}\big[z^2_t\big]\\
&=\mathbb{E}\Bigg[\Bigg(\frac{\mathbbm{1}[A_t=1]\big(Y_t - \widehat{f}_{t-1}(1, X_t)\big)}{\pi_t(1\mid  X_t, \mathcal{H}_{t-1})} - \frac{\mathbbm{1}[A_t=0]\big(Y_t - \widehat{f}_{t-1}(0, X_t)\big)}{\pi_t(0\mid  X_t, \mathcal{H}_{t-1})}\\
&\ \ \ \ \ \ \ \ \ \ \ \ \ \ \ \ \ \ \ \ \ \ \ \ \ \ \ \ \ \ \ \ \ \ \ \ \ \ \ \ \ \ \ \ \ \ \ \ \ \ \ \ \ \ \ \ \ \ \ \ \ \ \ \ \ \ \ \ \ \ +  \widehat{f}_{t-1}(1, X_t) - \widehat{f}_{t-1}(0, X_t) - \theta_0\Bigg)^2\Bigg]\\
&=\mathbb{E}\Bigg[\Bigg(\frac{\mathbbm{1}[A_t=1]\big(Y_t - \widehat{f}_{t-1}(1, X_t)\big)}{\pi_t(1\mid  X_t, \mathcal{H}_{t-1})} - \frac{\mathbbm{1}[A_t=0]\big(Y_t - \widehat{f}_{t-1}(0, X_t)\big)}{\pi_t(0\mid  X_t, \mathcal{H}_{t-1})}\\
&\ \ \ \ \ \ \ \ \ \ \ \ \ \ \ \ \ \ \ \ \ \ \ \ \ \ \ \ \ \ \ \ \ \ \ \ \ \ \ \ \ \ \ \ \ \ \ \ \ \ \ \ \ \ \ \ \ \ \ \ \ \ \ \ \ \ \ \ \ \  +  \widehat{f}_{t-1}(1, X_t) - \widehat{f}_{t-1}(0, X_t) - \theta_0\Bigg)^2\Bigg]\\
&\ \ \ -\mathbb{E}\left[\sum^{1}_{a=0}\frac{v\big(a, X_t\big)}{\widetilde{\pi}(a\mid  X_t)} + \Big(\theta_0(X_t) - \theta_0\Big)^2\right]+\mathbb{E}\left[\sum^{1}_{a=0}\frac{v\big(a, X_t\big)}{\widetilde{\pi}(a\mid  X_t)} + \Big(\theta_0(X_t) - \theta_0\Big)^2\right].
\end{align*}
Therefore, we prove that the RHS of the following equation varnishes asymptotically to show that the condition~(a) holds.
\begin{align}
\label{eq:1}
&\mathbb{E}\big[z^2_t\big] - \mathbb{E}\left[\sum^{1}_{a=0}\frac{v\big(a, X_t\big)}{\widetilde{\pi}(a\mid  X_t)} + \Big(\theta_0(X_t) - \theta_0\Big)^2\right]\nonumber\\
&=\mathbb{E}\Bigg[\Bigg(\frac{\mathbbm{1}[A_t=1]\big(Y_t - \widehat{f}_{t-1}(1, X_t)\big)}{\pi_t(1\mid  X_t, \mathcal{H}_{t-1})} - \frac{\mathbbm{1}[A_t=0]\big(Y_t - \widehat{f}_{t-1}(0, X_t)\big)}{\pi_t(0\mid  X_t, \mathcal{H}_{t-1})}\nonumber\\
&\ \ \ \ \ \ \ \ \ \ \ \ \ \ \ \ \ \ \ \ \ \ \ \ \ \ \ \ \ \ \ \ \ \ \ \ \ \ \ \ \ \ \ \ \ \ \ \ \ \ \ \ \ \ \ \ \ \ \ \ \ \ \ \ \ \ \ \ \ \ +  \widehat{f}_{t-1}(1, X_t) - \widehat{f}_{t-1}(0, X_t) - \theta_0\Bigg)^2\Bigg]\nonumber\\
&\ \ \ -\mathbb{E}\left[\sum^{1}_{a=0}\frac{v\big(a, X_t\big)}{\widetilde{\pi}(a\mid  X_t)} + \Big(\theta_0(X_t) - \theta_0\Big)^2\right].
\end{align}
First, for the first term of the RHS of \eqref{eq:1},
\begin{align*}
&\mathbb{E}\Bigg[\Bigg(\frac{\mathbbm{1}[A_t=1]\big(Y_t - \widehat{f}_{t-1}(1, X_t)\big)}{\pi_t(1\mid  X_t, \mathcal{H}_{t-1})} - \frac{\mathbbm{1}[A_t=0]\big(Y_t - \widehat{f}_{t-1}(0, X_t)\big)}{\pi_t(0\mid  X_t, \mathcal{H}_{t-1})}\nonumber\\
&\ \ \ \ \ \ \ \ \ \ \ \ \ \ \ \ \ \ \ \ \ \ \ \ \ \ \ \ \ \ \ \ \ \ \ \ \ \ \ \ \ \ \ \ \ \ \ \ \ \ \ \ \ \ \ \ \ \ \ \ \ \ \ \ \ \ \ \ \ \  +  \widehat{f}_{t-1}(1, X_t) - \widehat{f}_{t-1}(0, X_t) - \theta_0\Bigg)^2\Bigg]\\
&=\mathbb{E}\left[\left(\frac{\mathbbm{1}[A_t=1]\big(Y_t - \widehat{f}_{t-1}(1, X_t)\big)}{\pi_t(1\mid  X_t, \mathcal{H}_{t-1})}\right)^2\right]\\
&\ \ \ + \mathbb{E}\left[\left(\frac{\mathbbm{1}[A_t=0]\big(Y_t - \widehat{f}_{t-1}(0, X_t)\big)}{\pi_t(0\mid  X_t, \mathcal{H}_{t-1})}\right)^2\right]\\
&\ \ \ +  \mathbb{E}\left[\left(\widehat{f}_{t-1}(1, X_t) - \widehat{f}_{t-1}(0, X_t) - \theta_0\right)^2\right]\\
&\ \ \ -2\mathbb{E}\left[\left(\frac{\mathbbm{1}[A_t=1]\big(Y_t - \widehat{f}_{t-1}(1, X_t)\big)}{\pi_t(1\mid  X_t, \mathcal{H}_{t-1})}\right)\left(\frac{\mathbbm{1}[A_t=0]\big(Y_t - \widehat{f}_{t-1}(0, X_t)\big)}{\pi_t(0\mid  X_t, \mathcal{H}_{t-1})}\right)\right]\\
&\ \ \ +2\mathbb{E}\left[\left(\frac{\mathbbm{1}[A_t=1]\big(Y_t - \widehat{f}_{t-1}(1, X_t)\big)}{\pi_t(1\mid  X_t, \mathcal{H}_{t-1})}\right)\left(\widehat{f}_{t-1}(1, X_t) - \widehat{f}_{t-1}(0, X_t) - \theta_0\right)\right]\\
&\ \ \ -2\mathbb{E}\left[\left(\frac{\mathbbm{1}[A_t=0]\big(Y_t - \widehat{f}_{t-1}(0, X_t)\big)}{\pi_t(0\mid  X_t, \mathcal{H}_{t-1})}\right)\left(\widehat{f}_{t-1}(1, X_t) - \widehat{f}_{t-1}(0, X_t) - \theta_0\right)\right].
\end{align*}
Because $\mathbbm{1}[A_t=1]\mathbbm{1}[A_t=0] = 0$, $\mathbbm{1}[A_t=k]\mathbbm{1}[A_t=k] = \mathbbm{1}[A_t=k]$, and $\mathbbm{1}[A_t=k]Y_t = Y_t(a)$ for $k\in\{1, 0\}$, we have
\begin{align*}
&\mathbb{E}\left[\left(\frac{\mathbbm{1}[A_t=k]\big(Y_t - \widehat{f}_{t-1}(a, X_t)\big)}{\pi_t(a\mid  X_t, \mathcal{H}_{t-1})}\right)^2\right]=\mathbb{E}\left[\frac{\big(Y_t(a) - \widehat{f}_{t-1}(a, X_t)\big)^2}{\pi_t(a\mid  X_t, \mathcal{H}_{t-1})}\right],\\
&\mathbb{E}\left[\left(\frac{\mathbbm{1}[A_t=1]\big(Y_t - \widehat{f}_{t-1}(1, X_t)\big)}{\pi_t(1\mid  X_t, \mathcal{H}_{t-1})}\right)\left(\frac{\mathbbm{1}[A_t=0]\big(Y_t - \widehat{f}_{t-1}(0, X_t)\big)}{\pi_t(0\mid  X_t, \mathcal{H}_{t-1})}\right)\right]=0,\\
&\mathbb{E}\Bigg[\left(\frac{\mathbbm{1}[A_t=1]\big(Y_t - \widehat{f}_{t-1}(1, X_t)\big)}{\pi_t(1\mid  X_t, \mathcal{H}_{t-1})} - \frac{\mathbbm{1}[A_t=0]\big(Y_t - \widehat{f}_{t-1}(0, X_t)\big)}{\pi_t(0\mid  X_t, \mathcal{H}_{t-1})}\right)\nonumber\\
&\ \ \ \ \ \ \ \ \ \ \ \ \ \ \ \ \ \ \ \ \ \ \ \ \ \ \ \ \ \ \ \ \ \ \ \ \ \ \ \ \ \ \ \ \ \ \ \ \ \ \ \ \ \ \ \ \ \ \ \ \ \ \ \ \ \ \ \ \ \ \times \left(\widehat{f}_{t-1}(1, X_t) - \widehat{f}_{t-1}(0, X_t) - \theta_0\right)\Bigg]\\
&=\mathbb{E}\Bigg[\mathbb{E}\left[\frac{\mathbbm{1}[A_t=1]\big(Y_t - \widehat{f}_{t-1}(1, X_t)\big)}{\pi_t(1\mid  X_t, \mathcal{H}_{t-1})} - \frac{\mathbbm{1}[A_t=0]\big(Y_t - \widehat{f}_{t-1}(0, X_t)\big)}{\pi_t(0\mid  X_t, \mathcal{H}_{t-1})}\mid  X_t, \mathcal{H}_{t-1}\right]\\
&\ \ \ \ \ \ \ \ \ \ \ \ \ \ \ \ \ \ \ \ \ \ \ \ \ \ \ \ \ \ \ \ \ \ \ \ \ \ \ \ \ \ \ \ \ \ \ \ \ \ \ \ \ \ \ \ \ \ \ \ \ \ \ \ \ \ \ \ \ \ \times \left(\widehat{f}_{t-1}(1, X_t) - \widehat{f}_{t-1}(0, X_t) - \theta_0\right)\Bigg]\\
&=\mathbb{E}\left[\left(f_0(1, X_t) - f_0(0, X_t) - \hat f_{t-1}(1, X_t) + \hat f_{t-1}(0, X_t)\right)\left(\widehat{f}_{t-1}(1, X_t) - \widehat{f}_{t-1}(0, X_t) - \theta_0\right)\right].
\end{align*}

Therefore, for the first term of the RHS of \eqref{eq:1},
\begin{align*}
&\mathbb{E}\Bigg[\Bigg(\frac{\mathbbm{1}[A_t=1]\big(Y_t - \widehat{f}_{t-1}(1, X_t)\big)}{\pi_t(1\mid  X_t, \mathcal{H}_{t-1})} - \frac{\mathbbm{1}[A_t=0]\big(Y_t - \widehat{f}_{t-1}(0, X_t)\big)}{\pi_t(0\mid  X_t, \mathcal{H}_{t-1})}\nonumber\\
&\ \ \ \ \ \ \ \ \ \ \ \ \ \ \ \ \ \ \ \ \ \ \ \ \ \ \ \ \ \ \ \ \ \ \ \ \ \ \ \ \ \ \ \ \ \ \ \ \ \ \ \ \ \ \ \ \ \ \ \ \ \ \ \ \ \ \ \ \ \  +  \widehat{f}_{t-1}(1, X_t) - \widehat{f}_{t-1}(0, X_t) - \theta_0\Bigg)^2\Bigg]\\
&= \mathbb{E}\Bigg[\frac{\big(Y_t(1) - \widehat{f}_{t-1}(1, X_t)\big)^2}{\pi_t(1\mid  X_t, \mathcal{H}_{t-1})} +\frac{\big(Y_t(0) - \widehat{f}_{t-1}(0, X_t)\big)^2}{\pi_t(0\mid  X_t, \mathcal{H}_{t-1})} +  \Big(\widehat{f}_{t-1}(1, X_t) - \widehat{f}_{t-1}(0, X_t) - \theta_0\Big)^2\\
&\ \ \ + 2\left(f_0(1, X_t) - f_0(0, X_t) - \hat f_{t-1}(1, X_t) + \hat f_{t-1}(0, X_t)\right)\left(\widehat{f}_{t-1}(1, X_t) - \widehat{f}_{t-1}(0, X_t) - \theta_0\right)\Bigg].
\end{align*}
For the second term of the RHS of \eqref{eq:1},
\begin{align*}
&\mathbb{E}\left[\sum^{1}_{a=0}\frac{v\big(a, X_t\big)}{\widetilde{\pi}(a\mid  X_t)} + \Big(\theta_0(X_t) - \theta_0\Big)^2\right]\\
&=\mathbb{E}\left[\frac{\big(Y_t(1) - f_0(1, X_t)\big)^2}{\widetilde{\pi}(1\mid  X_t)} +\frac{\big(Y_t(0) - f_0(0, X_t)\big)^2}{\widetilde{\pi}(0\mid  X_t)} +  \Big(f_0(1, X_t) - f_0(0, X_t) - \theta_0\Big)^2\right].
\end{align*}
Using these equations, the RHS of \eqref{eq:1} can be calculated as 
\begin{align*}
&\mathbb{E}\Bigg[\Bigg(\frac{\mathbbm{1}[A_t=1]\big(Y_t - \widehat{f}_{t-1}(1, X_t)\big)}{\pi_t(1\mid  X_t, \mathcal{H}_{t-1})} - \frac{\mathbbm{1}[A_t=0]\big(Y_t - \widehat{f}_{t-1}(0, X_t)\big)}{\pi_t(0\mid  X_t, \mathcal{H}_{t-1})}\\
&\ \ \ \ \ \ \ \ \ \ \ \ \ \ \ \ \ \ \ \ \ \ \ \ \ \ \ \ \ \ \ \ \ \ \ \ \ \ \ \ \ \ \ \ \ \ \ \ \ \ \ \ \ \ \ \ \ \ \ \ \ \ \ \ \ \ \ \ \ \ \ \ \ \ \ \ \ +  \widehat{f}_{t-1}(1, X_t) - \widehat{f}_{t-1}(0, X_t) - \theta_0\Bigg)^2\Bigg]\\
&\ \ \ \ \ \ \ \ \ \ \ -\mathbb{E}\left[\sum^{1}_{a=0}\frac{v\big(a, X_t\big)}{\widetilde{\pi}(a\mid  X_t)} + \Big(\theta_0(X_t) - \theta_0\Big)^2\right]\\
&=\mathbb{E}\Bigg[\frac{\big(Y_t(1) - \widehat{f}_{t-1}(1, X_t)\big)^2}{\pi_t(1\mid  X_t, \mathcal{H}_{t-1})} +\frac{\big(Y_t(0) - \widehat{f}_{t-1}(0, X_t)\big)^2}{\pi_t(0\mid  X_t, \mathcal{H}_{t-1})} +  \Big(\widehat{f}_{t-1}(1, X_t) - \widehat{f}_{t-1}(0, X_t) - \theta_0\Big)^2\\
&\ \ \ + 2\left(f_0(1, X_t) - f_0(0, X_t) - \hat f_{t-1}(1, X_t) + \hat f_{t-1}(0, X_t)\right)\left(\widehat{f}_{t-1}(1, X_t) - \widehat{f}_{t-1}(0, X_t) - \theta_0\right)\Bigg]\\
&\ \ \ -\mathbb{E}\Bigg[\frac{\big(Y_t(1) - f_0(1, X_t)\big)^2}{\widetilde{\pi}(1\mid  X_t)} +\frac{\big(Y_t(0) - f_0(0, X_t)\big)^2}{\widetilde{\pi}(0\mid  X_t)} +  \Big(f_0(1, X_t) - f_0(0, X_t) - \theta_0\Big)^2\Bigg].
\end{align*}
By taking the absolute value, we can bound the RHS as 
\begin{align*}
&\mathbb{E}\Bigg[\frac{\big(Y_t(1) - \widehat{f}_{t-1}(1, X_t)\big)^2}{\pi_t(1\mid  X_t, \mathcal{H}_{t-1})} +\frac{\big(Y_t(0) - \widehat{f}_{t-1}(0, X_t)\big)^2}{\pi_t(0\mid  X_t, \mathcal{H}_{t-1})} +  \Big(\widehat{f}_{t-1}(1, X_t) - \widehat{f}_{t-1}(0, X_t) - \theta_0\Big)^2\\
&\ \ \ + 2\left(f_0(1, X_t) - f_0(0, X_t) - \hat f_{t-1}(1, X_t) + \hat f_{t-1}(0, X_t)\right)\left(\widehat{f}_{t-1}(1, X_t) - \widehat{f}_{t-1}(0, X_t) - \theta_0\right)\Bigg]\\
&\ \ \ -\mathbb{E}\Bigg[\frac{\big(Y_t(1) - f_0(1, X_t)\big)^2}{\widetilde{\pi}(1\mid  X_t)} +\frac{\big(Y_t(0) - f_0(0, X_t)\big)^2}{\widetilde{\pi}(0\mid  X_t)} +  \Big(f_0(1, X_t) - f_0(0, X_t) - \theta_0\Big)^2\Bigg]\\
&\leq \mathbb{E}\Bigg[\Bigg|\Bigg\{\frac{\big(Y_t(1) - \widehat{f}_{t-1}(1, X_t)\big)^2}{\pi_t(1\mid  X_t, \mathcal{H}_{t-1})} +\frac{\big(Y_t(0) - \widehat{f}_{t-1}(0, X_t)\big)^2}{\pi_t(0\mid  X_t, \mathcal{H}_{t-1})} +  \Big(\widehat{f}_{t-1}(1, X_t) - \widehat{f}_{t-1}(0, X_t) - \theta_0\Big)^2\\
&\ \ \ + 2\left(f_0(1, X_t) - f_0(0, X_t) - \hat f_{t-1}(1, X_t) + \hat f_{t-1}(0, X_t)\right)\left(\widehat{f}_{t-1}(1, X_t) - \widehat{f}_{t-1}(0, X_t) - \theta_0\right)\Bigg\}\\
&\ \ \ -\Bigg\{\frac{\big(Y_t(1) - f_0(1, X_t)\big)^2}{\widetilde{\pi}(1\mid  X_t)} +\frac{\big(Y_t(0) - f_0(0, X_t)\big)^2}{\widetilde{\pi}(0\mid  X_t)} +  \Big(f_0(1, X_t) - f_0(0, X_t) - \theta_0\Big)^2\Bigg\}\Bigg|\Bigg].
\end{align*}
From the triangle inequality, we have
\begin{align*}
&\mathbb{E}\Bigg[\Bigg|\Bigg\{\frac{\big(Y_t(1) - \widehat{f}_{t-1}(1, X_t)\big)^2}{\pi_t(1\mid  X_t, \mathcal{H}_{t-1})} +\frac{\big(Y_t(0) - \widehat{f}_{t-1}(0, X_t)\big)^2}{\pi_t(0\mid  X_t, \mathcal{H}_{t-1})} +  \Big(\widehat{f}_{t-1}(1, X_t) - \widehat{f}_{t-1}(0, X_t) - \theta_0\Big)^2\\
&\ \ \ + 2\left(f_0(1, X_t) - f_0(0, X_t) - \hat f_{t-1}(1, X_t) + \hat f_{t-1}(0, X_t)\right)\left(\widehat{f}_{t-1}(1, X_t) - \widehat{f}_{t-1}(0, X_t) - \theta_0\right)\Bigg\}\\
&\ \ \ -\Bigg\{\frac{\big(Y_t(1) - f_0(1, X_t)\big)^2}{\widetilde{\pi}(1\mid  X_t)} +\frac{\big(Y_t(0) - f_0(0, X_t)\big)^2}{\widetilde{\pi}(0\mid  X_t)} +  \Big(f_0(1, X_t) - f_0(0, X_t) - \theta_0\Big)^2\Bigg\}\Bigg|\Bigg]\\
&\leq \sum_{a\in\{1,0\}}\mathbb{E}\Bigg[\Bigg|\frac{\big(Y_t(a) - \widehat{f}_{t-1}(a, X_t)\big)^2}{\pi_t(a\mid  X_t, \mathcal{H}_{t-1})} - \frac{\big(Y_t(a) - f_0(a, X_t)\big)^2}{\widetilde{\pi}(a\mid  X_t)}\Bigg|\Bigg]\\
&\ \ \ + \mathbb{E}\Bigg[\Bigg|\Big(\widehat{f}_{t-1}(1, X_t) - \widehat{f}_{t-1}(0, X_t) - \theta_0\Big)^2 - \Big(f_0(1, X_t) - f_0(0, X_t) - \theta_0\Big)^2\Bigg|\Bigg]\\
&\ \ \ + 2\mathbb{E}\Bigg[\Bigg|\left(f_0(1, X_t) - f_0(0, X_t) - \hat f_{t-1}(1, X_t) + \hat f_{t-1}(0, X_t)\right)\\
&\ \ \ \ \ \ \ \ \ \ \ \ \ \ \ \ \ \ \ \ \ \ \ \ \ \ \ \ \ \ \ \ \ \ \ \ \ \ \ \ \ \ \ \ \ \ \ \ \ \ \ \ \ \ \ \ \ \ \ \ \ \ \ \ \ \ \times \left(\widehat{f}_{t-1}(1, X_t) - \widehat{f}_{t-1}(0, X_t) - \theta_0\right)\Bigg|\Bigg].
\end{align*}
Because all elements are assumed to be bounded and $b_1^2 - b_2^2 = (b_1+b_2)(b_1-b_2)$ for variables $b_1$ and $b_2$, there exist constants $\tilde C_0$, $\tilde C_1$, $\tilde C_2$, and $\tilde C_f$ such that
\begin{align*}
& \sum_{a\in\{1,0\}}\mathbb{E}\Bigg[\Bigg|\frac{\big(Y_t(a) - \widehat{f}_{t-1}(a, X_t)\big)^2}{\pi_t(a\mid  X_t, \mathcal{H}_{t-1})} - \frac{\big(Y_t(a) - f_0(a, X_t)\big)^2}{\widetilde{\pi}(a\mid  X_t)}\Bigg|\Bigg]\\
&\ \ \ + \mathbb{E}\Bigg[\Bigg|\Big(\widehat{f}_{t-1}(1, X_t) - \widehat{f}_{t-1}(0, X_t) - \theta_0\Big)^2 - \Big(f_0(1, X_t) - f_0(0, X_t) - \theta_0\Big)^2\Bigg|\Bigg]\\
&\ \ \ + 2\mathbb{E}\Bigg[\Bigg|\left(f_0(1, X_t) - f_0(0, X_t) - \hat f_{t-1}(1, X_t) + \hat f_{t-1}(0, X_t)\right)\\
&\ \ \ \ \ \ \ \ \ \ \ \ \ \ \ \ \ \ \ \ \ \ \ \ \ \ \ \ \ \ \ \ \ \ \ \ \ \ \ \ \ \ \ \ \ \ \ \ \ \ \ \ \ \ \ \ \ \times \left(\widehat{f}_{t-1}(1, X_t) - \widehat{f}_{t-1}(0, X_t) - \theta_0\right)\Bigg|\Bigg]\\
& \leq \tilde C_0\sum_{a\in\{1,0\}}\mathbb{E}\Bigg[\Bigg|\frac{\big(Y_t(a) - \widehat{f}_{t-1}(a, X_t)\big)}{\sqrt{\pi_t(a\mid  X_t, \mathcal{H}_{t-1})}} - \frac{\big(Y_t(a) - f_0(a, X_t)\big)}{\sqrt{\widetilde{\pi}(a\mid  X_t)}}\Bigg|\Bigg]\\
&\ \ \ + \mathbb{E}\Bigg[\Bigg|\Big(\widehat{f}_{t-1}(1, X_t) - \widehat{f}_{t-1}(0, X_t) - \theta_0\Big)^2 - \Big(f_0(1, X_t) - f_0(0, X_t) - \theta_0\Big)^2\Bigg|\Bigg]\\
&\ \ \ + 2\mathbb{E}\Bigg[\Bigg|\left(f_0(1, X_t) - f_0(0, X_t) - \hat f_{t-1}(1, X_t) + \hat f_{t-1}(0, X_t)\right)\\
&\ \ \ \ \ \ \ \ \ \ \ \ \ \ \ \ \ \ \ \ \ \ \ \ \ \ \ \ \ \ \ \ \ \ \ \ \ \ \ \ \ \ \ \ \ \ \ \ \ \ \ \ \ \ \ \ \ \times\left(\widehat{f}_{t-1}(1, X_t) - \widehat{f}_{t-1}(0, X_t) - \theta_0\right)\Bigg|\Bigg]\\
& \leq \tilde C_1\sum_{a\in\{1,0\}}\mathbb{E}\Bigg[\Bigg|\sqrt{\widetilde{\pi}(a\mid  X_t)}\big(Y_t - \widehat{f}_{t-1}(a, X_t)\big) - \sqrt{\pi_t(a\mid  X_t, \mathcal{H}_{t-1})}\big(Y_t - f_0(a, X_t)\big)\Bigg|\Bigg]\\
&\ \ \ + \mathbb{E}\Bigg[\Bigg|\Big(\widehat{f}_{t-1}(1, X_t) - \widehat{f}_{t-1}(0, X_t) - \theta_0\Big)^2 - \Big(f_0(1, X_t) - f_0(0, X_t) - \theta_0\Big)^2\Bigg|\Bigg]\\
&\ \ \ + 2\mathbb{E}\Bigg[\Bigg|\left(f_0(1, X_t) - f_0(0, X_t) - \hat f_{t-1}(1, X_t) + \hat f_{t-1}(0, X_t)\right)\\
&\ \ \ \ \ \ \ \ \ \ \ \ \ \ \ \ \ \ \ \ \ \ \ \ \ \ \ \ \ \ \ \ \ \ \ \ \ \ \ \ \ \ \ \ \ \ \ \ \ \ \ \ \ \ \ \ \ \times\left(\widehat{f}_{t-1}(1, X_t) - \widehat{f}_{t-1}(0, X_t) - \theta_0\right)\Bigg|\Bigg]\\
&\leq \widetilde{C}_1\sum_{a\in\{1,0\}}\mathbb{E}\left[\Big| \sqrt{\widetilde{\pi}(a\mid  X_t)}\widehat{f}_{t-1}(a, X_t) - \sqrt{\pi_t(a\mid  X_t, \mathcal{H}_{t-1})}f_0(a, X_t) \Big|\right]\\
&\ \ \ + \widetilde{C}_2 \sum_{a\in\{1,0\}}\mathbb{E}\left[\Big| \sqrt{\widetilde{\pi}(a\mid  X_t)} - \sqrt{\pi_t(a\mid  X_t, \mathcal{H}_{t-1})} \Big|\right]\\
&\ \ \ + \widetilde{C}_3 \sum_{a\in\{1,0\}}\mathbb{E}\left[\Big| \widehat{f}_{t-1}(a, X_t) - f_0(a, X_t) \Big|\right].
\end{align*}
From $b_1b_2 - b_3b_4 = (b_1 - b_3)b_4 - (b_4 - b_2)b_1$ for variables $b_1$, $b_2$, $b_3$, and $b_4$, there exist $\tilde C_z$ and $\tilde C_5$ such that 
\begin{align*}
&\widetilde{C}_1\sum_{a\in\{1,0\}}\mathbb{E}\left[\Big| \sqrt{\widetilde{\pi}(a\mid  X_t)}\widehat{f}_{t-1}(a, X_t) - \sqrt{\pi_t(a\mid  X_t, \mathcal{H}_{t-1})}f_0(a, X_t) \Big|\right]\\
&\ \ \ + \widetilde{C}_2 \sum_{a\in\{1,0\}}\mathbb{E}\left[\Big| \sqrt{\widetilde{\pi}(a\mid  X_t)} - \sqrt{\pi_t(a\mid  X_t, \mathcal{H}_{t-1})} \Big|\right]\\
&\ \ \ + \widetilde{C}_3 \sum_{a\in\{1,0\}}\mathbb{E}\left[\Big| \widehat{f}_{t-1}(a, X_t) - f_0(a, X_t) \Big|\right]\\
&\leq \widetilde{C}_4 \sum_{a\in\{1,0\}}\mathbb{E}\left[\Big| \sqrt{\widetilde{\pi}(a\mid  X_t)} - \sqrt{\pi_t(a\mid  X_t, \mathcal{H}_{t-1})} \Big|\right]\\
&\ \ \ + \widetilde{C}_5 \sum_{a\in\{1,0\}}\mathbb{E}\left[\Big| \widehat{f}_{t-1}(a, X_t) - f_0(a, X_t) \Big|\right].
\end{align*}
From $\pi_t(a\mid  x, \mathcal{H}_{t-1}) - \widetilde{\pi}(a\mid  x)\xrightarrow{\mathrm{p}}0$, we have $\sqrt{\pi_t(a\mid  x, \mathcal{H}_{t-1})} - \sqrt{\widetilde{\pi}(a\mid  x)}\xrightarrow{\mathrm{p}}0$. From the assumption that the point convergences in probability, i.e., for all $x\in\mathcal{X}$ and $k\in\{1, 0\}$, $\sqrt{\pi_t(a\mid  x, \mathcal{H}_{t-1})} - \sqrt{\widetilde{\pi}(a\mid  x)}\xrightarrow{\mathrm{p}}0$ and $\widehat{f}_{t-1}(a, x) - f_0(a, x)\xrightarrow{\mathrm{p}}0$ as $t\to\infty$, if $\sqrt{\pi_t(a\mid  x, \mathcal{H}_{t-1})}$, and $\widehat{f}_{t-1}(a, x)$ are  uniformly integrable, for fixed $x\in \mathcal{X}$, we can prove that 
\begin{align*}
&\mathbb{E}\big[|\sqrt{\pi_t(a\mid  X_t, \mathcal{H}_{t-1})} - \sqrt{\widetilde{\pi}(a\mid  X_t)}| \mid  X_t=x, \mathcal{H}_{t-1}\big]\\
&\ \ \ \ \ \ \ \ \ \ \ \ \ \ \ \ \ \ \ \ = \mathbb{E}\big[|\sqrt{\pi_t(a\mid  x, \mathcal{H}_{t-1})} - \sqrt{\widetilde{\pi}(a\mid  x)}|\big]  \to 0,\\
&\mathbb{E}\big[|\widehat{f}_{t-1}(a, X_t) - f_0(a, X_t)|\mid  X_t = x, \mathcal{H}_{t-1}\big] = \mathbb{E}\big[|\widehat{f}_{t-1}(a, x) - f_0(a, x)|\big] \to 0,
\end{align*}
as $t\to\infty$ using $L^r$-convergence theorem (Proposition~\ref{prp:lr_conv_theorem}). Here, we used the fact that $\widehat{f}_{t-1}(a, x)$ and $\sqrt{\pi_t(a\mid  x, \mathcal{H}_{t-1})}$ are independent from $X_t$. For fixed $x\in\mathcal{X}$, we can show that $\sqrt{\pi_t(a\mid  x, \mathcal{H}_{t-1})}$, and $\widehat{f}_{t-1}(a, x)$ are uniformly integrable from the boundedness of $\sqrt{\pi_t(a\mid  x, \mathcal{H}_{t-1})}$, and $\widehat{f}_{t-1}(a, x)$ (Proposition~\ref{prp:suff_uniint}). From the point convergence of $\mathbb{E}[|\sqrt{\pi_t(a\mid  X_t, \mathcal{H}_{t-1})} - \sqrt{\widetilde{\pi}(a\mid  X_t)}| \mid  X_t=x]$ and $\mathbb{E}[|\widehat{f}_{t-1}(a, X_t) - f_0(a, X_t)|\mid  X_t = x]$, by using the Lebesgue's dominated convergence theorem, we can show that
\begin{align*}
&\mathbb{E}_{X_t, \mathcal{H}_{t-1}}\big[\mathbb{E}\big[|\sqrt{\pi_t(a\mid  X_t, \mathcal{H}_{t-1})} - \sqrt{\widetilde{\pi}(a\mid  X_t)}| \mid  X_t, \mathcal{H}_{t-1}\big]\big]  \to 0,\\
&\mathbb{E}_{X_t, \mathcal{H}_{t-1}}\big[\mathbb{E}[|\widehat{f}_{t-1}(a, X_t) - f_0(a, X_t)| \mid  X_t, \mathcal{H}_{t-1}]\big]  \to 0.
\end{align*}
As $t\to\infty$,
\begin{align*}
&\mathbb{E}\big[z^2_t\big] - \mathbb{E}\left[\sum^{1}_{a=0}\frac{v\big(a, X_t\big)}{\widetilde{\pi}(a\mid  X_t)} + \Big(\theta_0(X_t) - \theta_0\Big)^2\right]\to 0.
\end{align*}
Therefore, for any $\epsilon > 0$, there exists $\tilde t > 0$ such that 
\begin{align*}
&\frac{1}{T} \sum^{T}_{t=1}\Bigg(\mathbb{E}\big[z^2_t\big] - \mathbb{E}\left[\sum^{1}_{a=0}\frac{v\big(a, X_t\big)}{\widetilde{\pi}(a\mid  X_t)} + \Big(\theta_0(X_t) - \theta_0\Big)^2\right]\Bigg)\leq \tilde t/T + \epsilon.
\end{align*}
Here, $\mathbb{E}\left[\sum^{1}_{a=0}\frac{v\big(a, X_t\big)}{\widetilde{\pi}(a\mid  X_t)} + \Big(\theta_0(X_t) - \theta_0\Big)^2\right] = \mathbb{E}\left[\sum^{1}_{a=0}\frac{v\big(a, X\big)}{\widetilde{\pi}(a\mid  X)} + \Big(\theta_0(X) - \theta_0\Big)^2\right]$ does not depend on periods. Therefore, $\big(1/T\big) \sum^{T}_{t=1}\sigma^2_t- \sigma^2 \leq \tilde t/T + \epsilon \to 0$ as $T\to\infty$, where 
\begin{align*}
\sigma^2 = \mathbb{E}\left[\sum^{1}_{a=0}\frac{v\big(a, X\big)}{\widetilde{\pi}(a\mid  X)} + \Big(\theta_0(X) - \theta_0\Big)^2\right].
\end{align*}

\subsection*{Step~2: Checking Condition~(b)}
From the boundedness of each variable in $z_t$, we can easily show that the condition~(b) holds. 

\subsection*{Step~3: Checking Condition~(c)}
Let $u_t$ be a martingale difference sequence such that
\begin{align*}
&u_t = z^2_t - \mathbb{E}\big[z^2_t|  \mathcal{H}_{t-1}\big]\\
&=\Bigg(\frac{\mathbbm{1}[A_t=1]\big(Y_t - \widehat{f}_{t-1}(1, X_t)\big)}{\pi_t(1\mid  X_t, \mathcal{H}_{t-1})} - \frac{\mathbbm{1}[A_t=0]\big(Y_t - \widehat{f}_{t-1}(0, X_t)\big)}{\pi_t(0\mid  X_t, \mathcal{H}_{t-1})}\\
&\ \ \ \ \ \ \ \ \ \ \ \ \ \ \ \ \ \ \ \ \ \ \ \ \ \ \ \ \ \ \ \ \ \ \ \ \ \ \ \ \ \ \ \ \ \ \ \ \ \ \ \ \ \ \ \ \ \ \ \ \ \ \ + \widehat{f}_{t-1}(1, X_t) - \widehat{f}_{t-1}(0, X_t) - \theta_0\Bigg)^2\\\
& - \mathbb{E}\Bigg[\Bigg(\frac{\mathbbm{1}[A_t=1]\big(Y_t - \widehat{f}_{t-1}(1, X_t)\big)}{\pi_t(1\mid  X_t, \mathcal{H}_{t-1})} - \frac{\mathbbm{1}[A_t=0]\big(Y_t - \widehat{f}_{t-1}(0, X_t)\big)}{\pi_t(0\mid  X_t, \mathcal{H}_{t-1})}\\
&\ \ \ \ \ \ \ \ \ \ \ \ \ \ \ \ \ \ \ \ \ \ \ \ \ \ \ \ \ \ \ \ \ \ \ \ \ \ \ \ \ \ \ \ \ \ \ \ \ \ \ \ \ \ \ \ \ \ \ \ \ \ \ +  \widehat{f}_{t-1}(1, X_t) - \widehat{f}_{t-1}(0, X_t) - \theta_0\Bigg)^2\mid \mathcal{H}_{t-1}\Bigg].
\end{align*}
From the boundedness of each variable in $z_t$, we can apply weak law of large numbers for a martingale difference sequence (Proposition~\ref{prp:mrtgl_WLLN} in Appendix~\ref{appdx:prelim}), and obtain
\begin{align*}
&\frac{1}{T}\sum^T_{t=1}u_t = \frac{1}{T}\sum^T_{t=1}\big(z^2_t - \mathbb{E}\big[z^2_t|  \mathcal{H}_{t-1}\big]\big)\xrightarrow{\mathrm{p}} 0.
\end{align*}
Next, we show that
\begin{align*}
\frac{1}{T}\sum^T_{t=1}\mathbb{E}\big[z^2_t|  \mathcal{H}_{t-1}\big] - \sigma^2_0\xrightarrow{\mathrm{p}} 0.
\end{align*}
From Markov's inequality, for $\varepsilon > 0$, we have
\begin{align*}
&\mathbb{P}\left(\left|\frac{1}{T}\sum^T_{t=1}\mathbb{E}\big[z^2_t|  \mathcal{H}_{t-1}\big] - \sigma^2_0\right| \geq \varepsilon\right)\\
&\leq \frac{\mathbb{E}\left[\left|\frac{1}{T}\sum^T_{t=1}\mathbb{E}\big[z^2_t|  \mathcal{H}_{t-1}\big] - \sigma^2_0\right|\right]}{\varepsilon}\\
&\leq \frac{\frac{1}{T}\sum^T_{t=1}\mathbb{E}\left[\big|\mathbb{E}\big[z^2_t|  \mathcal{H}_{t-1}\big] - \sigma^2_0\big|\right]}{\varepsilon}.
\end{align*}
We then consider showing $\mathbb{E}\left[\big|\mathbb{E}\big[z^2_t|  \mathcal{H}_{t-1}\big] - \sigma^2_0\big|\right] \to 0$. Here, we have
\begin{align*}
&\mathbb{E}\left[\big|\mathbb{E}\big[z^2_t|  \mathcal{H}_{t-1}\big] - \sigma^2_0\big|\right]\\
&=\mathbb{E}\Bigg[\Bigg|\mathbb{E}\Bigg[\frac{\big(Y_t(1) - \widehat{f}_{t-1}(1, X_t)\big)^2}{\pi_t(1\mid  X_t, \mathcal{H}_{t-1})} +\frac{\big(Y_t(0) - \widehat{f}_{t-1}(0, X_t)\big)^2}{\pi_t(0\mid  X_t, \mathcal{H}_{t-1})}\\
&\ \ \  +  \Big(\widehat{f}_{t-1}(1, X_t) - \widehat{f}_{t-1}(0, X_t) - \theta_0\Big)^2\\
&\ \ \ + 2\left(f_0(1, X_t) - f_0(0, X_t) - \hat f_{t-1}(1, X_t) + \hat f_{t-1}(0, X_t)\right)\left(\widehat{f}_{t-1}(1, X_t) - \widehat{f}_{t-1}(0, X_t) - \theta_0\right)\\
&\ \ \ - \frac{\big(Y_t(1) - f_0(1, X_t)\big)^2}{\widetilde{\pi}(1\mid  X_t)} - \frac{\big(Y_t(0) - f_0(0, X_t)\big)^2}{\widetilde{\pi}(0\mid  X_t)} -  \Big(f_0(1, X_t) - f_0(0, X_t) - \theta_0\Big)^2|  \mathcal{H}_{t-1}\Bigg]\Bigg|\Bigg]\\
&=\mathbb{E}\Bigg[\Bigg|\mathbb{E}\Bigg[\mathbb{E}\Bigg[\frac{\big(Y_t(1) - \widehat{f}_{t-1}(1, X_t)\big)^2}{\pi_t(1\mid  X_t, \mathcal{H}_{t-1})} +\frac{\big(Y_t(0) - \widehat{f}_{t-1}(0, X_t)\big)^2}{\pi_t(0\mid  X_t, \mathcal{H}_{t-1})}\\
&\ \ \ +  \Big(\widehat{f}_{t-1}(1, X_t) - \widehat{f}_{t-1}(0, X_t) - \theta_0\Big)^2\\
&\ \ \ + 2\left(f_0(1, X_t) - f_0(0, X_t) - \hat f_{t-1}(1, X_t) + \hat f_{t-1}(0, X_t)\right)\left(\widehat{f}_{t-1}(1, X_t) - \widehat{f}_{t-1}(0, X_t) - \theta_0\right)\\
&\ \ \ - \frac{\big(Y_t(1) - f_0(1, X_t)\big)^2}{\widetilde{\pi}(1\mid  X_t)} - \frac{\big(Y_t(0) - f_0(0, X_t)\big)^2}{\widetilde{\pi}(0\mid  X_t)}\\
&\ \ \ -  \Big(f_0(1, X_t) - f_0(0, X_t) - \theta_0\Big)^2\mid  X_t, \mathcal{H}_{t-1}\Bigg] |  \mathcal{H}_{t-1}\Bigg]\Bigg|\Bigg].
\end{align*}
By using Jensen's inequality, 
\begin{align*}
&\mathbb{E}\left[\big|\mathbb{E}\big[z^2_t|  \mathcal{H}_{t-1}\big] - \sigma^2_0\big|\right]\\
&\leq \mathbb{E}\Bigg[\mathbb{E}\Bigg[\Bigg|\mathbb{E}\Bigg[\frac{\big(Y_t(1) - \widehat{f}_{t-1}(1, X_t)\big)^2}{\pi_t(1\mid  X_t, \mathcal{H}_{t-1})} +\frac{\big(Y_t(0) - \widehat{f}_{t-1}(0, X_t)\big)^2}{\pi_t(0\mid  X_t, \mathcal{H}_{t-1})}\\
&\ \ \ +  \Big(\widehat{f}_{t-1}(1, X_t) - \widehat{f}_{t-1}(0, X_t) - \theta_0\Big)^2\\
&\ \ \ + 2\left(f_0(1, X_t) - f_0(0, X_t) - \hat f_{t-1}(1, X_t) + \hat f_{t-1}(0, X_t)\right)\left(\widehat{f}_{t-1}(1, X_t) - \widehat{f}_{t-1}(0, X_t) - \theta_0\right)\\
&\ \ \ - \frac{\big(Y_t(1) - f_0(1, X_t)\big)^2}{\widetilde{\pi}(1\mid  X_t)} - \frac{\big(Y_t(0) - f_0(0, X_t)\big)^2}{\widetilde{\pi}(0\mid  X_t)}\\
&\ \ \ -  \Big(f_0(1, X_t) - f_0(0, X_t) - \theta_0\Big)^2\mid  X_t, \mathcal{H}_{t-1}\Bigg] \Bigg| |  \mathcal{H}_{t-1}\Bigg]\Bigg]\\
&= \mathbb{E}\Bigg[\Bigg|\mathbb{E}\Bigg[\frac{\big(Y_t(1) - \widehat{f}_{t-1}(1, X_t)\big)^2}{\pi_t(1\mid  X_t, \mathcal{H}_{t-1})} +\frac{\big(Y_t(0) - \widehat{f}_{t-1}(0, X_t)\big)^2}{\pi_t(0\mid  X_t, \mathcal{H}_{t-1})}\\
&\ \ \  +  \Big(\widehat{f}_{t-1}(1, X_t) - \widehat{f}_{t-1}(0, X_t) - \theta_0\Big)^2\\
&\ \ \ + 2\left(f_0(1, X_t) - f_0(0, X_t) - \hat f_{t-1}(1, X_t) + \hat f_{t-1}(0, X_t)\right)\left(\widehat{f}_{t-1}(1, X_t) - \widehat{f}_{t-1}(0, X_t) - \theta_0\right)\\
&\ \ \ - \frac{\big(Y_t(1) - f_0(1, X_t)\big)^2}{\widetilde{\pi}(1\mid  X_t)} -\frac{\big(Y_t(0) - f_0(0, X_t)\big)^2}{\widetilde{\pi}(0\mid  X_t)}\\
&\ \ \ -  \Big(f_0(1, X_t) - f_0(0, X_t) - \theta_0\Big)^2\mid  X_t, \mathcal{H}_{t-1}\Bigg] \Bigg| \Bigg].
\end{align*}
Because $\hat f_{t-1}$ and $\pi_t$ are constructed from $\mathcal{H}_{t-1}$, we have
\begin{align*}
&\mathbb{E}\left[\big|\mathbb{E}\big[z^2_t|  \mathcal{H}_{t-1}\big] - \sigma^2_0\big|\right]\\
&\leq \mathbb{E}\Bigg[\Bigg|\mathbb{E}\Bigg[\frac{\big(Y_t(1) - \widehat{f}_{t-1}(1, X_t)\big)^2}{\pi_t(1\mid  X_t, \mathcal{H}_{t-1})} +\frac{\big(Y_t(0) - \widehat{f}_{t-1}(0, X_t)\big)^2}{\pi_t(0\mid  X_t, \mathcal{H}_{t-1})} +  \Big(\widehat{f}_{t-1}(1, X_t) - \widehat{f}_{t-1}(0, X_t) - \theta_0\Big)^2\\
&\ \ \ + 2\left(f_0(1, X_t) - f_0(0, X_t) - \hat f_{t-1}(1, X_t) + \hat f_{t-1}(0, X_t)\right)\left(\widehat{f}_{t-1}(1, X_t) - \widehat{f}_{t-1}(0, X_t) - \theta_0\right)\\
&\ \ \ - \frac{\big(Y_t(1) - f_0(1, X_t)\big)^2}{\widetilde{\pi}(1\mid  X_t)} - \frac{\big(Y_t(0) - f_0(0, X_t)\big)^2}{\widetilde{\pi}(0\mid  X_t)}\\
&\ \ \  -  \Big(f_0(1, X_t) - f_0(0, X_t) - \theta_0\Big)^2\mid  X_t, \widehat{f}_{t-1}, \pi_t\Bigg] \Bigg| \Bigg].
\end{align*}
From the results of Step~1, there exist $\widetilde{C}_4$ and $\widetilde{C}_5$ such that
\begin{align*}
&\mathbb{E}\left[\big|\mathbb{E}\big[z^2_t|  \mathcal{H}_{t-1}\big] - \sigma^2_0\big|\right]\\
&\leq \mathbb{E}\Bigg[\Bigg|\mathbb{E}\Bigg[\frac{\big(Y_t(1) - \widehat{f}_{t-1}(1, X_t)\big)^2}{\pi_t(1\mid  X_t, \mathcal{H}_{t-1})} +\frac{\big(Y_t(0) - \widehat{f}_{t-1}(0, X_t)\big)^2}{\pi_t(0\mid  X_t, \mathcal{H}_{t-1})}\\
&\ \ \ \  +  \Big(\widehat{f}_{t-1}(1, X_t) - \widehat{f}_{t-1}(0, X_t) - \theta_0\Big)^2\\
&\ \ \ + 2\left(f_0(1, X_t) - f_0(0, X_t) - \hat f_{t-1}(1, X_t) + \hat f_{t-1}(0, X_t)\right)\left(\widehat{f}_{t-1}(1, X_t) - \widehat{f}_{t-1}(0, X_t) - \theta_0\right)\Bigg\}\\
&\ \ \ - \frac{\big(Y_t(1) - f_0(1, X_t)\big)^2}{\widetilde{\pi}(1\mid  X_t)} +\frac{\big(Y_t(0) - f_0(0, X_t)\big)^2}{\widetilde{\pi}(0\mid  X_t)}\\
&\ \ \ -  \Big(f_0(1, X_t) - f_0(0, X_t) - \theta_0\Big)^2\mid  X_t, \widehat{f}_{t-1}, \pi_t\Bigg] \Bigg| \Bigg]\\
&\leq \widetilde{C}_4 \sum_{a\in\{1,0\}}\mathbb{E}\left[\Big| \sqrt{\widetilde{\pi}(a\mid  X_t)} - \sqrt{\pi_t(a\mid  X_t, \mathcal{H}_{t-1})} \Big|\right]\\
&\ \ \ \ + \widetilde{C}_5 \sum_{a\in\{1,0\}}\mathbb{E}\left[\Big| \widehat{f}_{t-1}(a, X_t) - f_0(a, X_t) \Big|\right].
\end{align*}

From $L^r$ convergence theorem, by using point convergence of $\pi_t$ and $\widehat{f}_{t-1}$ and the boundedness of $z_t$, we have $\mathbb{E}\left[\big|\mathbb{E}\big[z^2_t|  \mathcal{H}_{t-1}\big] - \sigma^2_0\big|\right]\to 0$. Therefore,
\begin{align*}
&\mathbb{P}\left(\left|\frac{1}{T}\sum^T_{t=1}\mathbb{E}\big[z^2_t|  \mathcal{H}_{t-1}\big] - \sigma^2_0\right| \geq \varepsilon\right) \leq \frac{\frac{1}{T}\sum^T_{t=1}\mathbb{E}\left[\big|\mathbb{E}\big[z^2_t|  \mathcal{H}_{t-1}\big] - \sigma^2_0\big|\right]}{\varepsilon} \to 0.
\end{align*}
As a conclusion, 
\begin{align*}
&\frac{1}{T}\sum^T_{t=1}z^2_t - \sigma^2 = \frac{1}{T}\sum^T_{t=1}\big(z^2_t - \mathbb{E}\left[z^2_t|  \mathcal{H}_{t-1}\big] + \mathbb{E}\big[z^2_t|  \mathcal{H}_{t-1}\big] - \sigma^2_0\right)\xrightarrow{\mathrm{p}} 0.
\end{align*}

\subsection*{Conclusion}
From Steps~1--3, we can use central limit theorem  for a martingale difference sequence. Hence, we have
\begin{align*} 
\sqrt{T}\left(\widehat{\theta}^{\mathrm{A2IPW}}_T-\theta_0\right)\xrightarrow{d}\mathcal{N}\left(0, \sigma^2_0\right),
\end{align*}
where $\sigma^2 = \mathbb{E}\left[\sum^{1}_{a=0}\frac{\nu\big(a, X_t\big)}{\widetilde{\pi}(a\mid  X_t)} + \Big(\theta_0(X_t) - \theta_0\Big)^2\right]$.
\end{proof}

\section{Proof of Theorem~\ref{thm:regret}}
\label{appdx:thm:regret}
\begin{proof}
We have
\begin{align*}
&\left(\theta_0 - \widehat{\theta}^{\mathrm{A2IPW}}_T\right)^2=\left(\frac{1}{T}\theta - \frac{1}{T}\Psi_1 + \cdots + \frac{1}{T}\theta - \frac{1}{T}\Psi_T\right)^2=\frac{1}{T^2}\left(\theta - \Psi_1 + \cdots + \theta - \Psi_T\right)^2.
\end{align*}
Let $z_t$ be $\theta_0 - \Psi_t$. Then,
\begin{align*}
&\mathbb{E}_{\Pi}\left[(\theta - \widehat{\theta}^{\mathrm{A2IPW}}_T)^2\right]=\frac{1}{T^2}\mathbb{E}_{\Pi}\left[\left(\sum^T_{t=1}z_t\right)^2\right]=\frac{1}{T^2}\mathbb{E}_{\Pi}\left[\sum^T_{t=1}z^2_t + 2\sum^T_{t=1}\sum^{t-1}_{s=1}z_tz_s\right].
\end{align*}
We use the following result:
\begin{align*}
&\mathbb{E}\left[\sum^T_{t=1}\sum^{t-1}_{s=1}z_tz_s\right]\\
&=\sum^T_{t=1}\sum^{t-1}_{s=1}\mathbb{E}_{\mathcal{H}_{t-1}}\left[\mathbb{E}_{\Pi|  \mathcal{H}_{t-1}}\left[z_tz_s |  \mathcal{H}_{t-1}\right]\right]\\
&=\sum^T_{t=1}\sum^{t-1}_{s=1}\mathbb{E}_{\mathcal{H}_{t-1}}\left[\mathbb{E}_{\Pi|  \mathcal{H}_{t-1}}\left[z_t|  \mathcal{H}_{t-1}\right]z_s\right]\\
&=\sum^T_{t=1}\sum^{t-1}_{s=1}\mathbb{E}_{\mathcal{H}_{t-1}}\left[0\times z_s\right] = 0.
\end{align*}

Therefore,
\begin{align*}
&\mathbb{E}_{\Pi}\left[(\theta_0 - \widehat{\theta}^{\mathrm{A2IPW}}_T)^2\right]=\frac{1}{T^2}\mathbb{E}_{\Pi}\left[\sum^T_{t=1}z^2_t\right]=\frac{1}{T^2}\sum^T_{t=1}\mathbb{E}_{\Pi}\left[z^2_t\right].
\end{align*}
As we showed in Step~1 of the proof of Theorem~\ref{thm:asymp_dist_A2IPW}, we have
\begin{align*}
&\mathbb{E}_{\Pi}\left[(\theta_0 - \widehat{\theta}^{\mathrm{A2IPW}}_T)^2\right]\\
&=\frac{1}{T^2}\sum^T_{t=1}\mathbb{E}_{\Pi}\Bigg[\frac{\big(Y_t(1) - \widehat{f}_{t-1}(1, X_t)\big)^2}{\pi_t(1\mid  X_t, \mathcal{H}_{t-1})} +\frac{\big(Y_t(0) - \widehat{f}_{t-1}(0, X_t)\big)^2}{\pi_t(0\mid  X_t, \mathcal{H}_{t-1})}\\
&\ \ \ +  \Big(\widehat{f}_{t-1}(1, X_t) - \widehat{f}_{t-1}(0, X_t) - \theta_0\Big)^2\\
&\ \ \ + 2\left(f_0(1, X_t) - f_0(0, X_t) - \hat f_{t-1}(1, X_t) + \hat f_{t-1}(0, X_t)\right)\left(\widehat{f}_{t-1}(1, X_t) - \widehat{f}_{t-1}(0, X_t) - \theta_0\right)\Bigg].
\end{align*}

On the other hand, we have
\begin{align*}
&\mathbb{E}_{\Pi^\mathrm{OPT}}\left[\left(\theta_0-\widehat{\theta}^{\mathrm{OPT}}_T\right)^2\right]\\
&=\frac{1}{T^2}\sum^T_{t=1}\mathbb{E}_{\Pi^\mathrm{OPT}}\Bigg[\Bigg(\frac{\mathbbm{1}[\tilde A_t=1]\big(Y_t - f_0(1, X_t)\big)}{\pi^*(1\mid  X_t)} - \frac{\mathbbm{1}[\tilde A_t=0]\big(Y_t - f_0(0, X_t)\big)}{\pi^*(0\mid  X_t)}\\
&\ \ \ \ \ \ \ \ \ \ \ \ \ \ \ \ \ \ \ \ \ \ \ \ \ \ \ \ \ \ \ \ \ \ \ \ \ \ \ \ \ \ \ \ \ \ \ \ \ \ \ \ \ \ \ \ \ \ \ \ \ \ \ \ \ \ \ \ \ \ \ \ +  f_0(1, X_t) -  f_0(0, X_t) - \theta_0\Bigg)^2\Bigg],
\end{align*}
where $\tilde A_t$ denotes the stochastic variable of a treatment under a treatment-assignment probability $\pi^*$. We have 
\begin{align*}
&\frac{1}{T^2}\sum^T_{t=1}\mathbb{E}_{\Pi^\mathrm{OPT}}\Bigg[\Bigg(\frac{\mathbbm{1}[\tilde A_t=1]\big(Y_t - f_0(1, X_t)\big)}{\pi^*(1\mid  X_t)} - \frac{\mathbbm{1}[\tilde A_t=0]\big(Y_t - f_0(0, X_t)\big)}{\pi^*(0\mid  X_t)}\\
&\ \ \ \ \ \ \ \ \ \ \ \ \ \ \ \ \ \ \ \ \ \ \ \ \ \ \ \ \ \ \ \ \ \ \ \ \ \ \ \ \ \ \ \ \ \ \ \ \ \ \ \ \ \ \ \ \ \ \ \ \ \ \ \ \ \ \ \ \ \ \ \ \ \ \  +  f_0(1, X_t) -  f_0(0, X_t) - \theta_0\Bigg)^2\Bigg]\\
&=\frac{1}{T^2}\sum^T_{t=1}\mathbb{E}\Bigg[\frac{\big(Y_t(1) - f_0(1, X_t)\big)^2}{\pi^*(1\mid  X_t)} +\frac{\big(Y_t(0) - f_0(0, X_t)\big)^2}{\pi^*(0\mid  X_t)} +  \Big(f_0(1, X_t) - f_0(0, X_t) - \theta_0\Big)^2\Bigg].
\end{align*}

Therefore, we have
\begin{align*}
&\mathbb{E}_{\Pi}\left[\left(\theta_0-\widehat{\theta}^{\mathrm{A2IPW}}_T\right)^2\right]-\mathbb{E}_{\Pi^\mathrm{OPT}}\left[\left(\theta_0-\widehat{\theta}^{\mathrm{OPT}}_T\right)^2\right]\\
&=\frac{1}{T^2}\sum^T_{t=1}\mathbb{E}\Bigg[\frac{\big(Y_t(1) - \widehat{f}_{t-1}(1, X_t)\big)^2}{\pi_t(1\mid  X_t, \mathcal{H}_{t-1})} +\frac{\big(Y_t(0) - \widehat{f}_{t-1}(0, X_t)\big)^2}{\pi_t(0\mid  X_t, \mathcal{H}_{t-1})}\\
&\ \ \  +  \Big(\widehat{f}_{t-1}(1, X_t) - \widehat{f}_{t-1}(0, X_t) - \theta_0\Big)^2\\
&\ \ \ + 2\left(f_0(1, X_t) - f_0(0, X_t) - \hat f_{t-1}(1, X_t) + \hat f_{t-1}(0, X_t)\right)\left(\widehat{f}_{t-1}(1, X_t) - \widehat{f}_{t-1}(0, X_t) - \theta_0\right)\Bigg]\\
&\ \ \ -\frac{1}{T^2}\sum^T_{t=1}\mathbb{E}_{\Pi}\Bigg[\frac{\big(Y_t(1) - f_0(1, X_t)\big)^2}{\pi^*(1\mid  X_t)} +\frac{\big(Y_t(0) - f_0(0, X_t)\big)^2}{\pi^*(0\mid  X_t)}\\
&\ \ \  +  \Big(f_0(1, X_t) - f_0(0, X_t) - \theta_0\Big)^2\Bigg]\\
&\leq\frac{1}{T^2}\sum^T_{t=1}\mathbb{E}\Bigg[\Bigg|\Bigg\{\frac{\big(Y_t(1) - \widehat{f}_{t-1}(1, X_t)\big)^2}{\pi_t(1\mid  X_t, \mathcal{H}_{t-1})} +\frac{\big(Y_t(0) - \widehat{f}_{t-1}(0, X_t)\big)^2}{\pi_t(0\mid  X_t, \mathcal{H}_{t-1})}\\
&\ \ \ +  \Big(\widehat{f}_{t-1}(1, X_t) - \widehat{f}_{t-1}(0, X_t) - \theta_0\Big)^2\\
&\ \ \ + 2\left(f_0(1, X_t) - f_0(0, X_t) - \hat f_{t-1}(1, X_t) + \hat f_{t-1}(0, X_t)\right)\left(\widehat{f}_{t-1}(1, X_t) - \widehat{f}_{t-1}(0, X_t) - \theta_0\right)\Bigg\}\\
&\ \ \ - \Bigg\{\frac{\big(Y_t(1) - f_0(1, X_t)\big)^2}{\pi^*(1\mid  X_t)} +\frac{\big(Y_t(0) - f_0(0, X_t)\big)^2}{\pi^*(0\mid  X_t)} +  \Big(f_0(1, X_t) - f_0(0, X_t) - \theta_0\Big)^2\Bigg\}\Bigg|\Bigg],
\end{align*}
where the expectation of the last equation is taken over random variables including $\mathcal{H}_{t-1}$.

As we proved in Step~1 of proof of Theorem~\ref{thm:asymp_dist_A2IPW}, there exist constants $\tilde C_{0}$ and $\tilde C_{1}$ such that
\begin{align*}
&\mathbb{E}\left[\left(\theta_0-\widehat{\theta}^{\mathrm{A2IPW}}_T\right)^2\right]-\mathbb{E}\left[\left(\theta_0-\widehat{\theta}^{\mathrm{OPT}}_T\right)^2\right]\\
&\leq \frac{\widetilde{C}_{0}}{T^2}\sum^T_{t=1}\sum_{a\in\{1,0\}}\mathbb{E}\left[\Big| \sqrt{\pi^*(a\mid  X_t)} - \sqrt{\pi_t(a\mid  X_t, \mathcal{H}_{t-1})} \Big|\right]\\
&\ \ \ + \frac{\widetilde{C}_{1}}{T^2}\sum^T_{t=1}\sum_{a\in\{1,0\}}\mathbb{E}\left[\Big| \widehat{f}_{t-1}(a, X_t) - f_0(a, X_t) \Big|\right].
\end{align*}
Therefore, we have
\begin{align*}
&\mathbb{E}\left[\left(\theta_0-\widehat{\theta}^{\mathrm{A2IPW}}_T\right)^2\right]-\mathbb{E}\left[\left(\theta_0-\widehat{\theta}^{\mathrm{OPT}}_T\right)^2\right]\\
&= \frac{1}{T^2}\sum^T_{t=1}\sum_{a\in\{1,0\}}\Bigg\{O\left(\mathbb{E}\left[\Big| \sqrt{\pi^*(a\mid  X_t)} - \sqrt{\pi_t(a\mid  X_t, \mathcal{H}_{t-1})} \Big|\right]\right)\\
&\ \ \ \ \ \ \ \ \ \ \ \ \ \ \ \ \ \ \ \ \ \ \ \ \ \ \ \ \ \ \ \ \ \ \ \ +O\left(\mathbb{E}\left[\Big| f_0(a, X_t) - \widehat{f}_{t-1}(a, X_t) \Big|\right]\right)\Bigg\}.
\end{align*}
\end{proof}
\subsection{Proof of Theorem~\ref{thm:conc_A2IPW}}
\label{appdx:thm:conc_A2IPW}
The procedure of this proof mainly follows \citet{Balsubramani2016}. For a martingale $M_t$, let $V_t = \sum^t_{i=1}\mathbb{E}\big[(M_{i} - M_{i-1})^2\mid \mathcal{H}_{i-1}\big]$. Before proving Theorem~\ref{thm:conc_A2IPW}, we prove the following three lemmas.

\begin{lemma}[Small Sample Bound for a Martingale Difference Sequence]
\label{prp:small_bound}
Let $M_t$ be a martingale such that for all $t\geq 1$, $\big|M_t - M_{t-1}\big|\leq e^2/2$ with probability $1$. Fix any $\delta>0$, and define $\tau_0=\min\big\{s : 2(e-2)V_s\geq 173 \log\left(\frac{4}{\delta}\right)\big\}$. Then, with probability $\geq 1-\delta$, for all $t\leq \tau_0$, 
\begin{align*}
|M_t| \leq 2\sqrt{\frac{173}{2(e-2)}} \log\left(\frac{4}{\delta}\right)
\end{align*}
\end{lemma}

\begin{lemma}[Uniform Bernstein Bound for Martingales at Any Time]
\label{prp:known_conc} 
Let $M_t$ be a martingale such that for all $t\geq 1$, $\big|M_t - M_{t-1}\big|\leq e^2/2$ with probability $1$. Then, with probability $\geq 1-\delta$, for all $t$ simultaneously,
\begin{align*}
\left|M_t\right| \leq C_0(\delta) + \sqrt{2C_1V_t\left(\log\log V_t + \log\left(\frac{4}{\delta}\right)\right)},
\end{align*}
where $C_0(\delta)=3(e-2) + 2\sqrt{\frac{173}{2(e-2)}} \log\left(\frac{4}{\delta}\right)$ and $C_1 = 6(e-2)$.
\end{lemma}

\begin{remark}
For the Napier's constant $e$, $e^2/2 \approx 3.694$.
\end{remark}

\begin{lemma}[Upper Bound of the Variance]
\label{prp:var_bound}
Let $M_t$ be a martingale such that for all $t\geq 1$, $\big|M_t - M_{t-1}\big|\leq e^2/2$ with probability $1$. Suppose that there exists $C_z$ such that $\left|(M_t-M_{t-1})^2 - \mathbb{E}\left[ (M_i - M_{i-1})^2|  \mathcal{H}_{i-1}\right]\right| \leq  C_z$. With probability $\geq 1-\delta$, for all $t$, for sufficiently large $V_t$ and  $\sum^t_{i=1}(M_i-M_{i-1})^2$, there is an absolute constant $C_f$ such that
\begin{align*}
V_t \leq C_f \left(\sum^t_{i=1}(M_{i} - M_{i-1})^2 + \frac{2C_zC_0(\delta)}{e^2}\right),
\end{align*}
where $C_0(\delta)=3(e-2) + 2\sqrt{\frac{173}{2(e-2)}} \log\left(\frac{4}{\delta}\right)$.
\end{lemma}
In this section, we use the following three propositions.
\begin{proposition}[\citet{Balsubramani2014SharpFI}, Lemma~23]
\label{prp:super_mart}
Suppose that, for all $\ell \geq 3$ and $t$, $\mathbb{E}[(M_{t} - M_{t-1})^\ell|  \mathcal{H}_{t-1}]\leq \frac{1}{2}\ell!\left(e/\sqrt{2}\right)^{2(\ell-2)}\mathbb{E}[(M_t - M_{t-1})^2|  \mathcal{H}_{t-1}]$. Then, for any $\lambda \in \left(-\frac{1}{e^2}, \frac{1}{e^2}\right)$, the process $U^\lambda_{t}:=\exp(\lambda M_t - \lambda^2V_t)$ is a super martingale. 
\end{proposition}
\begin{remark}
The condition that, for all $\ell \geq 3$ and all $t$, $\mathbb{E}[(M_{t} - M_{t-1})^\ell|  \mathcal{H}_{t-1}]\leq \frac{1}{2}\ell!\left(e/\sqrt{2}\right)^{2(\ell-2)}\mathbb{E}[(M_t - M_{t-1})^2|  \mathcal{H}_{t-1}]$ is satisfied when $|M_{t} - M_{t-1}|\leq \frac{e^2}{2}$ for all $t$ with probability $1$.
\end{remark}

\begin{proposition}[Uniform Bernstein Bound for Martingales, \citet{Balsubramani2014SharpFI}, Theorem~5]
\label{prp:uni_bern_bound}
Let $M_t$ be a martingale such that for all $t\geq 1$, $\big|M_t - M_{t-1}\big|\leq e^2$ with probability $1$. Fix any $\delta < 1$ and define $\tau_0=\min\big\{s : 2(e-2)V_s\geq 173 \log\left(\frac{4}{\delta}\right)\big\}$. Then, with probability $\geq 1-\delta$, for all $t\geq \tau_0$ simultaneously, $\big|M_t\big|\leq \frac{2(e-2)}{e^2(1+\sqrt{1/3})}V_t$ and 
\begin{align*}
\big|M_t\big| \leq \sqrt{6(e-2)V_t\left(2\log\log\left(\frac{3(e-2)V_t}{\big|M_t\big|}\right) + \log\left(\frac{2}{\delta}\right)\right)}. 
\end{align*}
\end{proposition}

\begin{proposition}
\label{prp:lemma9}
Suppose $b_1$, $b_2$, $c$ are positive constants, 
\[r\geq 8\max\big(e^4b_1\log \log (e^4r/4), e^4b_2\big),\]
and $r - \sqrt{b_1 e^4r\log\log \left(e^4r/4\right) + b_2 e^4r} - c \leq 0$. Then, 
\begin{align*}
\sqrt{r} \leq \sqrt{b_1e^4\log \log (e^4c/2) + b_2e^4} + \sqrt{c}.
\end{align*} 
\end{proposition}
This proposition is almost the same as Lemma~9 of \citet{Balsubramani2014SharpFI}, but we changed the statement a little. We show the proof as follows.
\begin{proof}[Proof of Lemma~\ref{prp:lemma9}]
Since $r \geq 8e^4 b_2$, 
\begin{align*}
&0\leq \frac{r}{8} - e^4b_2= \frac{r}{4} - \frac{r}{8} - e^4b_2 = \frac{r}{4} - b_1\frac{r}{8b_1} - e^4b_2\rightarrow 0 \leq \frac{r^2}{4} - b_1 r \frac{r}{8b_1} - b_2 e^4r.
\end{align*}
Substituting the assumption $\frac{r}{8b_1} \geq e^4\log \log (e^4 r/4)$ gives
\begin{align*}
&0 \leq \frac{r^2}{4} - b_1 r \frac{r}{8b_1} - b_2 e^4r \leq \frac{r^2}{4} - b_1 r e^4 \log \log \left(e^4 r/4\right) - b_2 e^4r\\
&\rightarrow \sqrt{ b_1 r e^4 \log \log \left(e^4 r/4\right) + b_2 e^4r} \leq \frac{r}{2}.
\end{align*}
By substituting this into $r - \sqrt{b_1 e^4r\log\log \left(e^4r/4\right) + b_2 e^4r} - c \leq 0$, we have $r\leq 2c$. Therefore, again using $r - \sqrt{b_1 e^4r/4\log\log \left(e^4r/4\right) + b_2 e^4r} - c \leq 0$,
\begin{align*}
0 & \geq r - \sqrt{b_1 e^4r\log\log \left(e^4r/4\right) + b_2 e^4r} - c\\
  & \geq r - \sqrt{b_1 e^4r\log\log \left(e^4c/2\right) + b_2 e^4r} - c.
\end{align*} 
This is a quadratic in $\sqrt{r}$. By solving it, we have
\begin{align*}
\sqrt{r} &\leq\frac{1}{2}\left(\sqrt{b_1 e^4\log\log \left(e^4c/2\right) + b_2 e^4} + \sqrt{b_1 e^4\log\log \left(e^4c/2\right) + b_2 e^4 + 4c}\right)\\
&\leq \sqrt{b_1 e^4\log\log \left(e^4c/2\right) + b_2 e^4} + \sqrt{c}
\end{align*}
\end{proof}

\noindent We prove Lemmas~\ref{prp:small_bound}--\ref{prp:var_bound} and Theorem~\ref{thm:conc_A2IPW} as follows.

\subsubsection*{Proof of Lemma~\ref{prp:small_bound}}
\begin{proof}
This proof mostly follows the proof of Theorem~24 of \citet{Balsubramani2014SharpFI}. 
First, by using Proposition~\ref{prp:super_mart}, we show that $2\geq \mathbb{E}\left[\exp\left(\lambda_0|M_\tau| - \lambda^2_0V_\tau\right)\right]$ for any stopping time $\tau$ and $\lambda\in\left(-\frac{1}{e^2}, \frac{1}{e^2}\right)$. From Proposition~\ref{prp:super_mart}, $U^\lambda_{t}:=\exp(\lambda M_t - \lambda^2V_t)$ is a super martingale. The condition that, for all $\ell \geq 3$, $\mathbb{E}[(M_{t} - M_{t-1})^\ell|  \mathcal{H}_{t-1}]\leq \frac{1}{2}\ell!\left(e/\sqrt{2}\right)^{2(\ell-2)}\mathbb{E}[(M_t - M_{t-1})^2|  \mathcal{H}_{t-1}]$ holds from the assumption that $|M_t - M_{t-1}|\leq e^2/2$ for all $t$ with probability $1$. For $\lambda_0\in\left(-\frac{1}{e^2}, \frac{1}{e^2}\right)$, let us consider a situation where $\lambda \in \{-\lambda_0, \lambda_0\}$ with probability $1/2$ each. After marginalizing over $\lambda$, the resulting process is 
\begin{align*}
&\widetilde{U}_t = \frac{1}{2}\exp(\lambda_0 M_t - \lambda^2_0V_t) + \frac{1}{2}\exp(-\lambda_0 M_t - \lambda^2_0V_t)\\
&\geq \frac{1}{2}\exp(\lambda_0 M_t - \lambda^2_0V_t).
\end{align*}
On the other hand, for any stopping time $\tau$, from the optimal stopping theorem for a super martingale \citep{Durrett}, we have
\begin{align*}
\mathbb{E}\left[\exp(\lambda_0 M_\tau - \lambda^2_0V_\tau)\right] \leq \mathbb{E}\left[\exp(\lambda_0 M_0 - \lambda^2_0V_0)\right] = 1.
\end{align*}
Similarly, 
\begin{align*}
\mathbb{E}\left[\exp(-\lambda_0 M_\tau - \lambda^2_0V_\tau)\right] \leq \mathbb{E}\left[\exp(-\lambda_0 M_0 - \lambda^2_0V_0)\right] = 1.
\end{align*}
Combining these results, we have 
\begin{align*}
&\mathbb{E}\left[\widetilde{U}_t\right] = \mathbb{E}\left[\frac{1}{2}\exp(\lambda_0 M_t - \lambda^2_0V_t) + \frac{1}{2}\exp(-\lambda_0 M_t - \lambda^2_0V_t)\right] \leq 1,
\end{align*}
and $1\geq \mathbb{E}\left[\frac{1}{2}\exp(\lambda_0 M_t - \lambda^2_0V_t)\right]$. Thus, we proved $2\geq \mathbb{E}\left[\exp\left(\lambda_0|M_\tau| - \lambda^2_0V_\tau\right)\right]$.

Next, note that $\tau_0=\min\left\{s : V_s\geq \frac{173}{2(e-2)} \log\left(\frac{4}{\delta}\right)\right\}$. Therefore, by defining the stopping time $\tau_1=\min\left\{s : |M_t| \geq 2\sqrt{\frac{173}{2(e-2)}} \log\left(\frac{4}{\delta}\right)\right\}$ and using $\lambda_0 = \sqrt{\frac{2(e-2)}{173}}\approx 0.091\leq \frac{1}{e^2}\approx 0.135$, 
\begin{align*}
&2 \geq \mathbb{E}\left[\exp\left(\lambda_0|M_{\tau_1}| - \lambda^2_0V_{\tau_1}\right)\right]\\
&\geq \mathbb{E}\left[\exp\left(\lambda_0|M_{\tau_1}| - \lambda^2_0V_{\tau_1}\right)|  \tau_1 < \tau_0\right]\mathbb{P}\left(\tau_1 < \tau_0\right)\\
&\geq \mathbb{E}\left[\exp\left(2\lambda_0\sqrt{\frac{173}{2(e-2)}} \log\left(\frac{4}{\delta}\right) - \lambda^2_0\frac{173}{2(e-2)}\log\left(\frac{4}{\delta}\right)\right)|  \tau_1 < \tau_0\right]\mathbb{P}\left(\tau_1 < \tau_0\right)\\
&\geq \mathbb{E}\left[\exp\left( \log\left(\frac{4}{\delta}\right)\right)|  \tau_1 < \tau_0\right]\mathbb{P}\left(\tau_1 < \tau_0\right) = \frac{4}{\delta}\mathbb{P}\left(\tau_1 < \tau_0\right).
\end{align*}
Thus, we obtain $\mathbb{P}\left(\tau_1 < \tau_0\right) \leq \frac{\delta}{2} < \delta$.
\end{proof}

\subsubsection*{Proof of Lemma~\ref{prp:known_conc}}
\begin{proof}
From Proposition~\ref{prp:uni_bern_bound}, with probability $\geq 1-\delta/2$, for all $t\geq \tau_0$ simultaneously, $\big|M_t\big|\leq \frac{2(e-2)}{e^2(1+\sqrt{1/3})} V_t$ and 
\begin{align*}
\big|M_t\big| \leq \sqrt{6(e-2) V_t\left(2\log\log\left(\frac{3(e-2) V_t}{\big|M_t\big|}\right) + \log\left(\frac{4}{\delta}\right)\right)}.
\end{align*}
We therefore have that, with probability $\geq 1-\delta/2$, for all $t\geq \tau_0$, simultaneously, $\big|M_t\big|\leq \frac{2(e-2)}{e^2(1+\sqrt{1/3})} V_t$ and 
\begin{align}
\label{bound0}
\big|M_t\big| \leq \max \left(3(e-2), \sqrt{2C_1  V_t\log\log V_t + C_1  V_t\log\left(\frac{4}{\delta}\right)}\right),
\end{align}
where note that $C_1 = 6(e-2)$.

Next, from Lemma~\ref{prp:small_bound}, with probability $\geq 1-\delta/4$, for all $t\leq \tau_0$ simultaneously,
\begin{align*}
|M_t| \leq 2\sqrt{\frac{173}{2(e-2)}} \log\left(\frac{4}{\delta}\right)
\end{align*}
By taking a union bound of \eqref{bound0}, with probability $\geq 1-\delta$, the following inequality holds for all $t$ simultaneously:
\begin{align*}
\big|M_t\big| \leq 
\begin{cases}
2\sqrt{\frac{173}{2(e-2)}} \log\left(\frac{4}{\delta}\right) & \mathrm{if}\ t\leq \tau_0\\
\frac{2(e-2)}{e^2(1+\sqrt{1/3})} V_t\ \mathrm{and}\ \max \left(3(e-2), \sqrt{2C_1  V_t\log\log V_t + C_1  V_t\log\left(\frac{4}{\delta}\right)}\right) & \mathrm{if}\ t\geq \tau_0.
\end{cases}
\end{align*}
With probability $\geq 1-\delta$, the following relationship holds for all $t$ simultaneously:
\begin{align*}
\left|M_t\right| \leq C_0(\delta) + \sqrt{C_1V_t\left(2\log\log V_t + \log\left(\frac{4}{\delta}\right)\right)}.
\end{align*}
\end{proof}

\subsubsection*{Proof of Lemma~\ref{prp:var_bound}}
\begin{proof}
Let $\tilde M_t$ be $\sum^t_{i=1}(M_i-M_{i-1})^2 - V_t$, where note that 
\[V_t = \sum^t_{i=1}\mathbb{E}\left[ (M_i - M_{i-1})^2|  \mathcal{H}_{i-1}\right].\] Suppose that there exists $C_z$ such that $\left|(M_t-M_{t-1})^2 - \mathbb{E}\left[ (M_i - M_{i-1})^2|  \mathcal{H}_{i-1}\right]\right| \leq  C_z$ with probability $1$ in which the existence is guaranteed by the boundedness of $M_i-M_{i-1}$, i.e., $|M_i-M_{i-1}|\leq e^2/2$ for all $t$ with probability $1$. Because $\tilde M_t$ is a martingale, we can apply Proposition~\ref{prp:known_conc}, i.e., for all $t$, with probability $\geq 1-\delta$
\begin{align*}
\left|\tilde M_t\right| \leq \frac{2C_z}{e^2}\left(C_0(\delta) + \sqrt{C_1B_t\left(2\log\log B_t + \log\left(\frac{4}{\delta}\right)\right)}\right),
\end{align*}
where $B_t = \mathbb{E}\left[\left(\sum^t_{i=1}(M_i-M_{i-1})^2 - V_t\right)^2|  \mathcal{H}_{i-1}\right]$. For $B_t$, we have
\begin{align*}
&B_t = \sum^t_{i=1}\left(\mathbb{E}\left[(M_i-M_{i-1})^4|  \mathcal{H}_{i-1}\right] - \left(\mathbb{E}\left[(M_i-M_{i-1})^2|  \mathcal{H}_{i-1}\right]\right)^2\right)\\
&\leq \sum^t_{i=1}\mathbb{E}\left[(M_i-M_{i-1})^4|  \mathcal{H}_{i-1}\right]\leq \big(e^8/2^4\big)\sum^t_{i=1}\mathbb{E}\left[(M_i-M_{i-1})^4/(e^8/2^4)|  \mathcal{H}_{i-1}\right]
\end{align*}
Because $M_i-M_{i-1} \leq e^2/2 \rightarrow \frac{(M_i-M_{i-1})^2}{e^4/2^2} \leq 1$, we have $(M_i-M_{i-1})^2/(e^4/2^2) \geq (M_i-M_{i-1})^4/(e^8/2^4)$, and
\begin{align}
&\sum^t_{i=1}\mathbb{E}\left[(M_i-M_{i-1})^4|  \mathcal{H}_{i-1}\right]\leq e^8/2^4\sum^t_{i=1}\mathbb{E}\left[(M_i-M_{i-1})^2/(e^4/2^2)|  \mathcal{H}_{i-1}\right] = e^4V_t/4.
\end{align}
Therefore, 
\begin{align*}
&\left|\tilde M_t\right| \leq \frac{2C_z}{e^2}\left(C_0(\delta) + \sqrt{C_1B_t\left(2\log\log B_t + \log\left(\frac{4}{\delta}\right)\right)}\right)\\
&\leq \frac{2C_z}{e^2}\left(C_0(\delta) + \sqrt{C_1e^4V_t/4\left(2\log\log \big(e^4V_t/4\big) + \log\left(\frac{4}{\delta}\right)\right)}\right).
\end{align*}
This can be relaxed to
\begin{align*}
&- \sum^t_{i=1}(M_i-M_{i-1})^2 + V_t - \frac{2C_z}{e^2}\left(C_0(\delta) + \sqrt{C_1e^4V_t/4\left(2\log\log \big(e^4V_t/4\big) + \log\left(\frac{4}{\delta}\right)\right)}\right)\\
&= - \sum^t_{i=1}(M_i-M_{i-1})^2 + V_t - \left(\frac{2C_zC_0(\delta)}{e^2} + \sqrt{\frac{C^2_4C_1}{e^4}e^4V_t\left(2\log\log \big(e^4V_t/4\big) + \log\left(\frac{4}{\delta}\right)\right)}\right)\\
&\leq 0.
\end{align*}
We consider two cases for $V_t$. First, we consider a case where \[V_t \geq 8\max\left(e^4\frac{C^2_4C_1}{e^4} 2\log \log \left(e^4V_t\right), e^4\frac{C^2_4C_1}{e^4}\log\left(\frac{4}{\delta}\right)\right).\] From Proposition~\ref{prp:lemma9}, we have
\begin{align*}
\sqrt{V_t} \leq& \sqrt{\frac{C^2_4C_1}{e^4} 2e^4\log \log \left(e^2C_zC_0(\delta)+ e^4\sum^t_{i=1}(M_i-M_{i-1})^2/2\right) + e^4\frac{C^2_4C_1}{e^4}\log\left(\frac{4}{\delta}\right)}\\
& + \sqrt{\frac{2C_zC_0(\delta)}{e^2}+ \sum^t_{i=1}(M_i-M_{i-1})^2}\\
=&\sqrt{2C^2_4C_1 \log \log \left(e^2C_zC_0(\delta)+ e^4\sum^t_{i=1}(M_i-M_{i-1})^2/2\right) + C^2_4C_1\log\left(\frac{4}{\delta}\right)}\\
& + \sqrt{\frac{2C_zC_0(\delta)}{e^2}+ \sum^t_{i=1}(M_i-M_{i-1})^2}.
\end{align*}
For sufficiently large $\sum^t_{i=1}(M_i-M_{i-1})^2$ such that
\[2C^2_4C_1 \log \log \left(e^2C_zC_0(\delta)+ e^4\sum^t_{i=1}(M_i-M_{i-1})^2/2\right) \geq C^2_4C_1\log\left(\frac{4}{\delta}\right),\] by using a constant $C_5$, the RHS is bounded as 
\begin{align*}
&\sqrt{2C^2_4C_1 \log \log \left(e^2C_zC_0(\delta)+ e^4\sum^t_{i=1}(M_i-M_{i-1})^2/2\right) + C^2_4C_1\log\left(\frac{4}{\delta}\right)}\\
&\ \ \ \ \ \ \ \ \ \ \ \ \ \ \ \ \ \ \ \ \ \ \ \ \ \ \ \ \ \ \ \ \ \ \ \ \ \ \ \ \ \ \ \ \ \ \ \ \ \ \ \ \ \ \ \ \ \ \ \ \ \ \ \ \ \ +\sqrt{\frac{2C_zC_0(\delta)}{e^2}+ \sum^t_{i=1}(M_i-M_{i-1})^2}\\
&\leq \sqrt{4C^2_4C_1 \log \log \left(e^2C_zC_0(\delta)+ e^4\sum^t_{i=1}(M_i-M_{i-1})^2/2\right)}+ \sqrt{\frac{2C_zC_0(\delta)}{e^2}+ \sum^t_{i=1}(M_i-M_{i-1})^2}\\
&\leq \sqrt{4C^2_4C_1 \left(e^2C_zC_0(\delta)+ e^4\sum^t_{i=1}(M_i-M_{i-1})^2/2\right)}+ \sqrt{\frac{2C_zC_0(\delta)}{e^2}+ \sum^t_{i=1}(M_i-M_{i-1})^2}.
\end{align*}
By squaring both sides of  
\begin{align*}
\sqrt{V_t} &\leq \sqrt{4C^2_4C_1 \left(e^2C_zC_0(\delta)+ e^4\sum^t_{i=1}(M_i-M_{i-1})^2/2\right)}+ \sqrt{\frac{2C_zC_0(\delta)}{e^2}+ \sum^t_{i=1}(M_i-M_{i-1})^2}\\
& = \sqrt{2e^4C^2_4C_1 \left(\frac{2C_zC_0(\delta)}{e^2}+ \sum^t_{i=1}(M_i-M_{i-1})^2\right)}+ \sqrt{\frac{2C_zC_0(\delta)}{e^2}+ \sum^t_{i=1}(M_i-M_{i-1})^2},,
\end{align*}
we obtain 
\begin{align*}
V_t \leq C_f \left(\sum^t_{i=1}(M_{i} - M_{i-1})^2 + \frac{2C_zC_0(\delta)}{e^2}\right),
\end{align*}
where $C_f$ is a constant. When $V_t < 8\max\left(e^4\frac{C^2_4C_1}{e^4} 2\log \log \left(e^4V_t\right), e^4\frac{C^2_4C_1}{e^4}\log\left(\frac{4}{\delta}\right)\right)$, the statement clearly holds for sufficiently high $V_t$ such that $V_t<e^4\frac{C^2_4C_1}{e^4} 2\log \log \left(e^4V_t\right)$.
\end{proof}

\subsubsection*{Proof of Theorem~\ref{thm:conc_A2IPW}}
Finally, combining the above results, we show Theorem~\ref{thm:conc_A2IPW} as follows.
\begin{proof}
Let us note that we can construct a martingale difference sequence from $z_t = q_t - \theta_0$ as $\{z_t\}^T_{t=1}$. Let us suppose that there exists a constant $C$ such that $|z_t| \leq C$. Let $\tilde z_t$ and $\tilde V_t$ be  $z_te^2/(2C)$ and $\sum^t_{i=1}\mathbb{E}\big[ \widetilde{z}^2_i\mid \mathcal{H}_{i-1}\big]$, respectively. From this boundedness of $z_t$, there exists a constant $C_z$ such that $|z^2_t - \mathbb{E}[z_t^2|  \mathcal{H}_{t-1}]| \leq C_z$. For fixed $\delta$, from Proposition~\ref{prp:known_conc}, with probability $\geq 1-\delta$, the following true for all $t$ simultaneously:
\begin{align*}
\left|t\widehat{\theta}^{\mathrm{A2IPW}}_t-t\theta_0\right| \leq \frac{2C}{e^2}\left(C_0(\delta) + \sqrt{2C_1\widetilde{V}^*_t\left(\log\log\widetilde{V}^*_t + \log\left(\frac{4}{\delta}\right)\right)}\right).
\end{align*}
Here, by using Proposition~\ref{prp:var_bound}, we have
\begin{align*}
\tilde V_t \leq C_f \left(\sum^t_{i=1}\widetilde{z}^2_i + \frac{2C_zC_0(\delta)}{e^2}\right),
\end{align*}
Then, we have
\begin{align*}
&\left|t\widehat{\theta}^{\mathrm{A2IPW}}_t-t\theta_0\right|\\
&\leq \frac{2C}{e^2}\Bigg(C_0(\delta) + \\
&\sqrt{2C_1C_f\left(\frac{e^4}{4C^2}\sum^t_{i=1}z^2_t + \frac{2C_zC_0(\delta)}{e^2}\right)\left(\log\log C_f \left(\frac{e^4}{4C^2}\sum^t_{i=1}z^2_t + \frac{2C_zC_0(\delta)}{e^2}\right) + \log\left(\frac{4}{\delta}\right)\right)}\Bigg).
\end{align*}
\end{proof}

\section{Proofs of Section~\ref{sec:expect}}

\subsection{Proof of Lemma~\ref{lem:proof}}
\label{appdx:proof_prob}
We have
\begin{align*}
&\mathbb{P}_{H_1}(\tau > \widetilde{t}) = 1 - \mathbb{P}_{H_1}(\tau\leq \widetilde{t})\\
& = 1 - \mathbb{P}_{H_1}\left(\exists t \leq \widetilde{t} : \left| t\widehat{\theta}^{\mathrm{A2IPW}}_t - t\mu \right| > q_t\right)\\
&\leq 1 - \mathbb{P}_{H_1}\left(\left|\widetilde{t}\widehat{\theta}^{\mathrm{A2IPW}}_{\widetilde{t}} - \widetilde{t}\mu \right|> q_{\widetilde{t}}\right)\\
& = \mathbb{P}_{H_1}\left(\left|\widetilde{t}\widehat{\theta}^{\mathrm{A2IPW}}_{\widetilde{t}} - \widetilde{t}\mu \right| \leq q_{\widetilde{t}}\right)\\
& = \mathbb{P}_{H_1}\left( -q_{\widetilde{t}} \leq \widetilde{t}\widehat{\theta}^{\mathrm{A2IPW}}_{\widetilde{t}} - \widetilde{t}\mu  \leq q_{\widetilde{t}}\right)\\
& = \mathbb{P}_{H_1}\left( -q_{\widetilde{t}} - \widetilde{t}\Delta \leq \widetilde{t}\widehat{\theta}^{\mathrm{A2IPW}}_{\widetilde{t}} - \widetilde{t}\mu - \widetilde{t}\Delta \leq q_{\widetilde{t}} - \widetilde{t}\Delta\right)\\
& \leq \mathbb{P}_{H_1}\left(\widetilde{t}\widehat{\theta}^{\mathrm{A2IPW}}_{\widetilde{t}} - \widetilde{t}\mu - \widetilde{t}\Delta \leq q_{\widetilde{t}} - \widetilde{t}\Delta\right).
\end{align*}
By substituting $q_{\widetilde{t}} = 1.1\left(\log \left(\frac{1}{\alpha}\right) + \sqrt{2\sum^{\widetilde{t}}_{i=1}z^2_i\left(\log\frac{\log \sum^{\widetilde{t}}_{i=1}z^2_i}{\alpha}\right)}\right)$, 
\begin{align*}
&\mathbb{P}_{H_1}(\tau > \widetilde{t})\\
&\leq \mathbb{P}_{H_1}\left(\widetilde{t}\widehat{\theta}^{\mathrm{A2IPW}}_{\widetilde{t}} - \widetilde{t}\mu - \widetilde{t}\Delta \leq 1.1\left(\log \left(\frac{1}{\alpha}\right) + \sqrt{2\sum^{\widetilde{t}}_{i=1}z^2_i\left(\log\frac{\log \sum^{\widetilde{t}}_{i=1}z^2_i}{\alpha}\right)}\right) - \widetilde{t}\Delta\right)\\
& = \mathbb{P}_{H_1}\left(\frac{\widetilde{t}\widehat{\theta}^{\mathrm{A2IPW}}_{\widetilde{t}} - \widetilde{t}\mu - \widetilde{t}\Delta}{\sqrt{\widetilde{\sigma}^2}} \leq \frac{1.1}{\sqrt{\widetilde{\sigma}^2}}\left(\log \left(\frac{1}{\alpha}\right) + \sqrt{2\sum^{\widetilde{t}}_{i=1}z^2_i\left(\log\frac{\log \sum^{\widetilde{t}}_{i=1}z^2_i}{\alpha}\right)}\right) - \frac{\widetilde{t}\Delta}{\sqrt{\widetilde{\sigma}^2}}\right)\\
& \leq \mathbb{P}_{H_1}\left(\frac{\widetilde{t}\widehat{\theta}^{\mathrm{A2IPW}}_{\widetilde{t}} - \widetilde{t}\mu - \widetilde{t}\Delta}{\sqrt{\widetilde{\sigma}^2}} \leq \frac{1.1}{\sqrt{\widetilde{\sigma}^2}}\left(\log \left(\frac{1}{\alpha}\right) + \sqrt{2C^2 \widetilde{t}\left(\log\frac{\log C^2 \widetilde{t}}{\alpha}\right)}\right) - \frac{\widetilde{t}\Delta}{\sqrt{\widetilde{\sigma}^2}}\right).
\end{align*}
Here, we used $|z_t| \leq C$ for all $t$. Let $\preceq$ and $\asymp$ be $\leq$ and $=$ when ignoring constants. By using Azuma-Heoffding inequality for martingales \citep{hoeffding1963probability,azuma1967}, $|z_t - z_{t-1}|\leq 2C$, and $\widetilde{t}\Delta \gg 1.1\left(\log \left(\frac{1}{\alpha}\right) + \sqrt{2C^2 \widetilde{t}\left(\log\frac{\log C^2 \widetilde{t}}{\alpha}\right)}\right)$,
\begin{align*}
&\mathbb{P}_{H_1}(\tau > \widetilde{t})\\
&\leq \mathbb{P}_{H_1}\left( \widetilde{t}\widehat{\theta}^{\mathrm{A2IPW}}_{\widetilde{t}} - \widetilde{t}\mu - \widetilde{t}\Delta \leq 1.1\left(\log \left(\frac{1}{\alpha}\right) + \sqrt{2C^2 \widetilde{t}\left(\log\frac{\log C^2 \widetilde{t}}{\alpha}\right)}\right) - \frac{\widetilde{t}\Delta}{\sqrt{\widetilde{\sigma}^2}} \right)\\
&\leq \exp\left(-\frac{\left(\widetilde{t}\Delta - 1.1\left(\log \left(\frac{1}{\alpha}\right) + \sqrt{2C^2 \widetilde{t}\left(\log\frac{\log C^2 \widetilde{t}}{\alpha}\right)}\right)\right)^2}{8\widetilde{t}C^2}\right)\\
&\asymp \exp\left(-\frac{\widetilde{t}\Delta^2}{8C^2}\right).
\end{align*}

\subsection{Proof of Theorem~\ref{thm:exp_sample}}
\label{appdx:exp_sample}
For $n^{\mathrm{OPT}*}_{\beta}(\Delta) = \frac{\widetilde{\sigma}^2}{\Delta^2}\big(z_{1-\alpha/2} - z_{\beta}\big)^2$, we have
\begin{align*}
\mathbb{E}_{H_1}[\tau] &= \sum_{n\geq 1} \mathbb{P}_{H_1}(\tau > n )\\
&\leq n^{\mathrm{OPT}*}_{\beta}(\Delta) + \sum_{t\geq n^{\mathrm{OPT}*}_{\beta}(\Delta)+1} \mathbb{P}_{H_1}(\tau > t)\\
&\leq n^{\mathrm{OPT}*}_{\beta}(\Delta) + \sum_{t\geq n^{\mathrm{OPT}*}_{\beta}(\Delta)-1} \mathbb{P}_{H_1}(\tau > t)\\
&\preceq n^{\mathrm{OPT}*}_{\beta}(\Delta) + \sum^\infty_{t\geq n^{\mathrm{OPT}*}_{\beta}(\Delta)-1} \exp\left(-\frac{t\Delta^2}{8C^2}\right)\\
&= n^{\mathrm{OPT}*}_{\beta}(\Delta) + \exp\left(-\frac{\left(n^{\mathrm{OPT}*}_{\beta}(\Delta) - 1 \right)\Delta^2}{8C^2}\right) + \exp\left(-\frac{n^{\mathrm{OPT}*}_{\beta}(\Delta) \Delta^2}{8C^2}\right) + \cdots\\
&= n^{\mathrm{OPT}*}_{\beta}(\Delta) + \exp\left(-\frac{\left(n^{\mathrm{OPT}*}_{\beta}(\Delta) - 1\right)\Delta^2}{8C^2}\right) \sum^\infty_{s=1} \exp\left(-\frac{(s-1)\Delta^2}{8C^2}\right).
\end{align*}
By using the infinite geometric series sum formula,
\begin{align*}
&n^{\mathrm{OPT}*}_{\beta}(\Delta) + \exp\left(-\frac{\left(n^{\mathrm{OPT}*}_{\beta}(\Delta) - 1\right)\Delta^2}{8C^2}\right) \sum^\infty_{s=1} \exp\left(-\frac{(s-1)\Delta^2}{8C^2}\right)\\
&= n^{\mathrm{OPT}*}_{\beta}(\Delta) + \exp\left(-\frac{ \left(n^{\mathrm{OPT}*}_{\beta}(\Delta) - 1\right) \Delta^2}{8C^2}\right) \frac{1}{1 - \exp\left(-\frac{\Delta^2}{8C^2}\right)}\\
&= n^{\mathrm{OPT}*}_{\beta}(\Delta) + \exp\left(-\frac{ n^{\mathrm{OPT}*}_{\beta}(\Delta)  \Delta^2}{8C^2}\right) \frac{1}{\exp\left(\frac{\Delta^2}{8C^2}\right) - 1}.
\end{align*}
By substituting $\exp\left(-\frac{\widetilde{t}\Delta^2}{8C^2}\right)\asymp \mathbb{P}_{H_1}(\tau > \widetilde{t})$,
\begin{align*}
\mathbb{E}_{H_1}[\tau] \preceq n^{\mathrm{OPT}*}_{\beta}(\Delta) + \frac{\mathbb{P}_{H_1}(\tau > n^{\mathrm{OPT}*}_{\beta}(\Delta) )}{\exp\left(\frac{\Delta^2}{8C^2}\right)-1}.
\end{align*}
Using the inequality, $1-\exp(-r) \leq r$, and $n^{\mathrm{OPT}*}_{\beta}(\Delta) = \frac{\widetilde{\sigma}^2}{\Delta^2}\big(z_{1-\alpha/2}-z_{\beta}\big)^2$, we have
\begin{align*}
&\mathbb{E}_{H_1}[\tau]\\
&\preceq n^{\mathrm{OPT}*}_{\beta}(\Delta) + \frac{8C^2}{\Delta^2}\mathbb{P}_{H_1}(\tau > n^{\mathrm{OPT}*}_{\beta}(\Delta))\\
&= n^*_{\beta}(\Delta) + \frac{8C^2n^{\mathrm{OPT}*}_{\beta}(\Delta)}{\widetilde{\sigma}^2\big(z_{1-\alpha/2}-z_{\beta}\big)^2}\mathbb{P}_{H_1}(\tau > n^{\mathrm{OPT}*}_{\beta}(\Delta)).
\end{align*}

\section{Additional Experimental Results}
\label{appdx:additional_exp_result}
\subsection{Setting}
In this section, we investigate the empirical performance of the proposed A2IPW and MA2IPW estimators, as well as the ADR estimator introduced in our follow-up study \citep{Kato2021adr}. The simulation setting follows the framework of \citet{Meehan2022} and \citet{athey2016}. 
We generate the covariates \(X_t\) from a beta distribution \(\mathrm{Beta}(2,5)\) on \(\mathcal{X} = [0,1]^d\). The potential outcomes and covariates follow
\[
Y_t(a) = \kappa_a(X_t) + \nu_a(X_t) \,\epsilon_{a, t},
\]
where \(\epsilon_{a,t}\sim N(0, 0.1)\). 


We adopt this functional form to incorporate both observed covariates \(X_t\) and unobserved noise through \(\epsilon_{a,t}\). In practice, the outcome depends on measurable characteristics \(X_t\), captured through \(\kappa_a(\cdot)\) and \(\nu_a(\cdot)\), as well as latent factors that can differ across treatment arms. The coefficient functions \(\kappa_a\) and \(\nu_a\) thus allow for heteroskedasticity and heterogeneous treatment effects shaped by both observables and unobservables.

We consider three settings with different specifications of $\kappa_a(\cdot)$ and $\nu_a(\cdot)$:
\begin{description}
    \item[Model~1] ATE: $\theta_0 = 0.12$.
    \begin{itemize}
        \item Dimension of $X_t$: $d = 2$. 
        \item $\kappa_0(x) = 0.2$, $\kappa_1(x) = 10x_1^2\mathbbm{1}[x_1 > 0.4] - 5x_2^2\mathbbm{1}[x_2 > 0.4]$.
        \item $\nu_0(x) = 5$, $\nu_1(x) = 1 + 10x_1^2\mathbbm{1}[x_1 > 0.6] + 5x_2^2\mathbbm{1}[x_2 > 0.6]$.
    \end{itemize}
    \item[Model~2] ATE: $\theta_0 = 0.079$.
    \begin{itemize}
        \item Dimension of $X_t$: $d = 10$. 
        \item $\kappa_0(x) = 0.5$, $\kappa_1(x) = \sum_{j=1}^{10} (-1)^{j-1}10^{-j+2}x_j^2\mathbbm{1}[x_j > 0.4]$.
        \item $\nu_0(x) = 5$, $\nu_1(x) = 1 + \sum_{j=1}^{10}10^{-j+2}x_j^2\mathbbm{1}[x_j > 0.6]$.
    \end{itemize}
    \item[Model~3] ATE: $\theta_0 = 0.12$.
    \begin{itemize}
        \item Dimension of $X_t$: $d = 10$. 
        \item $\kappa_0(x) = 0.2$, $\kappa_1(x) = \sum_{j=1}^{3}(-1)^{j-1}10x_j^2\cdot\mathbbm{1}[x_j > 0.4] + \sum_{j=4}^{10}(-1)^{j-1}5x_j^2\cdot\mathbbm{1}[x_j > 0.4]$.
        \item $\nu_0(x) = 9$, $\nu_1(x) = 1 + \sum_{j=1}^{3}10x_j^2\cdot\mathbbm{1}[x_j > 0.6] + \sum_{j=4}^{10}5x_j^2\cdot\mathbbm{1}[x_j > 0.6]$.
    \end{itemize}
\end{description}


In Section~\ref{sec:compare_a2ipw}, we evaluate the performance of the proposed estimators in both single-stage and sequential testing settings. 
In Section~\ref{sec:compare_strat}, we compare our proposed methods with those of \citet{Meehan2022}.

\subsection{Comparison among A2IPW Estimator, ADR Estimator, Single-stage Testing, and Sequential Testing}
\label{sec:compare_a2ipw}
We compare the performance of the proposed methods with a randomized controlled trial (RCT) where treatments are assigned with probability 0.5. We also compare with an oracle algorithm that knows the true variances of the potential outcomes and uses the optimal estimator \(\hat{\theta}^\mathrm{OPT}_T\).

For all settings, the null and alternative hypotheses are defined as $\mathcal{H}_0: \theta_0 = 0$ and $\mathcal{H}_1: \theta_0 \neq 0$, respectively. 
We conduct the following tests:
\begin{itemize}
    \item \textbf{Standard hypothesis testing:} Performed with $T$-statistics when the sample sizes are $1000$ and $5000$.
    \item \textbf{Sequential testing with Bonferroni correction:} Multiple testing is conducted at sample sizes $1000$, $2000$, $3000$, and $5000$.
    \item \textbf{Sequential testing with LIL:} Testing is based on the concentration inequality derived in Theorem~\ref{thm:conc_A2IPW}.
\end{itemize}
We compare different tests in terms of hypothesis testing power, precision, and efficiency under various scenarios.

We evaluate the methods in terms of power, MSE, and coverage for sample sizes up to \(5000\). Each simulation is repeated \(500\) times. Tables~\ref{tab:res1}--\ref{tab:res3} report the mean squared error (MSE), standard deviation of the squared error (SMSE), rejection rates (R/R), coverage ratios (CR), and stopping times for LIL and BF-based testing (BC and LIL columns). 

For Models~1--3, the A2IPW and ADR estimators generally achieve smaller MSE than the RCT baseline at larger sample sizes. In Table~\ref{tab:res1} (Model~1), A2IPW with kernel-based nuisance estimators attains lower MSE compared to RCT, reflecting the benefit of adaptive treatment assignment, while still controlling type I error when the null is true (Table~\ref{tab:res1_null}). Similar trends appear in Tables~\ref{tab:res2} and \ref{tab:res2_null} (Model~2), although the rejection rates differ owing to the smaller true effect size \(\theta_0 = 0.079\). For Model~3 (Tables~\ref{tab:res3} and \ref{tab:res3_null}), the presence of stronger heteroskedasticity causes slightly higher MSE across methods, but A2IPW and ADR still often outperform the RCT in precision. The oracle (not shown in every row but referenced for comparison) serves as a theoretical benchmark, consistently achieving the lowest MSE due to its knowledge of the true variances.

The sequential methods terminate earlier than a fixed-sample analysis if the observed data yield strong evidence. The Bonferroni correction can lead to earlier stops but occasionally increases type I error, whereas the LIL-based approach is often more conservative, as seen by larger average stopping times (BC vs. LIL columns in Tables~\ref{tab:res1}--\ref{tab:res3}). When \(\theta_0 = 0\) in Tables~\ref{tab:res1_null}--\ref{tab:res3_null}, both sequential approaches correctly fail to reject the null in most cases, albeit sometimes not until nearing the maximum sample size.

\subsection{Comparison with the Stratification Tree}
\label{sec:compare_strat}

In this section, we compare our proposed method with the stratification tree approach introduced in \citet{Meehan2022}. Specifically, we evaluate the proposed A2IPW estimators against the following alternative methods:

\begin{description}
    \item[Ad-hoc.] In this method, experimental units are stratified using an ``ad-hoc'' approach, and treatments are assigned to half the sample in each stratum.
    \item[Ad-hoc $+$ Neyman.] This is a two-stage experiment. In the first stage, the variances of the outcomes are estimated, and treatments are then allocated according to Neyman allocation. The ATE is subsequently estimated by averaging the sample mean of each stratum weighted by the probability that covariates fall into the stratum.
    \item[Stratification Tree \citep{Meehan2022}.] This method uses a two-stage experiment. In the first stage, a stratification tree is estimated. In the second stage, treatments are assigned using the estimated tree. The tree depth is fixed at three.
    \item[Cross-Validated Tree.] This method is similar to the Stratification Tree but selects the tree depth via 2-fold cross-validation.
\end{description}

For two-stage experiments, we consider three different sample size ratios between the first and second stages, using $100$, $500$, and $1500$ experimental units for the first stage. For detailed descriptions of each method, refer to \citet{Meehan2022}.

The total sample size $T$ is set at $5000$, and we conducted $1000$ independent trials. For each case at round $5000$, we report the MSE between $\theta$ and $\hat{\theta}$, the standard deviation of the squared error, the percentages of hypothesis rejections using $T$-statistics, and the coverage ratios. The results are presented in Tables~\ref{tab:res1}--\ref{tab:res3}.


We consider pilot sample sizes of \(100\), \(500\), or \(1500\) in the first stage for these two-stage methods. Full details appear in \citet{Meehan2022}. The total sample size is \(T=5000\), and the study is repeated \(1000\) times for each scenario. Tables~\ref{tab:res1}--\ref{tab:res3} compare the MSE, SMSE, rejection rates, and coverage (CR) across all methods.

The results generally show that the adaptive methods (A2IPW, ADR) achieve competitive or lower MSE compared to the stratification-based approaches. The performance advantage is especially noticeable when the pilot stage is small (for example, \(100\) units), since the two-stage stratification designs have less information to guide allocations in the second stage. However, when the pilot grows larger (such as \(1500\)), the performance of the two-stage methods can improve and sometimes approach that of the fully adaptive methods. The cross-validated tree often performs better than the fixed-depth tree, illustrating the importance of tuning the tree depth to capture heterogeneous effects. In scenarios with substantial heteroskedasticity (Model~3), the adaptive weighting in A2IPW and ADR generally yields more stable and accurate estimates relative to the stratification-based methods.


Overall, these findings confirm that adaptive procedures, such as A2IPW and ADR, effectively leverage ongoing data to update treatment allocation probabilities, leading to improved estimation and greater testing power under nonzero effects. Stratification approaches, especially with sufficient pilot data, can also offer good performance but may be more sensitive to the initial estimation of variances or tree-based splits. The choice of method should be guided by practical considerations, including available pilot data, computational resources, and how quickly strong evidence of treatment differences is needed.

\clearpage 

\begin{table}[h]
    \centering
     \caption{Simulation results of Model~1.}
         \label{tab:res1}
     \scalebox{0.95}{
    \begin{tabular}{ll|rrrrrr}
    \toprule
    \multicolumn{2}{c}{Method} & \multicolumn{6}{c}{Criteria} \\
     & Nuisance & MSE & SMSE & R/R & CR & BC & LIL \\
    \midrule
    \multirow{3}{*}{A2IPW} & KNN & 0.050 & 0.032 & 0.638 & 0.972 & 2228.000 & 2529.696 \\
     & NW & 0.056 & 0.037 & 0.520 & 0.952 & 2392.000 & 2798.364 \\
     & NN & 0.056 & 0.040 & 0.516 & 0.958 & 2394.000 & 2899.466 \\
     \cline{1-8}
    \multirow{3}{*}{ADR} & KNN & 0.050 & 0.032 & 0.644 & 0.970 & 2234.000 & 2549.920 \\
     & NW & 0.052 & 0.032 & 0.636 & 0.950 & 2260.000 & 2596.436 \\
     & NN & 0.061 & 0.043 & 0.616 & 0.932 & 2430.000 & 2912.940 \\
     \cline{1-8}
    \multirow{3}{*}{A2IPW (Oracle)} & KNN & 0.049 & 0.033 & 0.658 & 0.972 & 2208.000 & 2512.802 \\
     & NW & 0.051 & 0.031 & 0.602 & 0.958 & 2212.000 & 2596.902 \\
     & NN & 0.055 & 0.039 & 0.534 & 0.954 & 2416.000 & 2856.816 \\
     \cline{1-8}
    RCT &  & 0.056 & 0.039 & 0.540 & 0.966 & 2330.000 & 2819.306 \\
    \hline
    \hline
    & Pilot & MSE & SMSE & R/R & CR & &  \\
    \hline
    \multirow{3}{*}{Ad-hoc} & 100 & 0.053 & 0.037 & 0.610 & 0.964 &  &  \\
     & 500 & 0.053 & 0.037 & 0.592 & 0.964 &  &  \\
     & 1500 & 0.056 & 0.038 & 0.554 & 0.946 &  &  \\
    \cline{1-8}
    \multirow{3}{*}{Ad-hoc Neyman} & 100  & 0.054 & 0.038 & 0.564 & 0.962 &  &  \\
     & 500  & 0.052 & 0.037 & 0.626 & 0.946 &  &  \\
     & 1500  & 0.056 & 0.037 & 0.600 & 0.932 &  &  \\
    \cline{1-8}
    \multirow{3}{*}{Strat. Tree} & 100  & 0.059 & 0.044 & 0.508 & 0.964 &  &  \\
     & 500  & 0.053 & 0.036 & 0.604 & 0.944 &  &  \\
     & 1500  & 0.054 & 0.034 & 0.594 & 0.946 &  &  \\
    \cline{1-8}
    \multirow{3}{*}{CV Tree} & 100  & 0.054 & 0.035 & 0.574 & 0.938 &  &  \\
     & 500  & 0.053 & 0.035 & 0.594 & 0.950 &  &  \\
     & 1500  & 0.053 & 0.034 & 0.582 & 0.948 &  &  \\
    \bottomrule
    \end{tabular}
    }
\end{table}

\begin{table}[h]
    \centering
     \caption{Simulation results of Model~2.}
         \label{tab:res2}
     \scalebox{0.95}{
\begin{tabular}{ll|rrrrrr}
\toprule
    \multicolumn{2}{c}{Method} & \multicolumn{6}{c}{Criteria} \\
     & Nuisance & MSE & SMSE & R/R & CR & BC & LIL \\
\midrule
\multirow{3}{*}{A2IPW} & KNN & 0.052 & 0.034 & 0.302 & 0.958 & 2748.000 & 3446.008 \\
 & NW & 0.060 & 0.043 & 0.266 & 0.946 & 2868.000 & 3540.720 \\
 & NN & 0.056 & 0.040 & 0.258 & 0.952 & 2890.000 & 3645.408 \\
\cline{1-8}
\multirow{3}{*}{ADR} & KNN & 0.052 & 0.034 & 0.320 & 0.954 & 2704.000 & 3398.308 \\
 & NW & 0.059 & 0.043 & 0.472 & 0.926 & 2550.000 & 3194.026 \\
 & NN & 0.084 & 0.074 & 0.632 & 0.818 & 2756.000 & 3473.814 \\
\cline{1-8}
\multirow{3}{*}{A2IPW (Oracle)} & KNN & 0.052 & 0.034 & 0.324 & 0.962 & 2744.000 & 3431.802 \\
 & NW & 0.053 & 0.034 & 0.306 & 0.946 & 2812.000 & 3382.768 \\
 & NN & 0.055 & 0.037 & 0.288 & 0.940 & 2878.000 & 3604.940 \\
\cline{1-8}
RCT &  & 0.058 & 0.040 & 0.264 & 0.948 & 2862.000 & 3511.908 \\
    \hline
    \hline
    & Pilot & MSE & SMSE & R/R & CR & &  \\
    \hline
\multirow{3}{*}{Ad-hoc} & 100 & 0.050 & 0.036 & 0.286 & 0.982 &  &  \\
 & 500 & 0.054 & 0.038 & 0.294 & 0.970 &  &  \\
 & 1500 & 0.052 & 0.037 & 0.294 & 0.976 &  &  \\
\cline{1-8}
\multirow{3}{*}{Ad-hoc Neyman} & 100 & 0.054 & 0.039 & 0.278 & 0.966 &  &  \\
 & 500 & 0.052 & 0.036 & 0.276 & 0.970 &  &  \\
 & 1500 & 0.051 & 0.035 & 0.334 & 0.982 &  &  \\
\cline{1-8}
\multirow{3}{*}{Strat. Tree} & 100 & 0.057 & 0.046 & 0.246 & 0.972 &  &  \\
 & 500 & 0.049 & 0.034 & 0.276 & 0.976 &  &  \\
 & 1500 & 0.051 & 0.034 & 0.340 & 0.968 &  &  \\
\cline{1-8}
\multirow{3}{*}{CV Tree} & 100 & 0.051 & 0.035 & 0.304 & 0.976 &  &  \\
 & 500 & 0.050 & 0.033 & 0.326 & 0.966 &  &  \\
 & 1500 & 0.049 & 0.034 & 0.320 & 0.978 &  &  \\
\cline{1-8}
\bottomrule
\end{tabular}
}
\end{table}

\begin{table}[h]
    \centering
     \caption{Simulation results of Model~3.}
         \label{tab:res3}
     \scalebox{0.95}{
\begin{tabular}{ll|rrrrrr}
\toprule
    \multicolumn{2}{c}{Method} & \multicolumn{6}{c}{Criteria} \\
     & Nuisance & MSE & SMSE & R/R & CR & BC & LIL \\
\midrule
\multirow{3}{*}{A2IPW} & KNN & 0.082 & 0.054 & 0.314 & 0.944 & 2774.000 & 3504.320 \\
 & NW & 0.096 & 0.073 & 0.240 & 0.954 & 2886.000 & 3603.502 \\
 & NN & 0.093 & 0.068 & 0.238 & 0.948 & 2914.000 & 3523.114 \\
\cline{1-8}
\multirow{3}{*}{ADR} & KNN & 0.081 & 0.052 & 0.304 & 0.950 & 2720.000 & 3571.318 \\
 & NW & 0.079 & 0.052 & 0.332 & 0.952 & 2672.000 & 3516.880 \\
 & NN & 0.088 & 0.058 & 0.276 & 0.946 & 2854.000 & 3508.498 \\
\cline{1-8}
\multirow{3}{*}{A2IPW (Oracle)} & KNN & 0.082 & 0.055 & 0.302 & 0.934 & 2740.000 & 3537.698 \\
 & NW & 0.079 & 0.053 & 0.306 & 0.960 & 2692.000 & 3459.326 \\
 & NN & 0.088 & 0.062 & 0.254 & 0.954 & 2852.000 & 3606.298 \\
\cline{1-8}
RCT &  & 0.087 & 0.063 & 0.270 & 0.960 & 2816.000 & 3577.846 \\
    \hline
    \hline
    & Pilot & MSE & SMSE & R/R & CR & &  \\
    \hline
\multirow{3}{*}{Ad-hoc} & 100 & 0.083 & 0.058 & 0.284 & 0.950 &  &  \\
 & 500 & 0.084 & 0.062 & 0.284 & 0.954 &  &  \\
 & 1500 & 0.080 & 0.055 & 0.314 & 0.976 &  &  \\
\cline{1-8}
\multirow{3}{*}{Ad-hoc Neyman} & 100 & 0.083 & 0.061 & 0.274 & 0.956 &  &  \\
 & 500 & 0.083 & 0.059 & 0.278 & 0.956 &  &  \\
 & 1500 & 0.078 & 0.056 & 0.272 & 0.970 &  &  \\
\cline{1-8}
\multirow{3}{*}{Strat. Tree} & 100 & 0.098 & 0.073 & 0.230 & 0.960 &  &  \\
 & 500 & 0.086 & 0.063 & 0.294 & 0.946 &  &  \\
 & 1500 & 0.081 & 0.057 & 0.312 & 0.962 &  &  \\
\cline{1-8}
\multirow{3}{*}{CV Tree} & 100 & 0.081 & 0.060 & 0.292 & 0.972 &  &  \\
 & 500 & 0.082 & 0.057 & 0.256 & 0.950 &  &  \\
 & 1500 & 0.077 & 0.054 & 0.266 & 0.976 &  &  \\
\cline{1-8}
\bottomrule
\end{tabular}
}
\end{table}

\begin{table}[h]
    \centering
     \caption{Simulation results of Model~1 when the null hypothesis is true ($\theta_0 = 0$).}
         \label{tab:res1_null}
     \scalebox{0.95}{
\begin{tabular}{ll|rrrrrr}
\toprule
 &  & MSE & SMSE & R/R & CR & BC & LIL \\
Method & Nuisance &  &  &  &  &  &  \\
\midrule
\multirow[t]{3}{*}{A2IPW} & KNN & 0.049 & 0.032 & 0.028 & 0.972 & 3228.000 & 4091.740 \\
 & NW & 0.056 & 0.038 & 0.044 & 0.956 & 3278.000 & 4099.412 \\
 & NN & 0.057 & 0.041 & 0.046 & 0.954 & 3280.000 & 4033.838 \\
\cline{1-8}
\multirow[t]{3}{*}{ADR} & KNN & 0.050 & 0.032 & 0.030 & 0.970 & 3246.000 & 4131.160 \\
 & NW & 0.052 & 0.033 & 0.048 & 0.952 & 3296.000 & 4052.690 \\
 & NN & 0.061 & 0.044 & 0.066 & 0.934 & 3354.000 & 4054.154 \\
\cline{1-8}
RCT &  & 0.056 & 0.042 & 0.044 & 0.956 & 3368.000 & 4136.404 \\
\cline{1-8}
\bottomrule
\end{tabular}
}
\end{table}

\begin{table}[h]
    \centering
     \caption{Simulation results of Model~2 when the null hypothesis is true ($\theta_0 = 0$).}
         \label{tab:res2_null}
     \scalebox{0.95}{
\begin{tabular}{ll|rrrrrr}
\toprule
 &  & MSE & SMSE & R/R & CR & BC & LIL \\
Method & Nuisance &  &  &  &  &  &  \\
\midrule
\multirow[t]{3}{*}{A2IPW} & KNN & 0.052 & 0.034 & 0.042 & 0.958 & 3180.000 & 4122.806 \\
 & NW & 0.061 & 0.043 & 0.052 & 0.948 & 3300.000 & 4264.980 \\
 & NN & 0.056 & 0.040 & 0.050 & 0.950 & 3352.000 & 4313.950 \\
\cline{1-8}
\multirow[t]{3}{*}{ADR} & KNN & 0.052 & 0.034 & 0.046 & 0.954 & 3252.000 & 4120.868 \\
 & NW & 0.059 & 0.043 & 0.070 & 0.930 & 3176.000 & 4149.514 \\
 & NN & 0.085 & 0.074 & 0.190 & 0.810 & 3314.000 & 4249.732 \\
\cline{1-8}
RCT &  & 0.055 & 0.038 & 0.036 & 0.964 & 3276.000 & 4109.586 \\
\cline{1-8}
\bottomrule
\end{tabular}
}
\end{table}

\begin{table}[h]
    \centering
     \caption{Simulation results of Model~3 when the null hypothesis is true ($\theta_0 = 0$).}
         \label{tab:res3_null}
     \scalebox{0.95}{
\begin{tabular}{ll|rrrrrr}
\toprule
 &  & MSE & SMSE & R/R & CR & BC & LIL \\
Method & Nuisance &  &  &  &  &  &  \\
\midrule
\multirow[t]{3}{*}{A2IPW} & KNN & 0.082 & 0.054 & 0.056 & 0.944 & 3172.000 & 4209.466 \\
 & NW & 0.096 & 0.072 & 0.044 & 0.956 & 3202.000 & 4225.684 \\
 & NN & 0.095 & 0.069 & 0.064 & 0.936 & 3298.000 & 4208.704 \\
\cline{1-8}
\multirow[t]{3}{*}{ADR} & KNN & 0.081 & 0.052 & 0.050 & 0.950 & 3242.000 & 4235.044 \\
 & NW & 0.079 & 0.051 & 0.050 & 0.950 & 3344.000 & 4201.618 \\
 & NN & 0.088 & 0.058 & 0.046 & 0.954 & 3156.000 & 4202.234 \\
\cline{1-8}
RCT &  & 0.086 & 0.062 & 0.050 & 0.950 & 3234.000 & 4259.450 \\
\cline{1-8}
\bottomrule
\end{tabular}
}
\end{table}